\documentclass[11pt]{article}
\pdfoutput=1
\usepackage{authblk}

\def\pb{\pagebreak}   %\usepackage{refcheck} %\usepackage{lineno,lmodern} \linenumbers

\usepackage[dvipsnames]{xcolor}

\usepackage[framemethod=default]{mdframed}
\usepackage{caption}

\newcommand{\tu}{{\tilde{\u}}}

\newcommand{\phantomeq}{\phantom{{}={}}}

\newcommand{\bbw}{\bar{\mathbf{w}}}
\newcommand{\barw}{\bar{w}}

\newcommand{\bfa}{\mathbf{a}}
\newcommand{\bfV}{\mathbf{V}}

\usepackage{bbm}
\usepackage{indentfirst}
\usepackage{bm, mathrsfs, graphics,float,amssymb,amsmath,subeqnarray,setspace,graphicx,amsthm,epstopdf,subfigure, enumerate, color}
\usepackage[utf8]{inputenc}
\usepackage[colorlinks,
            linkcolor=red,
            anchorcolor=blue,
            citecolor=blue
            ]{hyperref}

\usepackage{fullpage}
\numberwithin{equation}{section}

\newtheorem{lemma}{Lemma}
\newtheorem{theorem}{Theorem}
\newtheorem{property}{Property}
\newtheorem{proposition}{Proposition}
\newtheorem{definition}{Definition}
\newtheorem{corollary}[theorem]{Corollary}
\newtheorem{remark}{Remark}

\ifx\assumption\undefined
\newtheorem{assumption}{Assumption}
\fi

\newcommand{\cO}{\mathcal{O}}

\newcommand{\EE}{\mathbb{E}}
\newcommand{\RR}{\mathbb{R}}
\newcommand{\NN}{\mathbb{N}}
\newcommand{\bX}{\bm{X}}

\newcommand{\PP}{\mathbb{P}}

\usepackage{underscore}

\newcommand{\cI}{\mathcal{I}}

\newcommand{\tC}{\tilde{C}}

\newcommand{\minimize}{\mathop{\mathrm{minimize}}}

\renewcommand{\th}{\dot{h}}
\newcommand{\tphi}{\dot{\phi}}

\renewcommand{\bX}{\mathbf{X}}

\newcommand{\hbtheta}{\hat{\bm{\theta}}}
\newcommand{\btheta}{\bm{\theta}}

\newcommand{\balpha}{\bm{\alpha}}
\newcommand{\tbalpha}{\tilde{\bm{\alpha}}}

\newcommand{\iiddistr}{{\stackrel{\text{i.i.d.}}{\sim}}}

\newcommand{\norm}[1]{\left\|{#1} \right\|}
\newcommand{\Norm}[1]{\|{#1} \|}

\newcommand{\bbalpha}{\bar{\bm{\alpha}}}
\newcommand{\bgamma}{\bm{\gamma}}

\newcommand{\hbalpha}{\hat{\bm{\alpha}}}
\newcommand{\bbeta}{\bm{\beta}}
\newcommand{\bbbeta}{\bar{\bm{\beta}}}
\newcommand{\hbbeta}{\hat{\bm{\beta}}}
\newcommand{\tbbeta}{\tilde{\bm{\beta}}}
\newcommand{\hv}{\hat{v}}
\newcommand{\bv}{\mathbf{v}}

\newcommand{\blambda}{\bar{\lambda}}

\newcommand{\bu}{\mathbf{u}}
\newcommand{\bbfu}{\bar{\mathbf{u}}}
\newcommand{\hbv}{\hat{\mathbf{v}}}
\newcommand{\bbv}{\bar{\mathbf{v}}}

\usepackage{multirow}
\usepackage{tablefootnote}
\usepackage{colortbl}% http://ctan.org/pkg/xcolor
\usepackage{hhline}% http://ctan.org/pkg/hhline

\usepackage{algorithm}
\usepackage{algorithmic}

\usepackage{makecell}%%%%%%%%%%%%%%%%%%%%%%%%%%%%%%ZWZ
\usepackage{tikz}%%%%%%%%%%%%%%%%%%%%%%%%%%%%%%ZWZ

\def\ep{\varepsilon}

\newcommand{\p}{\mathbf{p}}

\newcommand{\bx}{\mathbf{x}}

\newcommand{\w}{w}

\newcommand{\bxi}{\bm{\xi}}
\newcommand{\E}{\mathbb{E}}

\newcommand{\hbw}{\hat{\mathbf{w}}}

\newcommand{\tbTheta}{\tilde{\mathbf{\Theta}}}

\newcommand{\bbTheta}{\bar{\mathbf{\Theta}}}
\newcommand{\bbtheta}{\bar{\bm{\theta}}}

\newcommand{\hw}{\hat{w}}
\newcommand{\bw}{\mathbf{w}}

\newcommand{\bTheta}{\mathbf{\Theta}}

\newcommand{\tbtheta}{\tilde{\bm{\theta}}}
\newcommand{\ttw}{\tilde{w}}
\newcommand{\ttv}{\tilde{v}}

\newcommand{\hGrad}{\hat{\mathcal{G}}}
\newcommand{\Grad}{\mathcal{G}}
\newcommand{\hDel}{\hat{\mathcal{D}}}

\newcommand{\lineDel}{\overline{\mathcal{D}}}
\newcommand{\lineGrad}{\overline{\mathcal{G}}}
\newcommand{\ulineDel}{\underline{\mathcal{D}}}
\newcommand{\ulineGrad}{\underline{\mathcal{G}}}
\newcommand{\Del}{\mathcal{D}}
\newcommand{\NhGrad}{\widehat{\mathcal{NG}}}
\newcommand{\NhDel}{\widehat{\mathcal{ND}}}
\newcommand{\NGrad}{\mathcal{NG}}
\newcommand{\NDel}{\mathcal{ND}}

\newcommand{\lineTheta}{\bar{\mathbf{\Theta}}}
\newcommand{\linew}{\bar{w}}
\newcommand{\linebw}{\bar{\bw}}
\newcommand{\linetheta}{\bar{\bm{\theta}}}

\newcommand{\bK}{\mathbf{K}}
\newcommand{\tbK}{\tilde{\mathbf{K}}}
\newcommand{\hbK}{\hat{\mathbf{K}}}
\newcommand{\bbK}{\bar{\mathbf{K}}}

\newcommand{\bA}{\mathbf{A}}

\newcommand{\Dis}{\mathrm{D}}

\newcommand{\btPsi}{\mathbf{\tilde\Psi}}
\newcommand{\bfPsi}{\mathbf{\Psi}}

\newcommand{\btPhi}{\mathbf{\tilde\Phi}}
\newcommand{\bfPhi}{\mathbf{\Phi}}

\newcommand{\tv}{\tilde{v}}
\newcommand{\tbv}{\tilde{\bv}}

\newcommand{\bc}{\mathbf{c}}
\newcommand{\ddp}{\dot{p}}
\newcommand{\bC}{\mathbf{C}}
\newcommand{\bR}{\mathbf{R}}

\newcommand{\supp}{\mathrm{supp}}

\newcommand{\Up}{\Upsilon}

\newcommand{\e}{\mathbf{e}}

\renewcommand{\u}{{u}}

\newcommand{\<}{\left\langle}
\renewcommand{\>}{\right\rangle}

\usepackage{mathtools}

\newcommand{\cP}{\mathcal{P}}
\newcommand{\tO}{\tilde{\cO}}

\newcommand{\tw}{{\tilde{w}}}
\newcommand{\tp}{{\tilde{p}}}

\newcommand{\tPi}{\tilde{\Pi}}
\newcommand{\bPi}{\tilde{\Pi}}

\newcommand{\linev}{\bar{v}}
\newcommand{\linebv}{\bar{\bv}}
\newcommand{\linealpha}{\bar{\bm{\alpha}}}
\newcommand{\linebeta}{\bar{\bm{\beta}}}
\newcommand{\hbH}{\hat{\mathbf{H}}}

\newcommand{\var}{\mathsf{var}}
\newcommand{\bL}{\mathcal{L}}

\newcommand{\tPhi}{\tilde{\Phi}}

\newcommand{\tPsi}{\tilde{\Psi}}

\usepackage{amsmath}

\allowdisplaybreaks  %%%%%ZWZ

\usepackage{booktabs}

\title{Modeling from Features: a Mean-field Framework\\ for  Over-parameterized Deep Neural Networks}

\author[$*$]{Cong Fang}

\affil[$*$]{
 Department of Electrical Engineering,   Princeton University}

\author[$*$]{
Jason D. Lee}

\author[$*$]{
Pengkun Yang}

\author[$\dag$]{
Tong Zhang}
\affil[$\dag$]{Department of Computer Science and Mathematics\\ HKUST}
\date{June 30, 2020}

\usepackage{setspace}

\begin{document}

\maketitle
\begin{abstract}
%\iffalse
  This paper  proposes a new mean-field framework for over-parameterized deep neural networks (DNNs), which can be used to analyze neural network training.   In this framework,  a  DNN is represented by probability measures and functions over its features (that is, the function values of the hidden units over the training data) in the continuous limit,  instead of the neural network parameters as most  existing studies have done.
  This new representation   overcomes the   degenerate situation  where all the hidden units essentially have only one meaningful hidden unit in each middle layer, and further
  leads to a simpler representation of DNNs, for which the training objective can be reformulated as  a convex optimization  problem via suitable re-parameterization.   Moreover, we construct a non-linear dynamics  called  \emph{neural feature flow}, which captures the evolution of an over-parameterized DNN trained by Gradient Descent.   We illustrate the framework via the standard DNN and the Residual Network (Res-Net) architectures. Furthermore,  we show, for Res-Net,  when the neural feature flow process converges, it reaches a \emph{global} minimal solution  under suitable conditions. Our analysis leads to the first global convergence proof for over-parameterized neural network training with more than $3$ layers in the mean-field regime.
\end{abstract}

\tableofcontents

\section{Introduction}\label{sec:intro}
In recent years, deep neural networks (DNNs) have achieved great success empirically. However, the theoretical understanding of the practical success is still limited. One main conceptual difficulty is the non-convexity of DNN models.  More recently, there has been remarkable progress in understanding the over-parameterized neural networks (NNs), which are NNs with massive hidden units.  The over-parameterization is capable of circumventing the hurdles in analyzing non-convex functions under specific settings: 
\begin{enumerate}[(i)]
  \item  Under a specific scaling and initialization, it is sufficient to study the NN weights in a small region around the initial values given sufficiently many hidden units - the aptly named ``lazy training'' regime \cite{jacot2018neural,li2018learning,du2018gradient,arora2019fine,du2019gradient,allen2018learning,allen2019can,zou2018stochastic,chizat2019lazy}. 
    The NN in this regime is nearly a  linear model fitted with  random features that induces a  kernel in the tangent space.
    % , and provably achieves minimum training error. 
    % However, this regime does not explain why NNs can effectively learn representative features, and the expressive power of random kernels is limited \cite{yehudai2019power}. 
    
    \item Another line of research applies the mean-field analysis for NNs \cite{MeiE7665,chizat2018global,sirignano2019mean,rotskoff2018neural,mei2019mean, dou2019training,  wei2018margin,sirignano2019mean2,fang2019over, araujo2019mean,nguyen2020rigorous,chen2020mean}.  For over-parameterized NNs, it is instructive to first study the behavior of the infinitely wide NNs, known as the mean-field limit,  and then consider the approximation using finite neurons. This idea comes from   statistical mechanics \cite{engel2001statistical} suggesting that
    modeling a  volume of interacting neurons can be  largely simplified as modeling an averaging probability distribution.
    
    %, and thereby learning a  two-level  NN can be regarded as optimizing a (convex) functional over the probability distribution of its weights.  
    % The evolution of NN weights trained by the (noisy) Gradient Descent algorithm corresponds to a Wasserstein gradient flow called ``distributional dynamics'',  solution to a non-linear partial differential equation of McKean-Vlasov type \citep{sznitman1991topics}.  
    % In the mean-field limit, the Wasserstein gradient flow converges to the optimal solution for two-level NNs \citep{MeiE7665,chizat2018global,fang2019over}. Compared with  lazy training, the mean-field view can characterize the entire training process of NNs.
\end{enumerate}

The ``lazy training'' regime has been extensively investigated for both shallow and deep NNs. 
In that regime, the NN provably achieves minimum training error despite the non-convexity of NN models; % and the NNs generalizes well to new samples  
however, other useful characteristics of NNs such as feature learning remain obscure, and it is known that the expressive power of random kernels can be limited \cite{yehudai2019power}.

%In the existing literature, 
Turning to the mean-field view, most of the existing  studies focus  on two-level NNs. In the mean-field limit, a two-level NN is represented by a probability distribution over its weights, and (noisy) Gradient Descent corresponds to 
 a Wasserstein gradient flow called ``distributional dynamics'', which is the solution to a non-linear partial differential equation of McKean-Vlasov type \cite{sznitman1991topics}.  Moreover,   the overall learning problem  can be reformulated as a convex optimization  over the  
  probability space and it was shown in  \cite{MeiE7665,chizat2018global,fang2019over}   that such  gradient flow converges to the optimal solution under suitable conditions. Compared with lazy training, the mean-field view can characterize the entire training process of NNs.
 %The evolution of NN weights trained by the (noisy) Gradient Descent algorithm has been successfully characterized by a Wasserstein gradient flow called ``distributional dynamics'',  solution to a non-linear partial differential equation of McKean-Vlasov type \citep{sznitman1991topics}.  
%In the mean-field limit, the Wasserstein gradient flow converges to the optimal solution for two-level NNs \citep{MeiE7665,chizat2018global,fang2019over}. 
% Compared with  lazy training, the mean-field view can characterize the entire training process of NNs.

%However,  the mean-field analysis on DNNs is a challenging task.
However, it is a challenging task to conduct  the mean-field analysis on DNNs.
First of all, it is not easy to formulate the mean-field limit of DNNs. As we will discuss in Section~\ref{sec:related}, extending existing formulations to DNNs, hidden units in a middle layer essentially behave as a single unit along the training.  
This degenerate situation arguably cannot fully characterize the training process of actual DNNs.  
Furthermore, understandings for the global convergence of DNNs are still limited in the mean-field regime. Beyond two layers, the only result to the best of our knowledge came from \cite{nguyen2020rigorous}  recently, in which they proved the global convergence for three-level DNNs under  restrictive conditions.  It is not clear how to extend their analysis to deeper NNs.

% \red{The above is same as nips.}

In this paper,  we propose a new mean-field framework for over-parameterized DNNs to analyze NN training.  
In contrast to existing studies focusing on the NN weights, this framework represents a DNN in the continuous, i.e. mean-field, limit by probability measures and functions over its features, that is, the outputs of the hidden units over the training data.   
We also describe a non-linear dynamic called \emph{neural feature flow} that captures the evolution of a DNN trained by Gradient Descent.   We illustrate the framework via the standard DNN and Residual Network (Res-Net)  \cite{he2016deep} architectures. We  show that,  for Res-Nets,  when the neural feature flow process converges, it reaches a \emph{global} minimal solution  under suitable conditions.

% Especially, the paper is organized as follows. 
% Section \ref{sec:related}  discusses existing works on the mean-field theory of NNs, and the challenges in modeling DNNs. 
Specifically, we first investigate the fully-connected DNNs.
% and introduce  the  formulation in Section \ref{sec:DNN formulation}.    
Under our framework in Section \ref{sec:DNN formulation}, the hidden units and the weights of a DNN in  the continuous limit are characterized by  probability measures and   functions over the features, respectively. 
This new representation overcomes the degenerate situation in previous studies  \cite{araujo2019mean, nguyen2020rigorous}.
Then we propose the neural feature flow that captures the evolution of a DNN trained by Gradient Descent in Section \ref{sec:flow DNN} and analyze the theoretical properties in Section \ref{sec:theo DNN}.
% Section \ref{sec:flow DNN}  proposes the neural feature flow to capture the evolution of a DNN trained by Gradient Descent and Section \ref{sec:theo DNN} theoretically analyzes its properties.  
Neural feature flow involves the evolution of the features and does not require the boundedness of the weights.  We provide a general initialization condition for a discrete DNN and show that   Gradient Descent from such initialization with a suitable time scale  can be well-approximated by  its limit, i.e., neural feature flow, when the number of hidden units is sufficiently large.    
We demonstrate concrete examples that provably  achieve the initialization condition.  In fact,  under the standard initialization method of discrete DNNs \cite{glorot2010understanding,he2015delving}, the NN weights scale to infinity with the growth of the number of hidden units.   
There are empirical studies, e.g., \cite{zhang2019fixup}, which show that properly rescaling the standard initialization stabilizes training. We  introduce a simple $\ell_2$-regression at initialization (see Algorithm \ref{algo:init}) and show that  this regularized initialization  ensures the general initiation condition. On the other hand, our new formulation for a continuous DNN can be  re-parameterized as a convex  problem under proper conditions, which provides  the opportunities to find a globally optimal solution if we   impose suitable  regularizers on the features, e.g., relative entropy regularizers.    However, it  remains open how to analyze the evolution of DNNs with those regularizers.

Here, we try a different way.   We consider training a Res-Net model in Sections \ref{sec:resnet} and \ref{sec:theo resnet}.  We novelly  characterize the neural feature flow via trajectories  of the  skip-connected paths (explained in Subsection \ref{sec:general resnet}) for Res-Nets.  More or less surprisingly, we show that when the neural feature flow process converges, it reaches a \emph{globally} optimal solution under suitable conditions. To the best of our knowledge, our analysis leads to the \emph{first} proof for the global convergence of training over-parameterized DNNs with more than $3$ layers in the mean-field regime.   
% Finally, Section \ref{sec:conclu} concludes our paper with future directions.  
\subsection{Contributions}
The main contributions of this work are the following
\begin{enumerate}[(1)]
    \item We propose a new mean-field framework of DNNs which  characterizes  DNNs via probability measures and  functions  over  the  features and introduce  neural feature flow to capture the evolution of DNNs trained by the Gradient Descent algorithm.
    \item For the Res-Net model, we  show that  neural feature flow can find a global minimal solution of the learning task under certain conditions. 
\end{enumerate}

\subsection{Notations}
Let  $[m_1:m_2] := \{m_1, m_1+1, \dots,m_2 \}$ for $m_1,m_2\in\NN$ with $m_1\leq m_2$ and $[m_2] := [1:m_2]$ for $m_2\geq1$.  
Let $\mathcal{P}^{n}$ be the set of probability distributions over $\RR^{n}$.  
% The set of continuous paths on $\RR^n$ is denoted by $C()$.
For  a matrix $\bA\in \RR^{n \times m} $, let $\|\bA\|_2$, $\|\bA\|_F$, and $\|\bA \|_{\infty}$ denote its operator, Frobenius, max norms, respectively. If $\bA$ is symmetric, let $\lambda_{\min}(\bA)$ be its smallest eigenvalue. 
Vectors are treated as columns.  
For a vector $\mathbf{a} \in \RR^n$, let $\|\mathbf{a}\|_2$ and $\|\mathbf{a} \|_{\infty}$ denote its $\ell_2$ and $\ell_\infty$ norms, respectively.
The $i$-th coordinate is denoted by $\mathbf{a}(i)$.
% $\RR_{+}^n$ denotes the set of $n$-dimensional vectors that all the coordinates are non-negative.   
For $\mathbf{a},\mathbf{b}\in \RR^{n}$, denote the entrywise product by $\mathbf{a}\cdot \mathbf{b} $ that $\left[\mathbf{a}\cdot\mathbf{b}\right](i):= \mathbf{a}(i)\cdot \mathbf{b}(i)$ for $i\in[n]$. 
% \red{ 
For $c>0$ and $p\in[1, \infty]$, let $\mathcal{B}_{p}(\mathbf{a},c)$ denote the $\ell_{p}$-ball centered at $\mathbf{a}$ of radius $c$.
% }  
For an unary function  $f:\RR\to \RR$, define $\dot{f}:\RR^{n}\to \RR^{n}$ as the entrywise operation that $\dot{f}(\mathbf{a})(i) =f(\mathbf{a}(i))$ for $i\in[n]$ and $\mathbf{a}\in \RR^{n}$.  Denote $n$-dimensional identity matrix by $\mathbf{I}^n$.   Denote  $m$-by-$n$ zero matrix and   $n$-dimensional zero vector by 
$\mathbf{0}^{n\times m}$ and $\mathbf{0}^{n}$, respectively.   We say a univariate distribution $p$ is $\sigma$-sub-gaussian if $\EE_{x\sim p} \exp(x^2/\sigma^2)\leq e\footnote{\text{Here the value $e$ can be replaced by any number greater than one. See \cite[Remark 5.6]{vershynin2010introduction}.}}$; we say a $d$-dimensional distribution $p$ is $\sigma$-sub-gaussian if the law of $u^\top \mathbf{x}$ is $\sigma$-sub-gaussian for $\mathbf{x}\sim p$ and any $u\in \mathbb{S}^{d-1}$.  
For two positive sequences  $\{p_n\}$ and  $\{q_n\}$,
$p_n = \cO(q_n)$ if $p_n \leq Cq_n$ for some positive constant $C$, and $p_n = \Omega (q_n)$ if  $ q_n = \cO(p_n)$. 
%Moreover, $p_n = \tO(q_n)$ if $p_n = \cO (q_n\log^k(q_n))$ for some $k>0$, and $p_n = \tilde{\Omega} (q_n)$ if $q_n = \tO(p_n)$.

% \subsection{Extensions of the short version.}
% ...

\section{Discussions on  Deep Mean-field Theory}\label{sec:related}
\subsection{Challenges on Modeling DNNs in the Mean-field Limit}
We discuss related mean-field studies and point out the challenges in modeling  DNNs.  For two-level NNs, most of the existing works \cite{MeiE7665,chizat2018global,sirignano2019mean,rotskoff2018neural} formulate the continuous limit as
$$   f(\bx;p) =  \int w_2~h\left(\mathbf{w}_1^\top \bx\right)d p\left(w_2, \mathbf{w}_1\right),  $$
where $p$ is the probability distribution over the pair of weights $(w_2,\bw_1)$.
The weights of the second layer $w_2$ can be viewed as functions of $\mathbf{w}_1$, which is a $d$-dimensional vector. 
However, this approach indexes higher-layer weights, say $w_3$, by functions over features of the hidden layer, with a diverging dimensionality in the mean-field limit.
For $3$-level NNs, $w_3$ as the last hidden layer is indexed by the connection to the output units in \cite{nguyen2020rigorous}, which is not generalizable when middle layers present. 
An alternative approach is to model DNNs with nested measures (also known as multi-level measures; see \cite{dawson1982wandering,dawson2018multilevel} and references therein), which however suffers the closure problem to establish a well-defined limit (see discussions in \cite[Section 4.3]{sirignano2019mean2}). 
% However, the distributions of higher-level weights involve distributions over infinite-dimensional spaces. 

The continuous limit of DNNs is investigated by \cite{araujo2019mean,nguyen2020rigorous} under the initialization that all weights are i.i.d.~realizations of a \emph{fixed} distribution independent of the number of hidden units. 
However, under that setting, all neurons in a middle layer essentially behave as a single neuron.
Consider the output $\hbbeta$ of a middle-layer neuron connecting to $m$ hidden neurons in the previous layer: 
\begin{equation}
\label{eq:forward}
\hbbeta=\frac{1}{m}\sum_{i=1}^m h(\hbbeta_i')~w_i,
\end{equation}
where $\hbbeta_i'$ is the output of $i$-th hidden neuron in the previous layer with bounded variance, $w_i$ is the connecting weight.
If $w_i$ is initialized independently from $\mathcal{N}(0,1)$, it is clear that $\var[\hbbeta]\to 0$ as $m\to\infty$, and thus the hidden neurons in middle layers are indistinguishable at the initialization. 
Moreover, the phenomenon sustains along the entire training process, as shown in Proposition~\ref{prop:degenerate}. 
This phenomenon serves as the basis of \cite{araujo2019mean,nguyen2020rigorous} to characterize the mean-field limit using finite-dimensional probability distributions. 
This degenerate situation arguably does not fully characterize the actual DNN training. 
In fact, similar calculations to \eqref{eq:forward} are carried out by \cite{glorot2010understanding,he2015delving} and motivate the popular initialization strategy with $\mathcal{N}(0,\cO(m))$ such that the variance of  $\hbbeta$ is non-vanishing.

\begin{proposition}
\label{prop:degenerate}
Consider fully-connected $L$-layer DNNs with $m$ units in each hidden layer trained by Gradient Descent. 
Let $\hbbeta_{\ell,i}^k$ denote the output of $i$-th hidden neuron at $\ell$-th layer and $k$-th iteration, and define $\Delta_{\ell,m}:= \max_{i\ne i',k\in [K]}\|\hbbeta_{\ell,i}^k-\hbbeta_{\ell,i'}^k\|_\infty$.
Then, for every $\ell\in[2:L-1]$, 
\[
\lim_{m\to\infty}\Delta_{\ell,m}= 0.
\]
% \nbr{review and proof}
\end{proposition}

\subsection{Comparisons of Dynamics}
\iffalse
There is another line of researches that aim to propose a convex formulation  for DNN, which also yields   a non-trivial task since   DNN is no longer a linear model with respect to the distributions of the weights.  More recently, \citet{fang2019over, fang2019convex}  propose a new view named neural feature repopution for understanding NNs.   Interestingly, \citet{fang2019convex} show that under suitable re-parameterization, the continuous DNN can be reformulated as a convex optimization problem. However,   the continuous DNN in \cite{fang2019convex}   are characterized by probability measures over infinite dimensional spaces. Some of  their results are not stated with complete mathematical rigor. 

Our work provides an inspiring solution on the two issues. On the one hand,  our continuous DNN is modeled based on the features.  So there is no stringent  restriction on the features. On the other hand, following the technique proposed in \cite{fang2019convex}, we prove,  in a rigorous manner, that the learning problem in our formulation  can also be re-parameterized as a convex optimization  under suitable conditions, which  affords understandings on the  landscape of the over-parameterized DNNs. 
\fi

It is known that the evolution of a two-level NN  trained by the Gradient Descent can be described as  a  Mckean-Vlason process \cite{MeiE7665,chizat2018global}. However, to the best of our knowledge, it  still remains as a question whether  the evolution of DNNs can be captured by   PDEs of Mckean-Vlason type.  Recently, \cite{araujo2019mean} gave an affirmative answer for  DNNs under  a specific  condition where    the weights  in  the first and the last layers are not updated by the Gradient Descent. Nevertheless,  in the middle layers,  their  model only has one meaningful neuron.   More recently, \cite{nguyen2020rigorous}  proposed another attempt by directly tracking the trajectories of the  weights. In their description,   the evolution of Gradient Descent is characterized  by  systems of ODEs, which are relatively easier to analyze and avoid the  presence of  the conditional probabilities.

Our description for the evolution of Gradient Descent is similar to \cite{nguyen2020rigorous} in a more general situation where  our dynamic $(\mathrm{i})$ involves  the evolution of the features and  $(\mathrm{ii})$ does not require the boundedness of  the weights. Moreover,  we novelly introduce the conception of skip-connected paths to deal with the  Res-Net architecture.

\section{Continuous DNN Formulation}\label{sec:DNN formulation}
We consider the empirical minimization problem over $N$ training samples $\{\bx^i, y^i \}_{i=1}^N$, where $\bx^i\in \RR^d$ and $y^i\in\mathcal{Y}$.  For regression problems, $\mathcal{Y}$ is typically $\RR$; for classification problems,  $\mathcal{Y}$ is often $[K]$ for an integer $K$.  We first present the formulation of a standard $L$-layer  DNN ($L \geq 2$).

\subsection{Discrete DNN}\label{subsection:disdnn}
For discrete DNNs, let $m_\ell$ denote the number of units at layer $\ell$ for $\ell\in [0:L+1]$. 
Let $m_0=d$ and node $i$ outputs the value of $i$-th coordinate of the training data for $i\in[d]$.
Let $m_{L+1}=1$ that is the unit of the final network output.
For $\ell\in[L+1]$, the output, i.e. features, of node $i$ in layer $\ell$ is denoted by $\hbtheta_{\ell,i}\in\RR^N$; the weight that connects the node $i$ at layer $\ell-1$ to node $j$ at layer $\ell$ is denoted by $\hw_{\ell,i,j}\in \RR$.

\begin{enumerate}[(1)]
    \item  At the input layer, for $i\in[d]$, let
    \begin{eqnarray}\label{dataf}
      \hbtheta_{0,i} :=\left[\bx^1(i), \bx^2(i), \dots, \bx^N(i)  \right]^\top.
    \end{eqnarray}
    \item  We  recursively define the upper layers ($\ell \in [L]$) as below.
\begin{eqnarray}\label{dfor 1}
    \hbtheta_{\ell, j} :=
    \begin{cases}
 \frac{1}{m_{0}}\sum_{i=1}^{m_{0}} \hw_{1, i, j}~ \hbtheta_{0, i} ,&  \quad j \in [m_{1}],~ ~\ell = {1},\\
\frac{1}{m_{\ell-1}}\sum_{i=1}^{m_{\ell-1}}  \hw_{\ell, i, j}~ \th\left(\hbtheta_{\ell-1, i}\right),&  \quad j \in [m_{\ell}],~ \ell \in [2:L], \\
\end{cases}
\end{eqnarray}
where $h$ is the activation function.
     \item At the  output layer, 
     \begin{eqnarray}\label{dfor 2}
    \hbtheta_{L+1, 1} :=
 \frac{1}{m_{L}}\sum_{i=1}^{m_{L}} \hw_{L+1, i}~ \th\left(\hbtheta_{L, i}\right). 
\end{eqnarray}
 Finally,  there is  a loss function     to measure the quality of the predicted result and a  regulazier to  control the complexity of the model or to  avoid ill-conditions.
\end{enumerate}
For convenience,   we  collect the weights at the $\ell$-th layer ($\ell \in [L+1]$) into a single  vector: 
$$ \hbw_{\ell} := \left\{ \hat{w}_{\ell, i, j}:  ~i\in [m_{\ell-1}], ~j\in [m_{\ell}]  \right\} \in \RR^{m_{\ell-1} m_{\ell}} $$
and  all the weights into a single  vector:
\begin{eqnarray}\label{pad w}
\hbw := \left\{ \hbw_{\ell}:~ \ell \in[L+1]\right\},
\end{eqnarray}
where $\hbw$ is an element of $\RR^{w_p}$ with $w_p = \sum_{\ell=1}^ {L+1}  m_{\ell-1} m_{\ell}$.
Similarly,   we aggregate features at $\ell$-th layer ($\ell \in [L]$) into a single  vector:
$$ \hbtheta_{\ell} := \left\{ \hbtheta_{\ell,i}:  ~i\in [m_{\ell}] \right\}\in \RR^{N m_{\ell}}  $$
and all the features into a single  vector:
\begin{eqnarray}\label{pad feature}
 \hbtheta = \left\{ \hbtheta_{\ell}: \ell\in[ L] \right\}\in \RR^{N \sum_{\ell=1}^{L}m_{\ell}} .
\end{eqnarray}
 The overall  learning problem for a  DNN can be formulated as a constrained optimization problem as
\begin{eqnarray}\label{dis DNN}
    \min_{\hbw, \hbtheta }\hat{\bL} (\hbw,\hbtheta) =  \frac{1}{N}\sum_{n=1}^N \phi\left(   \hbtheta_{L+1,1}(n), ~y^n\right) +  \hat{R}(\hbw, \hbtheta),
\end{eqnarray}
where $\hbw$, $\hbtheta $, and $\hbtheta_{L+1}$  are subjected to \eqref{dfor 1} and \eqref{dfor 2}.  Note in \eqref{dis DNN}, $ \hat{R}$ is the regularizer and   $ \phi(\cdot, \cdot):\RR\times \mathcal{Y}\to \RR$ is the loss function that  is  assumed  to be convex in the first argument.

\subsection{Continuous DNN}
Next we introduce our continuous DNN formulation using similar forward propagation of the the discrete DNN in Section \ref{subsection:disdnn}. 
Given a set of $N$ training samples, it is clear that the feature of each neuron is a  $N$-dimensional vector.
This observation motivates our new formulation that uses the distribution of the features to characterize the overall state of each layer. 
% What motivates  our new formulation is the special observation that given  $N$ training samples, the function value for the training samples in each node, i.e., feature, is a  $N$-dimensional vector.    Therefore we use the distribution of the features to characterize the state of a layer.   
This formulation circumvents the issue of infinite dimensionality by reducing the redundancy of existing mean-field modelings that index neurons by the connection to the previous layers. 
We present the details below.

\begin{enumerate}[(1)]
    \item  At the input layer, let $\bX = \left[\bx^1, \bx^2, \dots, \bx^N\right]^\top\in\RR^{N\times d}$.
    \item At the first layer,   each hidden node (before the activation function) is computed by a linear mapping of the input data, so each node can be indexed  by the weights connecting it to the input. 
    We introduce  a probability  measure  $p_1\left(\bw_1\right) \in \mathcal{P}\left(\RR^d\right)$ for the weights to describe the states of  first layer and let\footnote{
The state of the first layer can be equivalently characterized by either the output or the weight that are related by a linear mapping. 
}
    \begin{eqnarray}\label{cfor 1d}
        \btheta_1\left(\bw_1\right) :=  \frac{1}{d} \left(\bX \bw_1\right).
    \end{eqnarray}
    \item  At the second layer, recall that the output of each node, i.e., the feature, for the training samples is a $N$-dimensional vector.
    We use the features $\btheta_{2}\in \RR^N$ to index those nodes. 
    % recalling that  the function value of each node, i.e. feature, for the training samples is a $N$-dimensional vector,   we define $\btheta_{2}\in \RR^N$ as the node  whose  feature is  $\btheta_{2}$.   
    We    introduce  a probability  measure  $p_2(\btheta_2) \in \mathcal{P}(\RR^N)$  to describe the overall states of  the second layer and function $w_2:\supp(p_1)\times \supp(p_2) \to  \RR$
    to denote the weights on the connections from layer $1$ to $2$. We have for all $\btheta_2\in \supp(p_2)$
    \begin{equation}
    \label{forward l2} 
    \int w_2\left(\bw_1,\btheta_2\right)~ \th\left(\btheta_1(\bw_1)\right)   d p_{1}\left(\bw_1\right) = \btheta_2.
    \end{equation}
    \item  Similarly,  for $\ell\in [3:L]$, let $\btheta_{\ell}\in \RR^N$ be the index of nodes according to the features. 
    We    introduce  a probability  measure  $p_\ell(\btheta_\ell) \in \mathcal{P}(\RR^N)$  to describe the states the $\ell$-th layer and  function $w_{\ell}:\supp(p_{\ell-1})\times \supp(p_{\ell}) \to  \RR$
    to denote the weights on the connections from layer $\ell-1$ to $\ell$. 
    We have for all $\btheta_\ell\in \supp(p_\ell)$
    \begin{equation}
    \label{forward l} 
        \int w_{\ell}\left(\btheta_{\ell-1},\btheta_{\ell}\right)~\th\left(\btheta_{\ell-1}\right) dp_{\ell-1}\left(\btheta_{\ell-1}\right)  = \btheta_{\ell}. 
    \end{equation}
    \item  Finally,   let $w_{L+1}: \supp(p)\to \RR$ be the weights in the layer $L+1$ and $\btheta_{L+1}$ be the final output, and we have
    \begin{eqnarray}
        \int w_{L+1}\left(\btheta_{L}\right)~\th\left(\btheta_{L}\right) d p_{L}\left(\btheta_{L}\right)  = \btheta_{L+1}. \notag
    \end{eqnarray}
\end{enumerate}
The overall learning problem for the continuous DNNs is formulated as
\begin{eqnarray}\label{problem1}
\minimize_{\left\{\w_{\ell}\right\}_{\ell=2}^{L+1},~ \left\{p_{\ell}\right\}_{\ell=1}^L }    &&\!\!\!\!\!\!\!\bL\left(\left\{\w_{\ell}\right\}_{\ell=2}^{L+1}, \left\{p_{\ell}\right\}_{\ell=1}^L  \right) =  \frac{1}{N}\sum_{n=1}^N \phi\left( \btheta_{L+1}(n), y^n \right) +  R\left( \left\{\w_{\ell}\right\}_{\ell=2}^{L+1}, \left\{p_{\ell}\right\}_{\ell=1}^L \right)  \\
\text{s.t.}  && \!\!\!\!\!\!\!  \int w_2\left(\bw_1,\btheta_2\right) \th\left(\btheta_1\left(\bw_1\right)\right)   dp_{1}\left(\bw_1\right)  = \btheta_2,\quad \text{for all}~ \btheta_2\in \supp(p_2),\notag\\
&&  \!\!\!\!\!\!\! \int w_{\ell}\left(\btheta_{\ell-1},\btheta_{\ell}\right)\th\left(\btheta_{\ell-1}\right) dp_{\ell-1}\left(\btheta_{\ell-1}\right) = \btheta_{\ell}, \quad \text{for all}~ \btheta_\ell\in \supp(p_\ell), ~\ell\in[3:L],\notag\\
&&  \!\!\!\!\!\!\!   \int w_{L+1}\left(\btheta_{L}\right)\th\left(\btheta_{L}\right) dp_{L}\left(\btheta_{L}\right)  = \btheta_{L+1}, \notag
\end{eqnarray}
where  the regularizer $R$ is the continuous formulation of $\hat{R}$.   

%\begin{remark}
% For convenience, in the case without of confusion,  we will denote $w_{L+1}(\btheta_{L},\btheta_{L+1} ) = w_{L+1}(\btheta_{L})$. 
%\end{remark}
%\begin{remark}
% The reason why  we  use the distribution of weights to characterize the $1$-st layer will be explained in Section \ref{}.
%\end{remark}

\section{Scaled Gradient Descent and Neural Feature Flow for DNN}\label{sec:flow DNN}
\subsection{Scaled Gradient Descent for Training DNN}
\begin{algorithm}[tb]
	\caption{Scaled Gradient Descent for Training a  DNN}
	\label{algo:GD}
		\begin{algorithmic}[1]
	\STATE	Input the data  $\{\bx^i,y^i\}_{i=1}^N$,  step size $\eta$, and  initial weights $\hbw^0$.
 \FOR{$k =0, 1,\dots, K-1$} 
	\STATE  Perform forward-propagation  \eqref{dfor 1} and \eqref{dfor 2} to compute $\hbtheta^{k}_{L+1,1}$.
	 \STATE  Perform backward-propagation to compute the gradient   $\hGrad^k_{\ell,i,j}=\frac{\partial \hat{\bL}}{\partial \hw_{\ell,i,j}^k}$. 
	  \STATE  Perform scaled Gradient Descent:
	  \begin{eqnarray}
	     \hw_{\ell,i,j}^{k+1} &=&   \hw^{k}_{\ell,i,j} - \big[\eta m_{\ell-1}m_{\ell}\big]  ~ \hGrad^k_{\ell,i,j}, \quad \ell\in[L+1], ~i\in [m_{\ell-1}],~j\in [m_{\ell}]. \notag
	  \end{eqnarray} \vspace{-0.25in}
	  \ENDFOR
	  \STATE Output the weights $\hbw^K$.
	\end{algorithmic}
\end{algorithm}

In this section, we  focus on the  scaled Gradient Descent algorithm and deduce  its continuous limit.   
% For the sake of simplicity in our analysis, we  temporarily do not consider the regularizer  ($R\equiv0$).  
For the sake of simplicity, we analyze the algorithm without regularizer.
We consider the scaled Gradient Descent algorithm with appropriate step sizes (time scales) for the parameters to match the scale in the continuous limit. Similar scaling is also adopted in existing mean-field theory of DNNs \cite{araujo2019mean,nguyen2020rigorous}.
Given an initial  weights $\hbw^0$, the meta algorithm of the scaled Gradient Descent is shown in Algorithm \ref{algo:GD}, where the gradients $\hGrad$ can be obtained by the standard backward-propagation algorithm.  Especially, by  introducing intermediate variables:
\begin{eqnarray}
 \hDel^k_{L+1,1} &:=&  N~\frac{\partial \hat{L}^k}{\partial \hbtheta_{L+1}  } = \left[\phi'_1\left( \hbtheta_{L+1}^k(1), y^1\right), \phi'_1\left( \hbtheta_{L+1}^k(2), y^2\right),\dots, \phi'_1\left( \hbtheta_{L+1}^k(N), y^N\right) \right]^\top,\notag\\
  \hDel^k_{\ell,i} &:=&   N~\frac{\partial \hat{L}^k}{\partial \hbtheta_{\ell,i}  }= \frac{1}{m_{\ell}} \sum_{j=1}^{m_{\ell+1}}\hw_{\ell+1,i,j}^k~ \left[\hDel^k_{\ell+1} \cdot \th'\left(\hbtheta^k_{\ell, i}\right)\right], \quad \ell \in[L],~i\in[m_{\ell}]. \notag
\end{eqnarray}
Then, we have 
\begin{eqnarray}
  \hGrad^k_{\ell+1,i,j} &=& \frac{1}{N m_{\ell}}   \left[\hDel_{\ell+1,j}^{k}\right]^\top \th\left(\hbtheta_{\ell, i}^k\right), \quad \ell \in[ L], ~i \in [m_{\ell}], ~ j\in [m_{\ell+1}],\notag\\
 \hGrad^k_{1,i,j} &=& \frac{1}{N d }   \left[\hDel_{1,j}^{k}\right]^\top \hbtheta_{0,i}^k, \quad  ~i \in [d], ~ j\in [m_{1}]. \notag
\end{eqnarray}

\iffalse
\begin{remark}[Scaling]
We choose  different time-scales (step sizes) for the parameters. In detail, if we have $m_{\ell}=m$, then  $\hbw_1$ and $\hbw_{L+1}$ evolve with scale $\cO(\eta m)$ when we treat $d$ as a constant,   whereas  $\left\{\hbw_{\ell}\right\}_{\ell=2}^L$ evolve with $\cO(\eta m^2)$.  This setting matches the scale in the  gradient flow for   continuous DNN. To the best of our knowledge, existing works, e.g. \citet{araujo2019mean,nguyen2020rigorous},  also adopt such time-scales.
\end{remark}
\begin{remark}[Initialization]
Our framework  requires a different initialization for DNN.    Shown in Section \ref{sec:related}, the standard initialization strategy  scales the weights with magnitude around  $\sqrt{m_{\ell}}$ which goes to infinity with the growth of $m_{\ell}$. Hence we  preform an $\ell_2$-regression procedure to reduce the redundancy of the weights.
The details are presented  in Subsection \ref{subsection:app}.
\end{remark}
 \fi

\subsection{Neural Feature Flow for Training Continuous DNN}
We derive the evolution of the Gradient Descent algorithm on a  continuous  DNN  $\left(\left\{\w_{\ell}\right\}_{\ell=2}^{L+1}, \left\{p_{\ell}\right\}_{\ell=1}^L\right)$. When the step size $\eta$ goes to $0$,  both the weights and the features  are expected to move continuously through time.  We first introduce the notations for the trajectories of $\bw_1$, $\{w_{\ell}\}_{\ell=2}^L$, and $\{\btheta_{\ell}\}_{\ell=2}^{L}$:
\begin{itemize}
    \item $\Psi_{\ell}^{\btheta}:\mathrm{supp}(p_{\ell})\to C\left([0,T], \RR^N\right)$ is the trajectory of $\btheta_{\ell}$ for $\ell\in[2:L]$;
    \item $\Psi_{1}^{\bw}:\mathrm{supp}(p_1)\to C([0,T], \RR^d)$  and $ \Psi_{L+1}^{\bw}:\mathrm{supp}(p_L) \to C([0,T], \RR) $ are the trajectories of $\bw_1$ and $w_{L+1}$, respectively;
    \item $ \Psi_{\ell}^{\bw} :\mathrm{supp}(p_{\ell-1}) \times \mathrm{supp}(p_{\ell}) \to C([0,T], \RR)$ is the trajectory of $w_{\ell}$ for $\ell\in[2:L]$; 
    \item  Let  $\Psi$ be the collection of  these trajectories.
\end{itemize}
The continuous gradient for the weight can be obtained from the backward-propagation algorithm. Especially, we define
\begin{subequations}
\begin{align}
    \btheta_{L+1}\left(\Psi,t\right) &:=\int \Psi^{\bw}_{L+1}\left(\btheta_L\right)(t)~ \th\left(\Psi^{\btheta}_{L}\left(\btheta_L\right)(t)\right)   d p_{L}\left(\btheta_L\right),\label{def linedeltl+1}\\
      \lineDel_{L+1}(\Psi,t) &:=  \left[\phi'_1\left(\btheta_{L+1}^t(1), y^1\right),\phi'_1\left(\btheta_{L+1}^t(2), y^2\right),\dots, \phi'_1\left(\btheta_{L+1}^t(N), y^N\right)\right]^\top,\notag\\
     \lineDel_{L}\left(\btheta_{L}; \Psi,t\right) &:=  \left[\Psi^{\bw}_{L+1}\left(\btheta_{L})(t\right)~ \lineDel_{L+1}(\Psi,t)\right]   \cdot \th'\left(\Psi^{\btheta}_{L}\left(\btheta_L\right)(t)\right),\label{def linedel}\\
       \lineDel_{\ell}\left(\btheta_{\ell}; \Psi,t\right)  &:= \!\left[\int\!  \! \Psi^{\bw}_{\ell+1}\!\left(\btheta_{\ell},\btheta_{\ell+1}\right)\!(t)~\lineDel_{\ell+1}\left(\btheta_{\ell+1}; \Psi,t\right)  dp_{\ell+1}\!\left(\btheta_{\ell+1}\right)\!\right] \! \cdot  \! \th'\!\left(\!\Psi^{\btheta}_{\ell}\left(\btheta_\ell\right)(t)\!\right)\!,\label{def linedel1}\\
       \lineDel_{1}\left(\bw_1; \Psi,t\right)  &:= \! \left[\int  \Psi^{\bw}_{2}\left(\bw_1,\btheta_{2})(t)  ~ \lineDel_{2}(\btheta_{2}; \Psi,t\right)  dp_{2}\left(\btheta_{2}\right) \right] \cdot \th'\Big(\btheta_1\big(\Psi^\bw_1(\bw_1)(t)\big)\Big),\label{def linedel2}
\end{align}
\end{subequations}
where in \eqref{def linedel}, $\btheta_L\in\supp(p_L)$,  in \eqref{def linedel1}, $\ell\in[2:L-1]$ and $\btheta_\ell\in\supp(p_\ell)$, and in \eqref{def linedel2}, $\bw_1\in \supp(p_1)$ and $\btheta_1(\cdot)$ is defined by \eqref{cfor 1d}.  Then the  gradient of the weights can be written as below.
\begin{subequations}
\begin{align}
  \lineGrad^{\bw}_{L+1}\left(\btheta_{L};\Psi,t\right) &:= \frac{1}{N} \left[\lineDel_{L+1}(\Psi,t)\right]^\top \th\left(\Psi^{\btheta}_{L}\left(\btheta_L\right)(t)\right),\label{linegrad wL+1}\\  
  \lineGrad^{\bw}_{\ell}\left(\btheta_{\ell-1}, \btheta_{\ell} ;\Psi,t\right) &:=  \frac{1}{N}\left[\lineDel_{\ell}(\btheta_\ell;\Psi,t)\right]^\top \th\left(\Psi^{\btheta}_{\ell-1}\left(\btheta_{\ell-1}\right)(t)\right),\label{linegrad wL+11}\\ 
   \lineGrad^{\bw}_{2}\left(\bw_1, \btheta_{2} ;\Psi,t\right) &:=   \frac{1}{N}\left[\lineDel_{2}\left(\btheta_2;\Psi,t\right)\right]^\top \th\Big(\btheta_1\big(\Psi^\bw_1(\bw_1)(t)\big)\Big),\label{linegrad wL+12}\\
        \lineGrad^{\bw}_{1}\left(\bw_1;\Psi,t\right) &:= \frac{1}{N}\bX^\top\left[\lineDel_{1}\left(\bw_1;\Psi,t\right)\right], \label{linegrad wL+13}
\end{align}
\end{subequations}
where in \eqref{linegrad wL+1}, $\btheta_L\in\supp(p_L)$,  in \eqref{linegrad wL+11}, $\ell\in[3:L]$, $\btheta_{\ell-1}\in\supp(p_{\ell-1})$, and $\btheta_{\ell}\in\supp(p_{\ell})$, in \eqref{linegrad wL+12}, $\bw_1\in\supp(p_1)$ and $\btheta_{2}\in\supp(p_{2})$, and in \eqref{linegrad wL+13}, $\bw_1\in\supp(p_1)$.

Moreover, we expect that the features    satisfy the constraints:   
\begin{eqnarray}
     \int \Psi^{\bw}_2\left(\bw_1,\btheta_2\right)(t)~ \th\Big(\btheta_1\big(\Psi^\bw_1(\bw_1)(t)\big)\Big)  d p_{1}\left(\bw_1\right)  \!\!\!&=&\!\!\! \Psi^{\btheta}_2\left(\btheta_2\right)(t),\quad \btheta_2\in\supp(p_2), \notag\\
   \int \Psi^{\bw}_{\ell}\left(\btheta_{\ell-1},\btheta_{\ell}\right)\left(t\right)\th\left(\Psi^{\btheta}_{\ell-1}\left(\btheta_{\ell-1}\right)\left(t\right)\right) d p_{\ell-1}\left(\btheta_{\ell-1}\right)  \!\!\!&=&\!\!\! \Psi^{\btheta}_{\ell}\left(\btheta_{\ell}\right)(t), \quad \ell\in[3:L], ~\btheta_\ell\in\supp(p_\ell). \notag
\end{eqnarray}
So the drift term for the features can be obtained by the chain rule:
 \begin{subequations}
 \begin{align}
   \lineGrad^{\btheta}_{1}\left(\bw_1;\Psi,t\right) &:=  \frac{1}{d} \left[ \bX~ \lineGrad^{\bw}_{1}\left(\bw_1;\Psi, t\right) \right],\label{grad11}\\
 \lineGrad^{\btheta}_{2}\left( \btheta_{2} ;\Psi, t\right) &:= \int \Psi^{\bw}_{2}\left(\bw_1, \btheta_{2}\right)(t)\left[  \th'\Big(\btheta_1\big(\Psi^\bw_1(\bw_1)(t)\big)\Big)\cdot \lineGrad^{\btheta}_{1}\left(\bw_1; \Psi,t\right)\right]d p_{1}\left(\bw_1\right) \notag\\
  &\quad~ +\int \th\Big(\btheta_1\big(\Psi^\bw_1(\bw_1)(t)\big)\Big)\cdot \lineGrad_2^{\bw}\left(\bw_1, \btheta_{2};\Psi,t\right) d p_{1}\left(\bw_1\right),\label{grad1}\\  
    \lineGrad^{\btheta}_{\ell}\left( \btheta_{\ell} ;\Psi, t\right) &:= \int  \Psi^{\bw}_{\ell}\left(\btheta_{\ell-1}, \btheta_{\ell}\right)(t)\left[\th'\left(\Psi^{\btheta}_{\ell-1}\left(\btheta_{\ell-1}\right)(t)\right) \cdot \lineGrad^{\btheta}_{\ell-1}\left(\btheta_{\ell-1}; \Psi,t\right)\right]d p_{\ell-1}\left(\btheta_{\ell-1}\right)\notag\\
  &\quad~+ \int \th\left(\Psi^{\btheta}_{\ell-1}\left(\btheta_{\ell-1}\right)(t)\right) \cdot\lineGrad^{\bw}_{\ell}\left(\btheta_{\ell-1}, \btheta_{\ell};\Psi,t\right) d p_{\ell-1}\left(\btheta_{\ell-1}\right), \label{grad12}
 \end{align}
 \end{subequations}
 where in \eqref{grad11}, $\bw_1\in\supp(p_1)$, in \eqref{grad1}, $\btheta_2\in\supp(p_2)$, and in \eqref{grad12}, $\ell\in[3:L]$ and $\btheta_{\ell}\in\supp(p_\ell)$. % Note that  the symbol $ \lineGrad^{\btheta}_\ell(\btheta_{\ell};\Psi,t)$  is not the functional derivative of $L$ with respect to feature $\btheta_{\ell}$ but the drift term   in our dynamic equations.
 \iffalse In the following, for the convenience of our notations, we  introduce
\begin{eqnarray}
 &&\!\!\!\!\! \!\!\!\!\!\!\! \lineGrad(\bTheta; \Psi,t):\RR^D\times \mathbf{\Psi}\times [0, T] \to \RR^D \notag\\
 \!\!\!\!\! &:=&\!\!\!\!\! \Big(    \lineGrad^{\bw}_1( \bw_1;\Psi,t ),  ~ \lineGrad^{\bw}_2( \bw_1, \btheta_2;\Psi,t ),~ \lineGrad^{\btheta}_2(\btheta_2;\Psi,t ),  \ \notag\\
 &&\quad\quad \lineGrad^{\bw}_3( \btheta_2, \btheta_3;\Psi,t ), \lineGrad^{\btheta}_3(\btheta_3;\Psi,t ),~  \dots,  ~\lineGrad^{\btheta}_L(\btheta_L;\Psi,t ), ~\lineGrad^{\bw}_{L+1}( \btheta_L;\Psi,t )\Big). \notag
\end{eqnarray}
\fi
 Now we  define   the process of  a continuous DNN trained by Gradient Descent called neural feature flow, which characterizes the evolution  of  both  weights and  features.
\begin{mdframed}[style=exampledefault]
\begin{definition}[Neural Feature Flow for DNN]\label{NFLp0}
Given an initial continuous DNN represented by  $\left(\left\{\w_{\ell}\right\}_{\ell=2}^{L+1}, \left\{p_{\ell}\right\}_{\ell=1}^L\right)$ and $T<\infty$, we say a trajectory $\Psi_*$ is
a neural feature flow  if for all $t\in[0,T]$,
\begin{enumerate}[(1)]
    \item for all $\ell\in[2:L]$ and $\btheta_{\ell}\in\supp(p_\ell)$,   $$\Psi_{*,\ell}^{\btheta}\left(\btheta_{\ell}\right)(t) = \btheta_{\ell} -\int_{s=0}^t\lineGrad^{\btheta}_{\ell}\left( \btheta_\ell;\Psi_*,s\right)ds, $$ 
    \item for all  $\bw_1\in\supp(p_1)$,  $$\vspace{-0.1in}\Psi_{*,1}^{\bw}\left(\bw_1\right)(t) = \bw_1 -\int_{s=0}^t\lineGrad^{\bw}_{1}\left( \bw_1;\Psi_*,s\right)ds,$$
    \item for all  $\bw_1\in\supp(p_1)$ and  $\btheta_2\in\supp(p_2)$,  $$\Psi_{*,2}^{\bw}\left(\bw_1, \btheta_2\right)(t) = w_2\left(\bw_1, \btheta_2\right)-\int_{s=0}^t\lineGrad^{\bw}_{2}\left(\bw_1, \btheta_2;\Psi_*,s\right)ds,$$
    \item  for all $\ell\in[2:L-1]$, $\btheta_{\ell}\in\supp(p_{\ell})$,  and  $\btheta_{\ell+1}\in\supp(p_{\ell+1})$,
     $$\Psi_{*,\ell+1}^{\bw}\left(\btheta_{\ell}, \btheta_{\ell+1}\right)(t) = w_{\ell+1}\left(\btheta_\ell, \btheta_{\ell+1}\right)-\int_{s=0}^t\lineGrad^{\bw}_{\ell+1}\left( \btheta_{\ell}, \btheta_{\ell+1};\Psi_*,s\right)ds,$$ 
     \item for all $\btheta_{L}\in\supp(p_{L})$, $$\Psi_{*,L+1}^{\bw}\left(\btheta_{L}\right)(t) = w_{L+1}\left(\btheta_L\right) -\int_{s=0}^t\lineGrad^{\bw}_{L+1}\left( \btheta_L;\Psi_*,s\right)ds.$$
\end{enumerate}
\end{definition}
\end{mdframed}

%From Definition \ref{NFLp0}, we know the training process of a continuous DNN is described  by infinite systems of ODEs. we shall prove the existence and the uniqueness of $\Psi_*$ in the next section.  It  characterizes  how a continuous DNN evolves through time. 
% We call our evolution dynamic   as neural feature flow which characterizes the evolution  of  both  weights and  features.

\section{Analysis of Continuous DNN}\label{sec:theo DNN}

\subsection{Assumptions for DNN}
We first present our  assumptions.  
In the analysis we treat $N$ and $d$ as constants.
We emphasize that these assumptions are  mild and  can be satisfied in practice.

\begin{assumption}[Activation Function]\label{ass:1} We assume the activation function is bounded and has bounded and Lipschitz continuous gradient.  That is, there exists  constants $L_1\geq 0$, $L_2\geq 0$, and $L_3\geq0$, such that for all $x\in\RR$,
\begin{eqnarray}
       \left|h(x)\right| \leq  L_1,\quad\quad    \left|   h' (x) \right| \leq L_2,\notag
\end{eqnarray}
and for all $x\in\RR$ and $y\in\RR$, 
$$  \left|   h' (x) -  h'(y) \right|\leq L_3 |x-y |. $$
\end{assumption}

\begin{assumption}[Loss Function]\label{ass:2} We assume the loss function $\phi(\cdot;\cdot):\RR\times \mathcal{Y}\to \RR$ has    bounded and Lipschitz continuous gradient for the first argument.   That is, there exists  constants $L_4\geq0$ and $L_5\geq0$, such that for all $y\in\mathcal{Y}$, and $x_1 \in\RR$, 
\begin{eqnarray}
        \left|   \phi'_1 (x_1,y) \right| \leq L_4,\notag
\end{eqnarray}
and for all $x_2\in\RR$, 
$$  \left|   \phi'_1 (x_1,y)  - \phi'_1 (x_2,y) \right|\leq L_5 |x_1-x_2 |. $$
\end{assumption}
Assumptions \ref{ass:1} and \ref{ass:2} only  require certain smoothness and boundedness of the loss and activation functions.  In the following, we propose the conditions for the initial continuous DNN  $\left(\left\{\w_{\ell}\right\}_{\ell=2}^{L+1}, \left\{p_{\ell}\right\}_{\ell=1}^L\right)$.
In Subsection \ref{subsection:app}, we will consider concrete examples that realize these assumptions. 

\iffalse
\begin{assumption}[Initialization on $\left\{p_{\ell}\right\}_{\ell=1}^L$]\label{ass:4}
We assume for all $\ell\in[L]$, $p_{\ell}$ is a sub-gaussian distribution\footnote{ This paper focuses on high probability results.  To  obtain  constant probability results, Assumption \ref{ass:4}  can be relaxed to require  a bounded $(2+\alpha)$-th moment, where $\alpha>0$. We also note that the definition in Assumption \ref{ass:4}   
on multivariate  sub-gaussian distribution slightly differs from the standard one, which states that  a random variable $\bxi\in\RR^{d_1}$ follows a   sub-gaussian distribution with parameter $\bar{\sigma}$ if  for  all $\mathbf{c}\in \mathbb{S}^{d_1-1}$,  $\EE_{\bxi}\left[ \exp\left( \left(\mathbf{c}^\top \bxi\right)^2/\bar{\sigma}^2\right)\right]\leq e$. It is straightforward to verify that the  two definitions  are  equivalent  in the sense that  $ \sqrt{C_{\sigma}d_1}  \bar{\sigma}\leq \sigma \leq  \sqrt{C_{\sigma}d_1} \bar{\sigma}$, where  $C_{\sigma}$ is an universal constant. See Appendix \ref{}.  
}.  In detail,  we assume that there exist a constant $\sigma>0$, such that 
   $$    \E_{p_{1}}\left[\exp\left(\left\|\bw_1\right\|_{\infty}^2/\sigma^2\right)\right] \leq e, \quad \quad    \E_{p_{\ell}}\left[\exp\left(\left\|\btheta_\ell\right\|^2_{\infty}/\sigma^2\right)\right] \leq e,\quad \ell \in[2:L].\notag$$
\end{assumption}
%Assumption \ref{ass:4} is not hard to be satisfied. In fact,  in Appendix, we will show that the standard initialization strategy, e.g. \cite{}  also guarantees Assumption \ref{ass:3} under mild conditions. 
\fi

\begin{assumption}[Initialization on $\left\{p_{\ell}\right\}_{\ell=1}^L$]\label{ass:4}
We assume for all $\ell\in[L]$, $p_{\ell}$ is  $\sigma$-sub-gaussian\footnote{ This paper focuses on high probability results. To obtain constant probability results,   Assumption \ref{ass:4} can be relaxed to as that $p_{\ell}$ has a bounded $(2+\alpha)$-th moment for $\alpha>0$.
}. \iffalse Formally,  there exist a constant $\sigma>0$, such that  for all $\bc_1\in  \mathbb{S}^{d-1}$ and $\bc_2\in  \mathbb{S}^{N-1}$, 
$$ \EE_{\bw_1\sim p_1}\left[ \exp\left( \left(\bc_1^\top \bw_1\right)^2/\sigma^2\right)\right]\leq e, \quad\quad  \EE_{\btheta_{\ell}\sim p_{\ell}}\left[ \exp\left( \left(\bc_2^\top \btheta_{\ell}\right)^2/\sigma^2\right)\right]\leq e,\quad \ell\in[2:L].  $$\fi
\end{assumption}

\begin{assumption}[Initialization on $\{w_{\ell}\}_{\ell=2}^{L+1}$]\label{ass:3}
%Under Assumption \ref{ass:4},  
We assume that, for all $\ell\in[2:L]$, $w_{\ell}(\cdot, \cdot)$ have a sublinear growth on the second argument. In other words, there is  a constants $C_1\geq0$, such that 
\begin{eqnarray}
    \left|w_{2}\left(\bw_1,\btheta_2\right)\right| \!\!\!&\leq&\!\!\! C_1\left(  1+ \left\|\btheta_2\right\|_{\infty} \right) , \quad \text{for all}~~ \bw_1\in\supp(p_1) ,\btheta_2\in \supp(p_2), \label{ass41}\\
        \left|w_{\ell}\left(\btheta_{\ell-1},\btheta_{\ell}\right)\right|\!\!\!&\leq&\!\!\! C_1 \left(1+ \left\|\btheta_{\ell}\right\|_{\infty}\right), \quad \text{for all}~~ \btheta_{\ell-1}\in \supp(p_{\ell-1}) ,\btheta_\ell\in \supp(p_{\ell}),  \ell \in [3: L]\label{ass42}.
\end{eqnarray}
Moreover, we assume that $w_{\ell}(\cdot, \cdot)$  are locally Lipschitz continuous where their Lipschitz constants have a sub-linear growth on the second argument. In detail, there is a  constant $C_2\geq0$ such that 
for all $\bw_1 \in\supp(p_{1})$, $\bar{\bw}_1 \in\supp(p_{1}) \cap \mathcal{B}_{\infty}(\bw_1,1)$, $\btheta_\ell \in \supp(p_{\ell})$, and $\bar{\btheta}_\ell \in \supp(p_{\ell})\cap \mathcal{B}_{\infty}(\btheta_{\ell},1)$ with $\ell\in[2:L]$, we have
\begin{eqnarray}
\big|w_{2}(\bw_1,\btheta_2) -w_{2}(\bar{\bw}_1,\bar{\btheta}_2) \big| &\leq& C_2\left(  1+ \left\|\btheta_2\right\|_{\infty} \right)\big(\big\|\bw_1 -\bar{\bw}_1\big\|_{\infty}+\big\|\btheta_2 -\bar{\btheta}_2\big\|_{\infty}\big),\label{ass43}\\  
\big|w_{\ell}(\btheta_{\ell-1},\btheta_\ell) -w_{\ell}(\bar{\btheta}_{\ell-1},\bar{\btheta}_\ell) \big| &\leq&C_2\left(  1+ \left\|\btheta_{\ell}\right\|_{\infty} \right)\big(\big\|\btheta_{\ell-1} -\bar{\btheta}_{\ell-1}\big\|_{\infty}+\big\|\btheta_\ell -\bar{\btheta}_\ell\big\|_{\infty}\big).\label{ass44}
\end{eqnarray}
For the last layer, we assume that $w_{L+1}$ is uniformly bounded  and is  Lipschitz continuous on $\btheta_{L}$, namely, there exist constants $C_3\geq0$ and $C_4\geq0$ such that for all $\btheta_{L}\in\RR^N$ and $\bar{\btheta}_{L}\in\RR^{N}$, we have
\begin{eqnarray}
   \big|w_{L+1}(\btheta_L)\big|\leq C_3 \quad\text{and}\quad  \big|w_{L+1}(\btheta_{L}) -w_{L+1}(\bar{\btheta}_L) \big| &\leq& C_4\big\|\btheta_{L} -\bar{\btheta}_{L}\big\|_{\infty}.      \label{ass45}
\end{eqnarray}
\end{assumption}

\subsection{Properties of Neural Feature Flow for DNN}
We first analyze the neural feature flow.  The following theorem guarantees the existence and uniqueness.
\begin{theorem}[Existence and Uniqueness of  Neural Feature Flow on DNN]\label{theorm:flow p0}
Under Assumptions \ref{ass:1}~--~\ref{ass:3}, for any $T<\infty$,   there exists an 
 unique neural feature flow  $\Psi_*$. 
\end{theorem}

Moreover, we show that $\Psi_*$ is a homotopy that continuously transforms a continuous DNN from state $\Psi(\cdot)(0)$ to  $\Psi(\cdot)(t)$ where $t\in[0,T]$. 
The continuity in time is due to the finite gradients (see Lemma \ref{lmm:Psi-property}); the continuity in features in given by following theorem:
% Moreover, we show that $\Psi_*$ is a \red{homotopy} that continuously transforms a continuous DNN from state $\Psi(\cdot)(0)$ to  $\Psi(\cdot)(t)$ where $t\in[0,T]$. Especially,\red{
% we have the theorem  below.}
\begin{theorem}[Property of $\Psi_*$]\label{theorem:conpsi}
Under Assumptions \ref{ass:1} -- \ref{ass:3},  let $\Psi_*$ be the  neural feature flow, there  are constants $C\geq0$ and $C'\geq0$ such that  for all $t\in[0,T]$,  $\bw_1\in\supp(p_1)$, $\bbw_1\in \supp(p_{1}) \cap \mathcal{B}_{\infty}(\bw_1,1) $, $\btheta_{\ell}\in\supp(p_\ell) $, and $\bbtheta_{\ell}\in \supp(p_{\ell}) \cap \mathcal{B}_{\infty}(\btheta_{\ell},1) $ with $\ell\in[2:L]$, we have
\begin{align*}
     \left\|\Psi^{\btheta}_{*,\ell}\left(\btheta_{\ell}\right)(t)-\Psi^{\btheta}_{*,\ell}\left(\bbtheta_{\ell}\right)(t) \right\|_{\infty}&\leq  Ce^{C' t}( \|\btheta_\ell\|_{\infty}+1 )\|\btheta_\ell -\bbtheta_{\ell} \|_{\infty},~~ \ell\in[2:L],\\
         \left\|\Psi^\bw_{*,1}(\bw_1)(t)-\Psi^\bw_{*,1}(\bbw_1)(t) \right\|_{\infty}&\leq  Ce^{C' t}( \|\bw_1\|_{\infty}+1 )\|\bw_1 -\bbw_1 \|_{\infty},\\
  \left|\Psi^{\bw}_{*,}(\bw_1, \btheta_2)(t) - \Psi^{\bw}_{*,2}(\bbw_1, \btheta_2)(t)  \right|&\leq C e^{C't} (\|\bw_1\|_{\infty}+ \|\btheta_2\|_{\infty}+1) \|\bw_1-\bbw_1 \|_{\infty},\\
    \left|\Psi^{\bw}_{*,2}(\bw_1, \btheta_2)(t) - \Psi^{\bw}_{*,2}(\bw_1, \bbtheta_2)(t)  \right|&\leq  Ce^{C't} ( \|\bw_1\|_{\infty}+ \|\btheta_2\|_{\infty}+1) \|\btheta_2-\bbtheta_2 \|_{\infty},\\
  \left|\Psi^{\bw}_{*,\ell}(\btheta_{\ell-1}, \btheta_\ell)(t) - \Psi^{\bw}_{*,\ell}(\bbtheta_{\ell-1}, \btheta_\ell)(t)  \right|&\leq Ce^{C't}( \|\btheta_{\ell-1}\|_{\infty}+ \|\btheta_{\ell}\|_{\infty}+1)  \|\btheta_{\ell-1}-\bbtheta_{\ell-1} \|_{\infty},~~ \ell\in[3:L],\\
   \left|\Psi^{\bw}_{*,\ell}(\btheta_{\ell-1}, \btheta_\ell)(t) - \Psi^{\bw}_{*,\ell}(\btheta_{\ell-1}, \bbtheta_\ell)(t)  \right|&\leq Ce^{C't}( \|\btheta_{\ell-1}\|_{\infty}+ \|\btheta_{\ell}\|_{\infty}+1)  \|\btheta_{\ell}-\bbtheta_{\ell} \|_{\infty},~~ \ell\in[3:L],\\
    \left|\Psi^{\bw}_{*,L+1}( \btheta_{L})(t) - \Psi^{\bw}_{*,L+1}( \bbtheta_{L})(t)  \right|&\leq C e^{C't}(  \|\btheta_L\|_{\infty}+1 )  \|\btheta_{L}-\bbtheta_{L} \|_{\infty}.
\end{align*}
\end{theorem}
The proofs of Theorems \ref{theorm:flow p0} and \ref{theorem:conpsi}  follow from the standard technique of Picard iterations (see, e.g., \cite{hartman1964ordinary}) with a special consideration on the search space to deal with the unboundedness of parameters.  The latter  differs from the former  by introducing  a more restrictive space in which all the candidates  satisfy the desired property. %All the proofs of this paper are left in the Appendix.  

\subsection{Approximation Using Finite Neurons for  DNN}\label{subsection:app}
We show that  the process of  a discrete DNN  trained by scaled Gradient Descent   can be approximated by the neural feature flow under suitable conditions.  
In the discrete DNNs, although the connecting weights are independently initialized, the features $\hbtheta_{\ell,j}$ are not mutually independent since they all depend on a \emph{common} set of random outputs from the previous layer. 
Our key observation is that $\hbtheta_{\ell,j}$ are \emph{almost} independent when the width $m$ of the hidden layers are sufficiently large; namely, there exist $\linetheta_{\ell,j}$ that are mutually independent such that the differences $\|\hbtheta_{\ell,j}-\linetheta_{\ell,j}\|_{\infty}$ are vanishing with $m$. 
This allows us to construct an ideal process to approximate the actual trajectory of the discrete DNN. 
For a precise statement, we first introduce the following concept of $\ep_1$-independent initialization.
% In the following, we  propose a general initial condition for the  discrete DNN. %Later, we will give an example to show that the  condition can be realized in practice. 

%\begin{mdframed}[style=exampledefault]
\begin{definition}[$\ep_1$-independent initial DNN]\label{DNN condition}
We say an initial discrete DNN $(\hbw,\hbtheta)$ is $\ep_1$-independent if there exist  a continuous DNN denoted by $(\left\{w_\ell\right\}^{L+1}_{\ell=2}, \left\{p_{\ell}\right\}_{\ell=1}^L)$ satisfying Assumptions \ref{ass:4} and \ref{ass:3} and $(\linebw, \linetheta)$  such that
\begin{enumerate}[(1)]
    \item   $\linebw_{1,i}\sim p_1$ for $i\in[m_1]$, $\linetheta_{\ell,i}\sim p_{\ell}$ for $\ell\in[2:L]$ and $i\in[m_{\ell}]$, and they are all mutually independent; 
    \item For the weights $\linew_\ell$ for $\ell\ge 2$,
    \begin{itemize}
        \item     $\linew_{2,i,j} = w_{2}\left(\linebw_{1,i}, \linetheta_{2,j}\right)$ for $i\in[m_1]$ and $j\in[m_2]$;
        \item  $\linew_{\ell+1,i,j} = w_{\ell+1}\left(\linetheta_{\ell,i}, \linetheta_{\ell+1,j}\right)$ for $\ell\in[2:L-1]$,  $i\in[m_{\ell}]$, and $j\in[m_{\ell+1}]$;
        \item $\linew_{L+1,i,1} = w_{L+1}\left(\linetheta_{L,i}\right)$ for $i\in[m_{L}]$;
    \end{itemize}
    \item $\ep_1$-closeness:
    \begin{itemize}
        \item $\left\| \hbw_{1,i} -\linebw_{1,i}   \right\|_{\infty} \leq  ( 1+  \left\| \linebw_{1,i}\right\|_{\infty})~\ep_1$ for  $i\in[m_1]$;
        \item $\left|\hw_{2, i,j}- \linew_{2,i, j}\right| \leq \big(1+ \left\|\linebw_{1,i}\right\|_{\infty} + \left\|\linetheta_{2,j}\right\|_{\infty} \big)~\ep_1$ for  $i\in[m_{1}]$ and $j\in[m_{2}]$;
        \item $\left|\hw_{\ell+1, i,j}- \linew_{\ell+1,i, j}\right| \leq (1+ \left\|\linetheta_{\ell,i}\right\|_{\infty} + \left\|\linetheta_{\ell+1,j}\right\|_{\infty} )~\ep_1$ for $\ell\in[2:L-1]$, $i\in[m_{\ell}]$, and $j\in[m_{\ell+1}]$;
        \item $\left|\hw_{L+1, i,1}- \linew_{L+1,i, 1}\right| \leq  \big(1+ \left\|\linetheta_{L,i}\right\|_{\infty}  \big)~\ep_1$ for $i\in [m_{L+1}]$.
    \end{itemize}
\end{enumerate}

\end{definition}
%\end{mdframed}

We show that scaled Gradient Descent from an $\ep_1$-independent initialization can be well-approximated by the corresponding  neural feature flow when the number of hidden units is $\tilde{\Omega}(\ep_1^{-2})$, where  $\tilde{\Omega}$ hides  poly-logarithmic factors.   This resembles a ``propagation of chaos'' argument \cite{sznitman1991topics}.  We compare the scaled Gradient Descent with an ideal discrete process determined by $\Psi_*$, the trajectory of the continuous DNN $(\left\{w_\ell\right\}^{L+1}_{\ell=2}, \left\{p_{\ell}\right\}_{\ell=1}^L)$.   
Specifically, we compare the following two processes:
\begin{itemize}
    \item Actual process \label{real} $(\hbw^{[0:K]},\hbtheta^{[0:K]})$ by executing  Algorithm \ref{algo:GD} in $K=\frac{T}{\eta}$ steps from $(\hbw,\hbtheta)$; 
    \item Ideal process \label{ideal} $\left(\linebw^{[0,T]},\linetheta^{[0,T]}\right)$ that evolves as the neural feature flow:
    \begin{align}
       &~~~~~~~ \linetheta_{\ell,i}^t\!\!\!\!\!\!\!\!\! \!\!&&=&& \!\!\!\!\!\!\!\!\! \Psi_{*,\ell}^{\btheta}\left(\linetheta_{\ell,i}\right)(t), \quad \ell\in [2:L],~ i\in [m_{\ell}], ~t\in[0,T],\notag\\
           &~~~~~~~   \linebw_{1,i}^t\!\!\!\!\! \!\!\!\!\!\!&&=&&\!\!\!\!\!\!\!\!\! \Psi_{*,1}^{\bw}\left(\bbw_{i}\right)(t), \quad i\in [m_{1}],  ~t\in[0,T], \notag\\
        &~~~~~~~ \linew_{2,i,j}^t\!\!\!\!\! \!\!\!\!\!\!&&=&&\!\!\!\!\!\!\!\!\! \Psi_{*,2}^{\bw}\left(\linebw_{1,i}, \linetheta_{2,j}  \right)(t),  \quad  i\in[m_1],~ j\in [m_2], ~t\in[0,T],\notag\\
     &~~~~~~~ \linew_{\ell+1,i,j}^t\!\!\!\!\! \!\!\!\!\!\!&&=&&\!\!\!\!\!\!\!\!\! \Psi_{*,\ell+1}^{\bw}\left(\linetheta_{\ell,i}, \linetheta_{\ell+1,j}  \right)(t),  \quad \ell \in [2:L-1],~ i\in[m_{\ell}],~ j\in [m_{\ell+1}], ~t\in[0,T],\notag\\
   &~~~~~~~ \linew_{L+1,i,1}^t\!\!\!\!\! \!\!\!\!\!\!&&=&&\!\!\!\!\!\!\!\!\! \Psi_{*,L+1}^{\bw}\left(\linetheta_{\ell,i}  \right)(t),~~i\in [m_{L}], ~t\in[0,T]. \notag
    \end{align}
\end{itemize}
We also compare the losses of the discrete DNN $\hat{\bL}^{k}:= \frac{1}{n}\sum_{n=1}^N\phi(\hbtheta_{L+1,1}^{k}(n), y^n )$ and the loss of the neural feature flow $\bL^{t}:=\frac{1}{N}\sum_{n=1}^n\phi\left(\btheta_{L+1}(\Psi_*,t)(n), y^n \right)$.

\begin{theorem}\label{theorm:app}
Under Assumptions \ref{ass:1} and \ref{ass:2},
suppose $\ep_1\leq \cO(1)$, $m_{\ell}=m\ge \tilde{\Omega}(\ep_1^{-2})$ for $\ell\in[L]$, and treat the parameters in assumptions and $T$ as constants.
Consider the actual process from an $\ep_1$-independent initialization in Definition \ref{DNN condition} with step size $\eta \leq \tO(\ep_1) $.
Then, the following holds with probability $1-\delta$:
\begin{itemize}
\item The two processes are close to each other:
\begin{align*}
&\sup_{k\in[0:K]}\bigg\{\sup_{ i\in [m]}\left\| \hbw^k_{1,i} -\linebw^{k\eta}_{1,i} \right\|_{\infty}  ,~\sup_{\ell\in[2:L], ~i\in [m]}\left\| \hbtheta^k_{\ell,i} -\linetheta^{k\eta}_{\ell,i} \right\|_{\infty}    \bigg\}\leq \tO(\ep_1),\\
&\sup_{k\in[0:K]}\bigg\{\sup_{ \ell\in[2:L], ~i, j\in[m]}\left| \hw^k_{\ell,i,j} -\linew^{k\eta}_{\ell,i,j} \right|,  ~\sup_{ i\in [m]}\left| \hw^k_{L+1,i,1} -\linew^{k\eta}_{L+1,i,1} \right| \bigg\}\leq \tO(\ep_1),
\end{align*}
\item The training losses are also close to each other:
\[
\sup_{k\in[0:K]}\left|\hat{L}^k -  L^{k\eta}\ \right|  \leq \cO(\ep_1),
\]
\end{itemize}
where $\tO$ and $\tilde{\Omega}$ hide poly-logarithmic factors on $\ep_1$ and $\delta$.
\end{theorem}

\begin{algorithm}[tb]
	\caption{Initializing  a  Discrete  DNN. }
	\label{algo:init}
		\begin{algorithmic}[1]
		\STATE Input the data $\{\hbtheta_{0,i}\}_{i=1}^d$ in \eqref{dataf},   variance $\sigma_1>0$, and a constant $C_3$.
		\STATE  Independently  draw $\hat{w}_{1,i,j}\sim p_0= \mathcal{N}\left(0, d\sigma^2_1\right)$ for $i\in[d]$ and $j\in[m]$.  
	\STATE  Set $\hbtheta_{1,j} = \frac{1}{d}\sum_{i=1}^d\hat{w}_{1,i,j}~\hbtheta_{0,i}$ where $j\in[m]$. {\hfill $\diamond$ Standard Initialization for  layer $1$ }
	 \FOR{$\ell =2,\dots, L$}
	\STATE    Independently draw $\tilde{w}_{\ell,i,j}\sim \mathcal{N}\left(0, m\sigma^2_1\right)$ for $i,j\in[m]$.
	\STATE   Set $\hbtheta_{\ell,j} = \frac{1}{m}\sum_{i=1}^m\tilde{w}_{\ell,i,j}~\th(\hbtheta_{\ell-1,i})$ where $j\in[m]$. {\hfill $\diamond$ Standard Initialization for  layer $\ell$ }
	 \ENDFOR
	  \STATE Set $\hw_{L+1,i,1} = C_3$ where $i\in[m]$. {\hfill $\diamond$   Simply initialize $\left\{\hw_{L+1,i,1}\right\}_{i=1}^m$ by a constant}
	  \FOR{$\ell =2,\dots, L$} 
	   	 \FOR{$j =1,\dots, m$} 
	   			\STATE  Solve convex optimization problem: {\hfill $\diamond$ Perform $\ell_2$-regression to reduce redundancy}
	  	\begin{equation}\label{ell2regression}
          \min_{\left\{\hw_{\ell,i,j}\right\}_{i=1}^{m}}~ \frac{1}{m}\sum_{i=1}^{m} \left(\hw_{\ell,i,j} \right)^2, ~~~~ \text{s.t.}~~  \hbtheta_{\ell,j} = \frac{1}{m}\sum_{i=1}^m\hw_{\ell,i,j}~\th(\hbtheta_{\ell-1,i}).
           \end{equation}
    \ENDFOR
	    \ENDFOR
	    \STATE Similarly to \eqref{pad w},  pad all the weights into a single vector denoted as $\hbw$. 
\STATE Similarly to \eqref{pad feature},  pad all the features into a single vector denoted as $\hbtheta$. 
\STATE Output the discrete DNN parameters $(\hbw, \hbtheta)$.
	\end{algorithmic}
\end{algorithm}

In the following,  we show that the  standard initialization \cite{glorot2010understanding,he2015delving} followed by a simple  $\ell_2$-regression procedure achieves the $\ep_1$-independence in Definition \ref{DNN condition}. 
The algorithm  is shown  in Algorithm \ref{algo:init}. 
Note that the  standard initialization strategy scales the weights as $\sqrt{m}$, which diverges in the mean-field limit.  Hence, we perform the simple $\ell_2$-regression to reduce the redundancy of the weights while preserving all initial features\footnote{
In Algorithm \ref{algo:init}, the weights in the last layer $\{w_{L+1,i,1}\}_{i=1}^m$  can also be initialized by the standard initialization followed by an $\ell_2$-regression.  
The $\ell_2$-regression \eqref{ell2regression} can be replaced by a soft version 
$$           \min_{\{\hw_{\ell,i,j}\}_{i=1}^{m}}~ \frac{\lambda_m}{m}\sum_{i=1}^{m} \left(\hw_{\ell,i,j} \right)^2 +  \left\|\hbtheta_{\ell,j} - \frac{1}{m}\sum_{i=1}^m\hw_{\ell,i,j}~\th(\hbtheta_{\ell-1,i})\right\|^2.$$
% To Theorem \ref{theorm:appres}, we can set  $\lim_{m\to\infty} a_m \geq \Omega(\ep^{-1}) $.  
}.  
% We have the theorem below.

\iffalse
In fact, given the  continuous DNN $\left(\left\{w_\ell\right\}^{L+1}_{\ell=2}, \left\{p_{\ell}\right\}_{\ell=1}^L\right)$,   if we allow sampling from distributions $\left\{p\right\}_{\ell=1}^L$ and accessing the functions $\left\{w_\ell\right\}_{\ell=2}^{L}$, one  can construct a feasible discrete DNN  from Algorithm \ref{algo:discrete}.  

On the other hand, in most cases,  however, we have little prior knowledge on  the underlying continuous DNN $\left(\left\{w_\ell\right\}^{L+1}_{\ell=2}, \left\{p_{\ell}\right\}_{\ell=1}^L\right)$ at beginning.  In these cases,  we propose  an special initialization method shown in Algorithm \ref{algo:init}. One can find that Algorithm \ref{algo:init} first initializes the DNN by a standard initialization strategy \cite{araujo2019mean,nguyen2020rigorous}, and then learning a simple $\ell_2$ regression to reduce the redundancy of the weights.  We have the theorem below.
\fi

\begin{theorem}\label{ini lemma}
% Under Assumption \ref{ass:1}, for a Gaussian process, define the $N\times N$  dimensional Gram matrices $\left\{\bK_{\ell}\right\}_{\ell=0}^{L}$ defined as
Define a sequence of Gram matrices $\left\{\bK_{\ell}\right\}_{\ell=0}^{L}\in \RR^{N\times N}$ as, for $i,j\in[N]$,
\begin{equation}
    \begin{split}
      \bK_0(i,j) &:= \frac{1}{d}\<\bx^i,\bx^j  \>,\\
\mathbf{\Sigma}_{\ell,i,j} &:= \left(
 \begin{matrix}
   \bK_{\ell}(i,i) & \bK_{\ell}(i,j)   \\
   \bK_{\ell}(j,i)  & \bK_{\ell}(j,j)  
  \end{matrix}
  \right)\in \RR^{2\times2},\quad \ell\in[L-1], \\
\bK_{\ell+1}(i,j) &:= \E_{(u,v)~\mathcal{N}\left(\mathbf{0}^2,\sigma_1^2\mathbf{\Sigma}_{\ell,i,j}  \right)} \left[    h(u)h(v)\right],~\quad \ell\in[L-1]. \label{kdef}  
    \end{split}
\end{equation}
Under Assumption \ref{ass:1},  suppose  $\blambda := \min_{\ell=1}^{L-1}\left\{\lambda_{\min}\left(\bK_{\ell}\right)\right\}>0$,
% \footnote{It holds for all analytic non-polynomial $h$ and non-parallel $\bx_i$; see \cite[Section F.2]{du2018gradient}.
% }, 
and treat the parameters in assumptions and $\blambda$ as constants.
% If $\ep_1\leq \cO(1)$ and $m\geq \tilde{\Omega}(\frac{1}{\ep_1^2})$, then, 
With probability  at least $1-\delta$, Algorithm \ref{algo:init} produces an $\ep_1$-independent initialization with $\ep_1\le \tO(\frac{1}{\sqrt{m}})$.
% that satisfies Assumptions \ref{ass:4} and \ref{ass:3}.
% Then there is a continuous DNN  $\left(\left\{w_\ell\right\}^{L+1}_{\ell=2}, \left\{p_{\ell}\right\}_{\ell=1}^L\right)$ that satisfies Assumptions \ref{ass:4} and \ref{ass:3} and   an ideal discrete DNN that satisfies initial condition in Definition \ref{DNN condition}. Moreover,    treat the problem-dependent parameters, i.e.,  $L_1$, $L_2$, $L_3$, $N$, $d$, and $\bar{\lambda}$ as  constants. For $\ep_1\leq \tO(1)$ and $\delta\leq 1$, with probability  at least $1-\delta$,    Algorithm \ref{algo:init} produces a discrete DNN that satisfies the initial condition in Definition \ref{DNN condition} when $m\geq \tilde{\Omega}(\ep^{-2})$.
\end{theorem}
In addition to Assumption  \ref{ass:1}, Theorem \ref{ini lemma} further requires that  the least eigenvalues  of the Gram matrices  are stricly positive.   It is shown  in the lazy training studies, e.g. \cite[Lemma F.1]{du2018gradient}, that the assumption holds for all analytic non-polynomial $h$.

\subsection{Convexify Continuous DNN}\label{sec:convex}
Problem \eqref{problem1} is  non-convex. Inspired by \cite{fang2019convex}, we show that it can be re-parameterized as a convex optimization  under suitable conditions.  
We consider the regularization term $R$ of the form
$$R\left( \left\{w_{\ell}\right\}_{\ell=2}^{L+1}, \left\{p_{\ell}\right\}_{\ell=1}^L\right)  =   \sum_{\ell=2}^{L}  \lambda^{w}_{\ell} {R}_\ell^{w}( w_{\ell}, p_{\ell-1},p_{\ell})+   \lambda^{w}_{L+1} {R}_{L+1}^{w}( w_{L+1}, p_{L}) + \sum_{\ell=1}^{L} \lambda^{p}_{\ell} R_\ell^{p}\left(p_{\ell}\right), $$
where $\left\{{R}_\ell^{w}\right\}_{\ell=2}^{L+1}$ and    $
\left\{{R}_\ell^{p}\right\}_{\ell=1}^L$ are regularizes imposed on the $\left\{w_{\ell}\right\}_{\ell=2}^{L+1}$ and $\left\{p_{\ell}\right\}_{\ell=1}^L$, respectively,  and $\lambda^{w}_{\ell}, \lambda^{p}_{\ell}\ge 0$.    Moreover,   suppose  $\left\{R_\ell^{w}\right\}_{\ell=2}^{L+1}$  are in form as
\begin{equation}\label{rw}
    R_\ell^{w} :=
\begin{cases}
\int\left[\int  \left|w_{\ell}(\bw_1, \btheta_{2})\right| d p_{2}(\btheta_{2}) \right]^r d p_{1}(\bw_1),&\quad \ell=2, \\
\int\left[\int  \left|w_{\ell}(\btheta_{\ell-1}, \btheta_{\ell})\right| d p_{\ell}(\btheta_{\ell}) \right]^r d p_{\ell-1}(\btheta_{\ell-1}),&\quad \ell=[3:L], \\
\int\left[w_{L+1}(\btheta_{L})\right]^r d p_{L}(\btheta_{L}),&\quad \ell=L+1, \\
\end{cases}
\end{equation}
where $r\geq 1$.  For all $\ell\in[L]$, if  $p_{\ell}$ are equivalent to  Lebesgue measure,     denoting  $\ddp_{\ell}$ as the probability density function of $p_{\ell}$,  we can do a change of variables as
\begin{eqnarray}
\ttw_2(\bw_1,\btheta_2 ) &=& w_2\left(\bw_1,\btheta_2\right) \ddp_1\left(\bw_1\right)\ddp_2\left(\btheta_2\right),\notag\\
\ttw_{\ell+1}(\btheta_\ell,\btheta_{\ell+1} ) &=& w_{\ell+1}\left(\btheta_\ell,\btheta_{\ell+1} \right) \ddp_{\ell}\left(\btheta_\ell\right)\ddp_{\ell}\left(\btheta_{\ell+1}\right), \quad \ell \in[2:L-1],\notag\\
\ttw_{L+1}(\btheta_L ) &=& w_{L+1}\left(\btheta_L \right) \ddp_{L}\left(\btheta_L\right), \notag
\end{eqnarray}
and  rewrite Problem \eqref{problem1} as
\begin{eqnarray}\label{problem2}
\minimize_{\{\ttw_{\ell}\}_{\ell=2}^{L+1},  \{p_{\ell}\}_{\ell=1}^L }    && \frac{1}{N}\sum_{n=1}^N \phi\left( \btheta_{L+1}(n), y^n \right) + \sum_{\ell=2}^{L+1} \lambda^w_{\ell} {R}_\ell^{\ttw}\left( \ttw_{\ell}, p_{\ell-1}\right)+ \sum_{\ell=1}^{L}\lambda^p_{\ell}{R}_\ell^{p}\left(p_{\ell}\right)  \\
\text{s.t.}  &&   \int \ttw_2\left(\bw_1,\btheta_2\right) \th\left(\btheta_1(\bw_1)\right)   d \bw_1 = \ddp_2\left(\btheta_2\right) \btheta_2,\quad   \btheta_2\in \RR^N,\notag\\
&&   \int \ttw_{\ell}\left(\btheta_{\ell-1},\btheta_{\ell}\right)\th\left(\btheta_{\ell-1}\right)  d\btheta_{\ell-1} = \ddp_{\ell}\left(\btheta_{\ell}\right)\btheta_{\ell}, \quad  \btheta_\ell\in \RR^N, ~\ell\in[3:L],\notag\\
&&     \int \ttw_{L+1}\left(\btheta_{L}\right)\th\left(\btheta_{L}\right)d\btheta_{L} = \btheta_{L+1}, \notag
\end{eqnarray}
where 
$$ {R}_\ell^{\ttw}( \ttw_{\ell}, p_{\ell-1}) =  \int \frac{\left(\int \left|\ttw_{\ell}( \btheta_{\ell-1},\btheta_{\ell} )\right|  d\btheta_{\ell}\right)^r }{\left(\ddp_{\ell-1}(\btheta_{\ell-1})\right)^{r-1}}    d\btheta_{\ell}\quad \ell \in[2:L], $$
and 
$$ {R}_{L+1}^{\ttw}( \ttw_{L+1}, p_{L}) =  \int \frac{ \left|\ttw_{L+1}( \btheta_{L} )\right|^r }{\left(\ddp(\btheta_{L})\right)^{r-1}}    d\btheta_{L}. $$
The theorem below demonstrates the convexity of  Problem \eqref{problem2}.
\begin{theorem}\label{theo:con}
Assume $\phi(\cdot;\cdot)$ is convex on the first argument,
${R}_\ell^{w}$ is in form of \eqref{rw} for $\ell\in[2:L+1]$, 
and ${R}_\ell^{p} $ is convex on $p_{\ell}$ for $\ell\in [L]$. Then  Problem \eqref{problem2} is joint convex on $\left(\{\ttw_{\ell}\}_{\ell=2}^{L+1}, \{p_{\ell}\}_{\ell=1}^L\right)$ over the set $\left\{ (\{\ttw_{\ell}\}_{\ell=2}^{L+1}, \{p_{\ell}\}_{\ell=1}^L) :   p_{\ell}  \text{~is equivalent to  Lebesgue measure for~} \ell\in[L]  \right\}.$ 
\end{theorem}
It is worth noting that the regularizes $R^{w}_{\ell}$ ($\ell\in[2:L+1]$) in the discrete formulation are the simple $\ell_{1,r}$ norm regularizers if we write the weights as a matrix.  
This type of regularizers control  the efficacy of the
features  in terms of representation for the underlying learning task; see \cite{fang2019convex}  for more discussion. 

Theorem \ref{theo:con} sheds light on the landscape of the continuous DNN, which  shows the non-existence of bad local minima when all the distributions  are equivalent to  Lebesgue measure.  This condition can be achieved by incorporating proper regularization terms on $\{p_{\ell}\}_{\ell=1}^L$, e.g., $D_{\mathrm{KL}}\left(p_{\ell} \|  p^0_\ell\right) +  D_{\mathrm{KL}}\left(p^0_{\ell} \|  p_\ell\right)$, where $D_{\mathrm{KL}}(\cdot\|\cdot)$ denotes relative entropy and $p^0_{\ell}$ is standard Gaussian distribution.  Therefore, Theorem \ref{theo:con} motivates us to  have a study on  the  dynamics under those entropic regularizers. 

However, the study has the following challenges:  $\mathrm{(i)}$ Our current analysis of the neural feature flow relies on Picard-type iterations, which requires the Lispchitz continuity of the gradients and is not directly applicable when there are such non-trivial regularizers.  $\mathrm{(ii)}$ Our convexity argument  in Theorem \ref{theo:con} is the  usual notion of convexity. They should not be confused with ``displacement convexity'' in the studies of optimal transport (see, e.g., \cite[Chapter 7.3]{santambrogio2015optimal}) and are not sufficient to  guarantee the global convergence.  In Appendix \ref{subsection:non-ri},  we explain the intuition why Gradient Descent can find a  global minimal solution, and a full treatment is left to future studies.
% A rigorous treatment on both  the analysis for the  dynamic equations  and   global convergence   is  left to future studies. 
 
Instead, we  study a relatively simpler case in this paper. 
% Suppose  no regularizer  is imposed, i.e. $R\equiv0$. If $h(\cdot)$ can  achieve  universal approximation, the  global minimum  can be easily  determined.  
We  consider the Res-Net architecture \cite{he2016deep}. Due to the skip connections,  it is possible that  high-level features  highly correlate with low-level ones. We show that under such  architecture the features  change relatively slowly. Then it suffices to prove that   $p_1$ has a full support in any finite time to achieve the global convergence.

\section{Res-Net Formulation and Neural Feature Flow}\label{sec:resnet}
% We  first present the formulation of $L$-layer Res-Nets. 
\subsection{Discrete Res-Net and Scaled Gradient Descent}
For discrete Res-Nets, let $m_\ell$ denote the number of units at layer $\ell$ for $\ell\in [0:L+1]$. 
Suppose each hidden layer has $m$ hidden units that $m_{\ell}=m$ for $\ell\in[L]$. 
Let $m_0=d$ and $m_{L+1}=1$.
For $\ell\in[L+1]$, the output of node $i$ in layer $\ell$ is denoted by $\hbbeta_{\ell,i}\in\RR^N$; the weight that connects the node $i$ at layer $\ell-1$ to node $j$ at layer $\ell$ is denoted by $\hv_{\ell,i,j}\in \RR$.
\begin{enumerate}[(1)]
    \item At the input layer, for $i\in[d]$, let
    \begin{eqnarray}\label{dfor 3}
      \hbbeta_{0,i} :=\left[\bx^1(i), \bx^2(i), \dots, \bx^N(i)  \right]^\top.
    \end{eqnarray}
    \item At the first layer, for $j\in[m]$, let
    \begin{eqnarray}\label{dfor 00}
      \hbbeta_{1, j}  =  \frac{1}{m_{0}}\sum_{i=1}^{m_{0}} \hv_{1, i, j}~ \hbbeta_{0, i}. 
    \end{eqnarray}
    \item We recursively define the upper layers for $\ell \in [2:L]$. 
    Let $ \hbalpha_{ \ell, j}\in\RR^N$ be the residual term at node $j$ at layer $\ell$:
    \begin{equation}\label{dfor 4}
            \hbalpha_{ \ell, j}   =  \frac{1}{m}\sum_{i=1}^{m}  \hv_{\ell, i, j}~ \th_1\left(\hbbeta_{\ell-1, i}\right),   \quad  j \in [m],
    \end{equation}
    where $h_1:\RR\to\RR$ is the activation function.  Furthermore, we consider the following coupling between the residual and the previous feature:
    \begin{equation}
           \hbbeta_{ \ell, j} =   \th_2\left(\hbalpha_{ \ell, j}\right)+  \hbbeta_{ \ell-1, j}, \quad  j \in [m]. \label{dfor 5}
    \end{equation}
    where $h_2:\RR\to \RR$.
    \item At the  output layer, 
     \begin{eqnarray}\label{dfor 6}
    \hbbeta_{L+1, 1} =
 \frac{1}{m}\sum_{i=1}^{m} ~\hv_{L+1, i,1} \th_1\left(\hbbeta_{L, i}\right).
\end{eqnarray}
\end{enumerate}
\iffalse
 We  can  collect all the weights  a single parameter vector: 
 \begin{equation}\label{pad v}
   \hbv := \left\{v_{1,i_1,j}, v_{2, i, j}, \dots,  v_{L, i, j},  v_{L+1,i,1}:  ~i_1\in[d], i\in [m], ~j\in [m]\right\} \in \RR^{m^2(L-1)+(d+1)m}.   
 \end{equation}

Also,  we aggregate  all the residuals into a single vector:
\begin{eqnarray}\label{pad alpha}
  \hbalpha := \left\{ \hbalpha_{\ell,i}:  ~i\in [m],~\ell\in[L] \right\}\in \RR^{NmL},
\end{eqnarray}
and all the features   into a single   vector:
\begin{equation}\label{pad beta}
   \hbbeta := \left\{ \hbbeta_{1,i},\dots, \hbbeta_{L,i}:  ~i\in [m] \right\}\in \RR^{NmL}.  
\end{equation}
\fi 
 We collect  weights, residuals, and features from all layers into  single  vectors $\hbv\in \RR^{D_1}$, $\hbalpha\in \RR^{D_2}$, and $\hbbeta\in \RR^{D_2}$, respectively, where  $D_1:= {m^2(L-1)+(d+1)m}$ and $D_2:=NmL$.
The minimization problem for the Res-Nets is given by
\begin{eqnarray}\label{dis Res-Net}
    \min_{\hbv, \hbalpha, \hbbeta}\hat{\bL}_R (\hbv, \hbalpha, \hbbeta) =  \frac{1}{N}\sum_{n=1}^N \phi\left(   \hbbeta_{L+1,1}(n), ~y^n\right),
     + \hat{R}_R(\hbv, \hbalpha, \hbbeta ), \notag
\end{eqnarray}
where $(\hbv, \hbalpha,\hbbeta)$ satisfies \eqref{dfor 00} -- \eqref{dfor 6}, and $\phi:\RR\times\mathcal{Y}\to \RR$ denotes the loss function and $\hat{R}^R: \RR^{D_1}\times \RR^{D_2}\times \RR^{D_2}\to \RR$ denotes the regularizer.
One noteworthy feature in the architecture  is  \eqref{dfor 5}, where we  introduce a mapping $h_2$ on the residual  $\hbalpha_{\ell,j}$ before fusing it with   $\hbbeta_{\ell-1,j}$.
 We assume that $h_2$ is bounded by a constant $L_1$, and hence $\|\hbbeta_{\ell,j} - \hbbeta_{\ell-1,j}\|_{\infty} \leq L_1$. Therefore,   the high-level features can be regarded as  perturbations of the low-level ones.  Similar ideas have also appeared in \cite{du2018gradient,hardt2016identity}, but are realized in a different way.  For example, in the lazing training regime,  \cite{du2018gradient}  achieved it  by scaling $\hbalpha_{\ell,j}$  with a vanishing $\cO(\frac{1}{\sqrt{m}})$ factor.

%\subsection{Scaled Gradient Descent}

\begin{algorithm}[t]
    \caption{Scaled Gradient Descent for Training a  Res-Net. } 
    \label{algo:RGD}
    \begin{algorithmic}[1]
        \STATE Input the data $\{\bx^i,y^i\}_{i=1}^N$,  step size $\eta$, and  initial weights $\hbv^0$.	
        \FOR{$k =0, 1,\dots, K-1$} 
        \STATE  Perform forward-propagation \eqref{dfor 00} -- \eqref{dfor 6} to compute  $\hbbeta^{k}_{L+1,1}$.
        \STATE  Perform backward-propagation to compute the gradient  $\Grad^k_{\ell,i,j}=\frac{\partial \hat{\bL}_R}{\partial \hv_{\ell,i,j}^k}$.
        \STATE  Perform scaled Gradient Descent:
        \begin{eqnarray}
        \hv_{\ell,i,j}^{k+1} &=&   \hv^{k}_{\ell,i,j} - [\eta m_{\ell-1}m_{\ell}]~ \Grad^k_{\ell,i,j},\quad \ell\in[L+1], ~i\in[m_{\ell-1}], ~j\in[m_{\ell}]. \notag
        \end{eqnarray}
        \ENDFOR
        \STATE Output the weights $\hbv^K$.
    \end{algorithmic}
\end{algorithm}

The  scaled Gradient Descent algorithm without regularization for training a Res-Net is shown in Algorithm \ref{algo:RGD}.
% ,  where we do not consider the regularizer. 
Define intermediate variables in the back-propagation as
\begin{eqnarray}
 \Del^k_{L+1,1} &:=&  N~\frac{\partial \hat{\bL}^k_R}{\partial \hbbeta_{L+1}  } = \left[\phi'_1\left( \hbbeta_{L+1}^k(1), y^1\right),\phi'_1\left( \hbbeta_{L+1}^k(2), y^2\right),\dots, \phi'_1\left( \hbbeta_{L+1}^k(N), y^N\right)\right],\notag\\
   \Del^{\bbeta,k}_{L,i} &:=&   N~\frac{\partial \hat{\bL}^k_R}{\partial \hbbeta_{L,i}  }=\frac{1}{m}\left[\hv_{L+1,i,1}^k~ \Del^k_{L+1,1}\right] \cdot \th_1'\left(\hbbeta^k_{L, i}\right), \quad i\in[m],\notag\\
      \Del^{\balpha,k}_{L,i} &:=&   N~\frac{\partial \hat{\bL}^k_R}{\partial \hbalpha_{L,i}  }=\Del^{\bbeta,k}_{L,i}\cdot \th'_2\left(\hbalpha^k_{L, i}\right), \quad i\in[m],\notag\\
  \Del^{\bbeta,k}_{\ell,i} &:=&   N~\frac{\partial \hat{\bL}^k_R}{\partial \hbbeta_{\ell,i}  }= \frac{1}{m} \left[\sum_{j=1}^{m}\hv_{\ell+1,i,j}^k~ \Del^{\balpha,k}_{\ell+1,j}\right] \cdot \th_1'\left(\hbbeta^k_{\ell, i}\right) + \Del^{\bbeta,k}_{\ell+1,i}, \quad \ell \in[L-1],~i\in[m],\notag\\
    \Del^{\balpha,k}_{\ell,i} &:=&   N~\frac{\partial \hat{\bL}^k_R}{\partial \hbalpha_{\ell,i}  }=\Del^{ \bbeta,k}_{\ell,i}\cdot \th'_2\left(\hbalpha^k_{\ell, i}\right), \quad \ell \in[2: L-1],~i\in[m]\notag.
\end{eqnarray}
Then, we have 
\begin{eqnarray}
  \Grad^{k}_{L+1,i,1} &=& \frac{1}{Nm}   \left[\Del_{L+1}^{k}\right]^\top \th_1\left(\hbbeta_{\ell, i}^k\right), \quad  i \in [m],\notag\\
  \Grad^k_{\ell+1,i,j} &=& \frac{1}{Nm}   \left[\Del_{\ell+1,j}^{\balpha,k}\right]^\top \th_1\left(\hbbeta_{\ell, i}^k\right), \quad \ell \in[ L-1], ~i,j\in [m],\notag\\
 \Grad^k_{1,i,j} &=& \frac{1}{Nd}   \left[\Del_{1,j}^{\bbeta,k}\right]^\top \hbbeta_{0, i}^k, \quad  ~i \in [d], ~ j\in [m]. \notag
\end{eqnarray}

\subsection{Continuous Res-Net Formulation}\label{sec:general resnet}
\iffalse
We then propose a general formulation  for the continuous Res-Net. In the next section, the reader can find the convenience  of   it as  the description of  the  initial state of neural feature flow. Given a continuous Res-Net under the basic formulation, it is not hard to    transform it  to the general formulation. 
The main differentiation  is that
the general formulation $\mathrm(i)$   indexes (represents)  the hidden nodes in layer $\ell\in[2:L]$ by  the function values of residuals, i.e., $\balpha_{\ell}$, instead of $\bbeta_{\ell}$, to avoid inverse operation on $h_2$ and   $\mathrm(ii)$  novelly  introduce  a joint distributions over $(\bv_1, \balpha_1, \balpha_2, \dots,\balpha_L  )$ to characterize the overall state of the skip connections in the continuous Res-Net.  We first present the formulation and later provides more  explanations. 
\fi 

In the continuous Res-Net, we index the hidden nodes in layer $\ell\in[2:L]$ by  the function values of residuals $\balpha_{\ell}$.   To deal with Res-Nets, our main technique here is to characterize the overall state of the continuous Res-Nets by the joint distribution $p$ over ``skip-connected paths'' $\bTheta=(\bv_1, \balpha_1, \balpha_2, \dots,\balpha_L)\in \RR^D$ for $D= d+(N-1)L$.
For $\bTheta = (\bv_1, \balpha_2, \dots, \balpha_L)\in \RR^{D}$, one can intuitively regard $\bTheta$   as    an input-output path $\bv_1 \rightarrow \balpha_2  \rightarrow \dots  \rightarrow \balpha_L$.  
Then $p(\bTheta)$ can be interpreted as the density of such skip-connected paths in the  continuous Res-Nets. Thus the  joint distribution $p$ can be regarded as a description of  the overall topological structure about the skip connections.  
We represent the features $\bbeta_\ell$ in the hidden layer $\ell\in[2:L]$ as functions of $\bTheta$ that we introduce next:

\begin{enumerate}[(1)]
    \item At the input layer, let $\bX = \left[\bx^1, \bx^2, \dots, \bx^N\right]^\top\in\RR^{N\times d}$.
    \item At the first layer, let the features be
    $$\bbeta_{1}\left(\bTheta\right) =  \frac{1}{d} \left(\bX \bv_1\right).$$
    \item At layer $\ell \in [2:L]$, let $v_{\ell}: \supp(p)\times \supp(p) \to \RR$ denote the weights on the connections from layer $\ell-1$ to $\ell$. 
    For any given skip-connected path $\bTheta$, let $\balpha_\ell$ be its $\ell$-th element.
    We have
    % \red{For   $\bTheta=(\bv_1, \balpha_2, \dots, \balpha_L)\in \supp(p)$, we have}
    \begin{equation}
        \begin{split}\label{forres}
           &\balpha_{\ell} =  \int  v_{\ell}\left(\bTheta, \bbTheta\right) \th_1\left(\bbeta_{\ell-1}(\bbTheta)\right)  d p\left(\bbTheta\right),\\
    &\bbeta_{\ell}\left(\bTheta\right) =  \th_2\left(\balpha_{\ell}\right)+ \bbeta_{\ell-1}\left(\bTheta\right) .     
        \end{split}
    \end{equation}
    \item At the output layer, let $v_{L+1}: \supp(p)\to \RR$ be the weights in the layer $L+1$, and we have
    \begin{eqnarray}\label{hbbeta}
      \bbeta_{L+1}  =   \int  v_{L+1}\left(\bTheta\right)  \th_1 \left(\bbeta_{L}\left(\bTheta\right)\right) d p\left(\bTheta\right). \notag
    \end{eqnarray}
\end{enumerate}

The overall learning problem for the continuous Res-Nets is formulated as
\begin{align}\label{problem resnet}
  \minimize_{\left\{v_{\ell}\right\}_{\ell=2}^{L+1}, ~p }  
% \min \mathcal{L}(p,v_2,\dots,v_{L+1})=
   \quad&  \frac{1}{N}\sum_{n=1}^N \phi\left( \bbeta_{L+1}(n), y^n \right) +  R_R\left( \left\{v_{\ell}\right\}_{\ell=2}^{L+1}, p\right)\\
%   + R^R\left(\left\{v_{\ell}\right\}_{\ell=2}^{L+1},  ~p\right) \\
\text{s.t.}  \quad  \bbeta_{\ell}\left(\bTheta\right) &= \frac{1}{d}\bX \bv_1+  \sum_{i=2}^{\ell}   \th_2\left(\balpha_{i}\right), ~~ \bTheta=(\bv_1, \balpha_2, \dots, \balpha_L)\in \supp(p),~\ell\in[L],\notag\\
~~~\balpha_{\ell} =& \int  v_{\ell}\left(\bTheta, \bbTheta\right) \th_1\left(\bbeta_{\ell-1}(\bbTheta)\right)  d p\left(\bbTheta\right),  ~~ \bTheta=(\bv_1, \balpha_2, \dots, \balpha_L)\in \supp(p),~\ell\in[2:L],\notag\\
\bbeta_{L+1}  &=   \int  v_{L+1}\left(\bTheta\right)  \th_1 \left(\bbeta_{L}\left(\bTheta\right)\right) d p\left(\bTheta\right).\notag 
\end{align}

\begin{algorithm}[tb]
	\caption{Example $1$ for Initializing  a  Discrete  Res-Net. }
	\label{algo:init:sim}
		\begin{algorithmic}[1]
		\STATE Input the data $\left\{\hbbeta_{0,i}\right\}_{i=1}^d$,   variance $\sigma_1>0$, and a constant $C_5$.
	 \STATE  Independently  draw $\hat{v}_{1,i,j}\sim p_0= \mathcal{N}\left(0, d\sigma^2_1\right)$ for $i\in[d]$ and $j\in[m]$. 
		\STATE For $\ell\in[2:L]$, $i\in[m]$, and $j\in[m]$, set $\hv_{\ell,i,j} =0$.
		\STATE For $i\in[m]$, set $\hv_{L+1,i} =C_5$.
		\STATE  Perform forward-propagation \eqref{dfor 00} -- \eqref{dfor 6} to compute    to compute $\hbalpha$ and $\hbbeta$.
\STATE Output the  discrete Res-Net  $(\hbv, \hbalpha, \hbbeta)$.
	\end{algorithmic}
\end{algorithm}

\begin{algorithm}[tb]
\caption{Example $2$ for  Initializing  a  Discrete  Res-Net.} 
\label{algo:init:resnet}
\begin{algorithmic}[1]
    \STATE Input the data $\left\{\hbbeta_{0,i}\right\}_{i=1}^d$,   variance $\sigma_1>0$, and a constant $C_5$.
    \STATE  Independently  draw $\hat{v}_{1,i,j}\sim p_0= \mathcal{N}\left(0, d\sigma^2_1\right)$ for $i\in[d]$ and $j\in[m]$. 
    \STATE  Set $\hbbeta_{1,j} = \frac{1}{d}\sum_{i=1}^{d} \hat{v}_{1, i, j}~ \hbbeta_{0, i}$ where $j\in[m]$. {\hfill $\diamond$ Standard Initialization for  layer $1$ }
    \FOR{ $\ell =2,\dots, L$ }
        \STATE  Independently draw $\tilde{v}_{\ell,i,j}\sim \mathcal{N}\left(0, m\sigma^2_1\right)$ for $i,j\in[m]$.
        \STATE   Set $\hbalpha_{\ell,j} = \frac{1}{m}\sum_{i=1}^m\tilde{v}_{\ell,i,j}~\th_1(\hbbeta_{\ell-1,i})$ where $j\in[m]$. 
        \STATE  Set $\hbbeta_{\ell,j}=\hbbeta_{\ell-1,j}+ \th_2\left(\hbalpha_{\ell,j}\right)$ for $j\in[m]$.
        {\hfill $\diamond$ Standard Initialization for  layer $\ell$ }
    \ENDFOR
    \STATE Set $\hv_{L+1,i,1} = C_5$ where $i\in[m]$. {\hfill $\diamond$  Simply initialize $\left\{\hv_{L+1,i,1}\right\}_{i=1}^m$ by a constant}
    \FOR{$\ell =2,\dots, L$}  
        \FOR{$j =1,\dots, m$}  
           \STATE Solve convex optimization problem: {\hfill $\diamond$ Perform $\ell_2$-regression to reduce redundancy} \\
           \begin{equation}
           \min_{\left\{\hv_{\ell,i,j}\right\}_{i=1}^{m}}~ \frac{1}{m}\sum_{i=1}^{m} \left(\hv_{\ell,i,j} \right)^2, ~~~~ \text{s.t.}~~  \hbalpha_{\ell,j} = \frac{1}{m}\sum_{i=1}^m\hv_{\ell,i,j}~\th_1(\hbbeta_{\ell-1,i}). \notag
           \end{equation}
        
        \ENDFOR
    \ENDFOR
    \STATE Output the discrete Res-Net parameters $(\hbv, \hbalpha, \hbbeta)$.
\end{algorithmic}
\end{algorithm}

To have a better understanding of $p$, let us consider  two concrete examples. 
\begin{itemize}
    \item  Algorithm \ref{algo:init:sim} simply sets the weights in layer $\ell\in[2:L]$ as $0$. The continuous limit is
    $$ p\left(\bv_1, \balpha_1, \dots, \balpha_L\right) = p_1^{\bv}(\bv_1)\times \prod_{\ell=2}^L\delta\left(\balpha_{\ell} = \mathbf{0}^{N}   \right),$$
    where $p_1^{\bv} = \mathcal{N}\left(0, d\sigma^2_1\mathbf{I}^d\right)$.
    \item  Algorithm \ref{algo:init:resnet} generates a Res-Net by a standard initialization strategy with  an additional  $\ell_2$-regression procedure to reduce the redundancy of the weights. In its continuous limit, we have the following properties for the distributions of features and residuals:
    \begin{enumerate}[(1)]
        \item  At the first layer, $\bbeta_1\sim p_1^{\bbeta}= \mathcal{N}\left( \mathbf{0}^{N}, \sigma_1^2 \bK_0\right)$, where $\bK_0 :=\frac{1}{d}\bX \bX^\top $.  
        \item At the layer $\ell\in[L-1]$, let $\bK_{\ell}^{\bbeta} := \int \th_1\left(\bbeta_{\ell}\right) \th_1\left(\bbeta_{\ell}\right)^\top d p_{\ell}^{\bbeta}\left(\bbeta_{\ell}\right)$. Then the residuals at layer $\ell+1$ follows the distribution 
        \begin{equation}\label{palpha}
            p_{\ell+1}^{\balpha} =    \mathcal{N}\left( \mathbf{0}^{N}, \sigma_1^2 \bK_{\ell}^{\bbeta}\right). 
        \end{equation}
        Similar to Subsection~\ref{subsection:app}, $\balpha_{\ell+1}$  is independent of $\bbeta_{\ell}$ in the continuous limit. 
        Defining the mapping   $\tilde{f}_{\ell+1}\left( \bbeta_{\ell}, \balpha_{\ell+1}\right) := \bbeta_{\ell}+\th_2(\balpha_{\ell+1})$, the  features  at layer $\ell+1$ follows  the pushforward measure by $\tilde{f}_{\ell+1}$: 
        $$  p_{\ell+1}^{\bbeta} = \tilde{f}_{\ell+1}\#\left( p_{\ell}^{\bbeta}\times p_{\ell+1}^{\balpha} \right). $$
    \end{enumerate}
    Therefore, $p$ is a multivariate Gaussian distribution of the form 
    \begin{eqnarray}\label{ppp}
   p\left( \bv_1, \balpha_1, \balpha_2,\dots,\balpha_L \right) :=  p_1^{\bv}(\bv_1) \times p_2^{\balpha}(\balpha_2)\times p_3^{\balpha}(\balpha_3)\times\dots \times p_L^{\balpha}(\balpha_L).  
    \end{eqnarray}
\end{itemize}

\subsection{Neural Feature Flow for Res-Net}
We introduce the  evolution of a continuous Res-Net trained by the scaled Gradient Descent Algorithm.   In contrast with DNNs,  
the situation for Res-Nets is  more complex. For Res-Nets, the weights may receive a different gradients even they are on the connection  of the hidden units with the same output. 
It means that the states of $\bv_1$, $\{v\}_{\ell=2}^L$ $\{\bbeta\}_{\ell=1}^{L}$, and $\{\balpha\}_{\ell=2}^{L}$ will spit during training.  However,  one important observation is that the splitting occurs  only when the weights  are on different skip-connected paths.  Therefore, following our continuous formulation,   we  represent all the trajectories  as functions of the   skip-connected  paths.  Especially,  we  introduce the notations for the trajectories of $\bv_1$, $\{v_{\ell}\}_{\ell=2}^L$, $\{\bbeta_{\ell}\}_{\ell=1}^{L}$, and $\{\balpha_{\ell}\}_{\ell=2}^{L}$:
\begin{itemize}
    \item $\Phi_{\ell}^{\bbeta}:\mathrm{supp}(p)\to C([0,T], \RR^N)$ is the trajectory of $\bbeta_{\ell}$ for $\ell\in[L]$;
    \item $\Phi_{\ell}^{\balpha}:\mathrm{supp}(p)\to C([0,T], \RR^N)$ is the trajectory of $\balpha_{\ell}$ for $\ell\in[2:L]$;
    \item $\Phi_{1}^{\bv}:\mathrm{supp}(p)\to C([0,T], \RR^d)$  and $ \Phi_{L+1}^{\bv}:\mathrm{supp}(p) \to C([0,T], \RR) $ are the trajectories of $\bv_1$ and $v_{L+1}$, respectively;
    \item $ \Phi_{\ell}^{\bv} :\mathrm{supp}(p) \times \mathrm{supp}(p) \to C([0,T], \RR)$ is the trajectory of $v_{\ell}$ for $\ell\in[2:L]$. 
\end{itemize}

Then  the continuous gradient for the weight can be obtained from the backward-propagation algorithm. Specifically, for all  $\bTheta=(\bv_1,\balpha_2,\dots, \balpha_L ) \in\supp(p)$, $t\in[0,T]$, and $\ell\in[2:L]$, let
\begin{align}
  \bbeta_{L+1}\left(\Phi,t\right) &:=\int \Phi^{\bv}_{L+1}\left(\bTheta\right)(t)~ \th_1\left(\Phi^{\bbeta}_{L}\left(\bTheta\right)(t)\right)   d p\left(\bTheta\right),\label{betal+1}   \\
  \ulineDel_{L+1}(\Phi,t) &:=  \left\{\phi'_1\left(\bbeta_{L+1}\left(\Phi,t\right) (n), y^n\right): n\in[N]\right\},\notag\\
 \ulineDel^{\bbeta}_{L}(\bTheta; \Phi,t) &:= \left[\Phi^{\bv}_{L+1}\left(\bTheta\right)(t)  ~ \ulineDel_{L+1}(\Phi,t)\right] \cdot  \th'_1\left(\Phi^{\bbeta}_{L}(\bTheta)(t)\right),\label{def linedel res}\\
      \ulineDel^{\balpha}_{\ell}(\bTheta; \Phi,t) &:=  \ulineDel_{\ell}^{\bbeta}(\bTheta; \Phi,t) \cdot \th'_2\left(\Phi^{\balpha}_{\ell}(\bTheta)(t)\right),\notag\\
  \ulineDel_{\ell-1}^{\bbeta}(\bTheta; \Phi,t) &:=  \ulineDel_{\ell}^{\bbeta}(\bTheta; \Phi,t) \!+\! \left[\int \! \Phi^{\bv}_{\ell}\left(\bTheta, \bbTheta\right)(t) ~ \ulineDel_{\ell}^{\balpha}(\bbTheta; \Phi,t)d p\left(\bbTheta\right)\right]\! \cdot  \th'_1\left(\Phi^{\bbeta}_{\ell-1}\left(\bTheta\right)(t)\right).\notag
\end{align}
For all $\bTheta,\bbTheta\in\supp(p)$, the drift term for the weights is given by
\begin{subequations}
\begin{align}
 \ulineGrad^{\bv}_{L+1}\left(\bTheta;\Phi,t\right) :=&  \frac{1}{N} \left[\ulineDel_{L+1}(\Phi,t)\right]^\top \th_1\left(\Phi^{\bbeta}_{L}(\bTheta)(t)\right),\label{linegrad wL+1 res}\\
\ulineGrad^{\bv}_{\ell}\left(\bTheta, \bbTheta ;\Phi,t\right) :=&   \frac{1}{N}\left[\ulineDel_{\ell}^{\balpha}(\bbTheta;\Phi,t)\right]^\top \th_1\left(\Phi^{\bbeta}_{\ell-1}\left(\bTheta\right)(t)\right), \quad \ell\in[2:L],\notag\\
  \ulineGrad^{\bv}_{1}\left(\bTheta;\Phi,t\right) :=&   \frac{1}{N} \bX~ \ulineDel_{1}^{\bbeta}\left(\bTheta;\Phi,t\right).\notag
\end{align}
\end{subequations}
Moreover, the drift term for the residuals and features can be obtained by the chain rule: for $\ell\in[L-1]$ and  $\bTheta\in\supp(p)$, 
\begin{subequations}
\begin{align}
      \ulineGrad^{\bbeta}_{1}\left(\bTheta;\Phi,t\right) &:=  \frac{1}{d} \left[ X \ulineGrad^{\bv}_{1}(\bTheta;\Phi, t) \right], \notag\\
      \ulineGrad^{\balpha}_{\ell+1}\left( \bTheta ;\Phi, t\right) &:= \int 
      \Phi^{\bv}_{\ell+1}\left(\bbTheta, \bTheta\right)(t)~
    %   v_{\ell+1}^t\left(\bbTheta, \bTheta\right)~ 
      \left[\th_1'\left(\Phi^{\bbeta}_{\ell}(\bbTheta)(t)\right)\cdot \ulineGrad^{\bbeta}_{\ell}\left(\bbTheta; \Phi,t\right)\right] d p\left(\bbTheta\right)+ \notag\\
  &\quad~~ +\int \th_1\left(\Phi^{\bbeta}_{\ell}\left(\bbTheta\right)(t)\right)\cdot \ulineGrad_{\ell+1}^{\bv}\left(\bbTheta, \bTheta;\Phi,t\right) d p\left(\bbTheta\right),\notag\\
        \ulineGrad^{\bbeta}_{\ell+1}\left( \bTheta ;\Phi, t\right) &:=  \ulineGrad^{\bbeta}_{\ell}\left(\bTheta; \Phi,t\right) +  \ulineGrad^{\balpha}_{\ell+1}\left( \bTheta ;\Phi, t\right)~ \th'_2\left(\Phi^{\balpha}_{\ell+1}\left(\bTheta\right)(t)\right).  \notag
\end{align}
\end{subequations}

The process of  a continuous Res-Net trained by Gradient Descent can be defined below.
\begin{mdframed}[style=exampledefault]
\begin{definition}[Neural Feature Flow for Res-Net]\label{resNFLp0}
Given an initial continuous Res-Net represented by  $(\left\{v_{\ell}\right\}_{\ell=2}^{L+1}, p)$ and $T<\infty$, we say a trajectory $\Phi_*$ is
a neural feature flow 
% a solution of the continuous Res-Net trained by Gradient Descent  
if   for all $\bTheta=(\bv_1,\balpha_2,\dots, \balpha_L )\in \supp(p)$,  $\bbTheta\in\supp(p)$, and $t\in[0,T]$,
\begin{align}
  \Phi_{*,\ell}^{\bbeta}\left(\bTheta\right)(t) &= \left[\frac{1}{d}\bX \bv_1+  \sum_{i=2}^{\ell}   \th_2\left(\balpha_{i}\right)\right] -\int_{0}^t\ulineGrad^{\bbeta}_{\ell}\left( \bTheta;\Phi_*,s\right),\quad\ell\in[L],  \notag\\
  \Phi_{*,\ell}^{\balpha}\left(\bTheta\right)(t) &=   \balpha_{\ell} -\int_{0}^t\ulineGrad^{\balpha}_{\ell}\left( \bTheta;\Phi_*,s\right)ds,\quad \ell\in[2:L], \notag\\
  \Phi_{*,1}^{\bv}\left(\bTheta\right)(t) &= \bv_1 -\int_{0}^t\ulineGrad^{\bv}_{1}\left( \bTheta;\Phi_*,s\right)ds,\notag\\
  \Phi_{*,\ell}^{\bv}\left(\bTheta, \bbTheta\right)(t) &= v_{\ell}(\bTheta,\bbTheta) -\int_{0}^t\ulineGrad^{\bv}_{\ell}\left( \bTheta, \bbTheta;\Phi_*,s\right)ds,\quad \ell\in[2:L],\notag\\
  \Phi_{*,L+1}^{\bv}\left(\bTheta\right)(t) &= v_{L+1}(\bTheta) -\int_{0}^t\ulineGrad^{\bv}_{L+1}\left( \bTheta;\Phi_*,s\right)ds. \notag  
\end{align}
\end{definition}
\end{mdframed}
\section{Analysis of Continuous Res-Net}\label{sec:theo resnet}
\subsection{Assumptions for Res-Net}
We  make the following assumptions that are needed in our analysis. Firstly, the assumptions for the loss and  activation functions in analyzing the DNNs still hold. Specially, we assume that 
\begin{assumption}[Activation Functions and  Loss Function]\label{ass:5}
For the activation functions, we assume  that  there exist  constants $L_1, L_2, L_3>0$ such that, for all $x\in\RR$,
\begin{eqnarray}
       \left|h_1(x)\right| \leq  L_1,\quad \left|h_2(x)\right| \leq  L_1,\quad    \left|   h_1' (x) \right| \leq L_2,\quad\left|   h_2' (x) \right| \leq L_2.\notag
\end{eqnarray}
Moreover,  for all $x,y\in\RR$, 
$$  \left|   h_1' (x) -  h_1'(y) \right|\leq L_3 |x-y |,\quad  \left|   h_2' (x) -  h_2'(y) \right|\leq L_3 |x-y |. $$
For the loss function, we assume that  there exist  constants $L_4,L_5 > 0$ such that, for all $y\in\mathcal{Y}$,  $x_1 \in\RR$, and $x_2\in \RR$,
\begin{eqnarray}
        \left|   \phi'_1 (x_1,y) \right| \leq L_4,\quad\quad  \left|   \phi'_1 (x_1,y)  - \phi'_1 (x_2,y) \right|\leq L_5 |x_1-x_2 |.\notag
\end{eqnarray}
\end{assumption}

We also assume that $p$ is a sub-gaussian distribution  and the weights in the Res-Net are initialized with proper  boundedness and continuity property.
\begin{assumption}[Initialization for Res-Net]\label{ass:6}
We assume  that $p$ is $\sigma$-sub-gaussian distribution. %There exist a constant $\sigma>0$,  for all $\bc_3\in  \mathbb{S}^{D-1}$, we have
%$$ \EE_{\bTheta\sim p}\left[ \exp\left( \left(\bc_3^\top \bTheta\right)^2/\sigma^2\right)\right]\leq e.  $$
We assume that, for all $\ell\in[2:L]$, $v_{\ell}(\cdot, \cdot)$ has sublinear growth on the second argument, that is, there is  a constant $C_5$ such that 
\begin{eqnarray}
        \left|v_{\ell}\left(\bTheta,\bbTheta\right)\right| &\leq& C_5 \left(1+ \left\|\bbTheta\right\|_{\infty}\right), \quad \text{for all}~~ \bTheta, \bbTheta\in \supp(p)  , ~ \ell \in [2: L]\notag.
\end{eqnarray}
Moreover, we assume  that  $v_{\ell}(\cdot, \cdot)$   are locally Lipschitz continuous where the Lipschitz constant has sub-linear growth on  the second argument. In detail, there is a  constant $C_6$, such that  for  $\bTheta_1 \in \supp(p)$,  $\tbTheta_1\in \supp(p)\cap \mathcal{B}_{\infty}\left(\bTheta_1,1\right)$, $\bTheta_2 \in \supp(p)$,  and $\tbTheta_2 \in \supp(p)\cap \mathcal{B}_{\infty}\left(\bTheta_2,1\right)$, we have
\begin{eqnarray}
\left|v_{\ell}\big(\bTheta_1,\bTheta_2\big) -v_{\ell}\big(\tbTheta_1,\tbTheta_2\big) \right| &\leq&C_6\big(  1+ \left\|\bTheta_2\right\|_{\infty} \big)\left(\big\|\bTheta_1 -\tbTheta_1\big\|_{\infty}+\big\|\bTheta_2 -\tbTheta_2\big\|_{\infty}\right).\notag
\end{eqnarray}
For the last layer,  there exist constants  $C_7$ and $C_8$, such that for all $\bTheta, \bbTheta\in\supp(p)$, we have
\begin{eqnarray}
   \left|v_{L+1}\big(\bTheta\big)\right|\leq C_7 \quad\text{and}\quad  \left|v_{L+1}\big(\bTheta\big) -v_{L+1}(\bbTheta) \right| &\leq& C_8\left\|\bTheta -\bbTheta\right\|_{\infty}.      \notag
\end{eqnarray}
\end{assumption}

\iffalse
\begin{assumption}[Topology of Res-Net]\label{ass:6}
We assume that $M_{\ell}$ is Lipschitz continuous.  Specially, for all $\ell\in[L-1]$,  there exists a  constant  $L_6\geq0$, such that for all $\bbeta_{\ell}\in\supp(p_{\ell})$ and $\bbbeta_{\ell}\in\supp(p_{\ell})$,
\begin{eqnarray}
 \left\| M_{\ell}(\bbeta) - M_{\ell}(\bbbeta)  \right\|_{\infty} \leq L_6 \left\|\bbeta - \bbbeta  \right\|_{\infty}.\notag
\end{eqnarray}
\end{assumption}
The simplest way to achieve this assumption is to set $M_{\ell}(\bbeta)=\mathbf{0}^{N}$ for all $\bbeta_{\ell}\in\supp(p_{\ell})$.
\fi

We then propose the assumptions for the global convergence guarantee.

\begin{assumption}[Initial Topological Structure of Res-Net]\label{ass:7}
 We assume that there exists a continuous function $f_1:\RR^{d}\to \RR^{D-d}$ such that $\supp(p) \supseteq \left\{ \left(\bv_1,  f_1(\bv_1)\right): \bv_1\in\RR^d\right\}$.
\end{assumption}
 Assumption \ref{ass:7} implies that marginal distribution of $p$ on $\bv_1$ has a full support. 
Note that Assumption \ref{ass:7} can be realized by both Algorithms \ref{algo:init:sim} and \ref{algo:init:resnet}.

\begin{assumption}[Strong Universal Approximation Property]\label{ass:8}
% Assume that for any bounded $\epsilon_i:\RR^d\mapsto \RR$ that $\|\epsilon_i\|_\infty\le C_B$ for $i\in[N]$, by letting $g_i(x)=h_1(v_1^\top \bx^i+ \epsilon_i(v_1))$, 
% the functions $(g_1,\dots,g_N)$ are linearly independent and the smallest singular value is lower bound by $\delta>0$:
% \[
% \min_{\ba\ne \mathbf{0}}\frac{\EE\left|\sum_{i=1}^N a_i g_i(\bZ)\right|}{\|\ba\|_2}>\delta,
% \]
% where $\bZ\sim \mathcal{N}\left(0, \mathbf{I}^d\right)$ and $\delta$ only depends on $C_B$.
 Assume that  for  any function $f_2:\RR^{d}\to \RR^{N}$ that is bounded by $C_{B}$, i.e.,  for all $\bv_1\in \RR^{d}$, $\left\|f_2(\bv_1) \right\|_{\infty}\leq C_{B}$, we have
\begin{eqnarray}\label{uni}
 \lambda_{\min} \left[\int     \left[\th_1\left(\frac{1}{d} \bX \bv_1+ f_2\left(\bv_1\right)  \right)\right]\left[\th_1\left( \frac{1}{d} \bX\bv_1 + f_2\left(\bv_1\right)  \right)\right]^\top  d\tp_1\left(\bv_1\right)\right]  \geq \blambda >0.
\end{eqnarray}
where   $\blambda$ only depends on $\bX$,  $C_B$, and $h_1$, and  $\tp_1 = \mathcal{N}\left(\mathbf{0}^d, \mathbf{I}^d\right)$.   
\end{assumption}

Assumption \ref{ass:8} is a technical assumption that we conjecture to hold under fairly general conditions. 
Notably when $C_B=0$, it is shown in \cite[Lemma F.1]{du2018gradient} that the assumption holds for all analytic non-polynomial $h_1$.
% Intuitively, it says 
Lemma \ref{lemma:uni} affords many examples that satisfy the assumption for constant~$C_B$.
\begin{lemma}\label{lemma:uni}
Suppose that the data is  non-parallel, i.e., $\bx_i \notin \mathrm{Span}(\bx_j)$ for all $i\ne j$. 
\begin{enumerate}[(i)]
    \item\label{lemmai} If $g:\RR\to \RR$ is a non-polynomial function that is bounded and has Lipschitz continuous gradient, then  $h_1(x):=g(c x)$ satisfies Assumption \ref{ass:8} when $c>0$ is sufficiently small.
    \item\label{lemmaii}  The Relu-type function $h_1(x) = (x)_+^{\alpha}$ for $\alpha>0$ satisfies Assumption \ref{ass:8}.
    \item\label{lemmaiii}  
    If $h_1(x)=c |x|^{-\alpha}$ or $h_1(x)=c (x)_+^{-\alpha}$ for $|x|>c'$, where $c,c',\alpha>0$, then $h_1$ satisfies Assumption \ref{ass:8}.
\end{enumerate}
\end{lemma}

\iffalse
\begin{assumption}[Input Data]\label{ass:9}
We assume that for all $i\in[N]$ $\|\bx^i\|\leq C_x$, moreover, the data is  non-parallel, i.e. for all $i\in[N]$, and $j\in[N]$,  if $i\neq j$, then  $\bx_i \notin \mathrm{Span}(\bx_j)$.
\end{assumption}

\begin{assumption}[Universal Approximation for $h_1$]\label{ass:10}
 For all bounded function $f:\RR^d\to\RR$, i.e. for any $\bv\in\RR^d$, $|f(\bv)|\leq C_{10}$,  $h_1$
 
 we assume that $H$ has full rank.

\end{assumption}

\fi

\subsection{Properties of Neural Feature Flow for Res-Net}
Simialar to fully-connected DNNs, we show the existence and uniqueness of  neural feature flow, and the solution $\Phi_*$ is a continuous mapping on $\bTheta$ given a time $t$.
\begin{theorem}[Existence and Uniqueness of Neural Feature Flow on Res-Net]\label{theorm:resflow p0}
Under Assumptions \ref{ass:5} and \ref{ass:6}, for any $T<\infty$,   there exists an 
 unique neural feature flow $\Phi_*$. 
\end{theorem}
% We also have that $\Phi_*$ is a continuous mapping on $\bTheta$ given time $t$, which is stated  as follows.
\begin{theorem}[Property of $\Phi_*$]\label{theorm:phi}
Under Assumptions \ref{ass:5} and \ref{ass:6}, let $\Phi_*$ be the neural feature flow, there  are constants    $R\geq0$ and $R'\geq0$ such that  for all $t\in[0,T]$,  $\bTheta_1 \in \supp(p)$ and  $\tbTheta_1 \in \supp(p)\cap \mathcal{B}_{\infty}\left(\bTheta,1\right)$, $\bTheta_2 \in \supp(p)$, and $\tbTheta_2\in \supp(p)\cap \mathcal{B}_{\infty}\left(\bTheta_2,1\right)$,  we have
\begin{equation}\label{tho:22res}
    \begin{split}
 \left\|\Phi^{\bbeta}_{*,\ell}\big(\bTheta_1\big)(t)-\Phi^{\bbeta}_{*,\ell}\big(\tbTheta_1\big)(t) \right\|_{\infty}&\leq  Re^{R' t}\left( \|\bTheta_1\|_{\infty}+1 \right)\left\|\bTheta_1 -\tbTheta_1 \right\|_{\infty},\quad \ell\in[L],\\
  \left\|\Phi^{\balpha}_{*,\ell}\big(\bTheta_1\big)(t)-\Phi^{\balpha}_{*,\ell}\big(\tbTheta_1\big)(t) \right\|_{\infty}&\leq  Re^{R' t}\left( \|\bTheta_1\|_{\infty}+1 \right)\left\|\bTheta_1 -\tbTheta_1 \right\|_{\infty},\quad \ell\in[2:L],\\    
         \left\|\Phi^{\bv}_{*,1}\big(\bTheta_1\big)(t)-\Phi^\bv_{*,1}\big(\tbTheta_1\big)(t) \right\|_{\infty}&\leq  Re^{R' t}\left( \|\bTheta_1\|_{\infty}+1 \right)\left\|\bTheta_1 - \tbTheta_1 \right\|_{\infty},\\
  \left|\Phi^{\bv}_{*,\ell}\big(\bTheta_1, \bTheta_2\big)(t) - \Phi^{\bv}_{*,\ell}\big(\tbTheta_1, \bTheta_2\big)(t)  \right|&\leq R e^{R't}\left( \|\bTheta_1\|_{\infty}+ \|\bTheta_2\|_{\infty}+1\right)  \left\|\bTheta_1 -\tbTheta_1 \right\|_{\infty},\quad \ell\in[2:L],\\
   \left|\Phi^{\bv}_{*,\ell}\big(\bTheta_1, \bTheta_2\big)(t) - \Phi^{\bv}_{*,\ell}\big(\bTheta_1, \tbTheta_2\big)(t)  \right|&\leq R e^{R't}\left( \|\bTheta_1\|_{\infty}+ \|\bTheta_2\|_{\infty}+1\right) \left\|\bTheta_2 -\tbTheta_2 \right\|_{\infty},\quad \ell\in[2:L],\\
    \left|\Phi^{\bv}_{*,L+1}\big( \bTheta_1\big)(t) - \Phi^{\bv}_{*,L+1}\big( \tbTheta_1\big)(t)  \right|&\leq R e^{R't} \left(  \|\bTheta_1\|_{\infty}+1 \right)  \left\|\bTheta_1-\tbTheta_1 \right\|_{\infty}. \notag
    \end{split}
\end{equation}
\end{theorem}

\subsection{Approximation Using Finite Neurons for Res-Net}
We consider the approximation between   a discrete DNN  trained by scaled Gradient Descent and a continuous one evolving as neural feature flow. Following the  procedure of Subsection \ref{subsection:app}, we first propose the general initial condition for the  discrete Res-Net and  define the actual and idea processes, respectively.

% \begin{mdframed}[style=exampledefault]
\begin{definition}[$\ep_1$-independent Initial Res-Net]\label{Res-Net condition}
We say an initial discrete Res-Net $(\hbv, \hbalpha,\hbbeta)$ is $\ep_1$-independent if there exist  a continuous initial Res-Net $(\left\{v_\ell\right\}^{L+1}_{\ell=2},p)$ satisfying Assumption~\ref{ass:6} and $(\bbv, \bbalpha,\bbbeta)$ such that
\begin{enumerate}[(1)]
\item  $\bbTheta_i=\left(\bbv_{1,i}, \bbalpha_{2,i}, \dots, \bbalpha_{L,i} \right)\iiddistr p$;
\item For $\bbbeta$ and $\linebv$,
\begin{itemize}
    \item $\bbbeta_{\ell,i} =  \frac{1}{d}\left(\bX \bbv_{1,i}\right)  +\sum_{\ell_1=2}^{\ell}\th_2(\balpha_{\ell_1,i}) $ for $\ell\in[L]$ and $i\in[m]$;
    \item $\linev_{\ell,i,j} = v_{\ell}\left(\bbTheta_i, \bbTheta_j\right)$ for  $\ell\in[2:L]$, $i,j\in[m]$;
    \item $\linev_{L+1,i,1} = v_{L+1}\left(\bbTheta_i\right)$ for $i\in[m]$;
\end{itemize}
\item $\ep_1$-closeness:
\begin{itemize}
    \item  $\| \bbv_{1,i} -  \hbv_{1,i}  \|_{\infty} \leq  \left( 1+  \left\| \bbTheta_{i}\right\|_{\infty}\right)\ep_1$ for  $i\in[m]$;
    \item    $\left|\bar{v}_{\ell+1, i,j}- \hv_{\ell+1,i, j}\right| \leq \left(1+ \left\|\bbTheta_{i}\right\|_{\infty}+ \left\|\bbTheta_{j}\right\|_{\infty} \right)\ep_1$ for  $\ell\in[L-1]$, $i,j\in[m]$;
    \item $\left|\bar{v}_{L+1, i,1}- \hv_{L+1,i, 1}\right|\leq  \left(1+ \left\|\bbTheta_{i}\right\|_{\infty} \right)\ep_1$ for  $i\in [m]$.
\end{itemize}
\end{enumerate}

% Suppose there is  a continuous Res-Net $\left(\left\{v_\ell\right\}^{L+1}_{\ell=2},p\right)$ that satisfies Assumptions \ref{ass:5} and \ref{ass:6} and  suppose there is an ideal discrete Res-Net $\left(\bbv, \bbalpha,\bbbeta\right)$ such that 
% \begin{enumerate}[(1)]
% \item  For $i\in[m]$, denote $\bbTheta_i=\left(\bbv_{1,i}, \bbalpha_{2,i}, \dots, \bbalpha_{L,i} \right)$.  $\left\{\bbTheta_i\right\}_{i=1}^m$ i.i.d. follow the distribution $p$.
% \item   For  $\ell\in[L]$ and $i\in[m]$,   $\bbbeta_{\ell,i} =  \frac{1}{d}\left(\bX \bbv_{1,i}\right)  +\sum_{\ell_1=2}^{\ell}\th_2(\balpha_{\ell_1,i}) $.
%     \item  $\linev_{\ell,i,j} = v_{\ell}\left(\bbTheta_i, \bbTheta_j\right)$ for  $\ell\in[2:L]$, $i\in[m]$, and $j\in[m]$.
%         \item  $\linev_{L+1,i,1} = v_{L+1}\left(\bbTheta_i\right)$ where $i\in[m]$.
% \end{enumerate}
\end{definition}
% We assume that the actual discrete Res-Net $\big(\hbv, \hbalpha,\hbbeta\big)$  satisfies the following.
% \begin{enumerate}[(1)]
%     \item  $\| \bbv_{1,i} -  \hbv_{1,i}  \|_{\infty} \leq  \left( 1+  \left\| \bbTheta_{i}\right\|_{\infty}\right)~\ep_1$ for  $i\in[m]$.
%     \item    $\left|\bar{v}_{\ell+1, i,j}- \hv_{\ell+1,i, j}\right| \leq \left(1+ \left\|\bbTheta_{i}\right\|_{\infty}+ \left\|\bbTheta_{j}\right\|_{\infty} \right)~\ep_1$ for  $\ell\in[L-1]$, $i\in[m]$, and $j\in[m]$.
%     \item $\left|\bar{v}_{L+1, i,1}- \hv_{L+1,i, 1}\right|\leq  \left(1+ \left\|\bbTheta_{i}\right\|_{\infty} \right)~\ep_1$ for  $i\in [m]$.
%     \item $\hbalpha$ and $\hbbeta$  satisfy the forward-propagation constrains  \eqref{dfor 00} -- \eqref{dfor 6}. 
% \end{enumerate}

% \end{mdframed}

We compare the discrete and ideal processes:
\begin{itemize}
    \item Actual process $(\hbv^{[0:K]}, \hbalpha^{[0:K]},\hbbeta^{[0:K]})$ by executing Algorithm \ref{algo:RGD} in $K=\frac{T}{\eta}$ steps on the discrete Res-Net from $(\hbv, \hbalpha,\hbbeta)$;
    \item Ideal process $\left(\bbv^{[0,T]}, \bbalpha^{[0,T]},\bbbeta^{[0,T]}\right)$that evolves as neural feature flow:
    \begin{align}
    &~~~~~~~   \bbbeta_{\ell,i}^t\!\!\!\!\! \!\!\!\!\!\!&&=&&\!\!\!\!\!\!\!\!\! \Phi_{*,\ell}^{\bbeta}\left(\bbTheta_{i}\right)(t), \quad\ell\in[L],~ i\in [m],  ~t\in[0,T], \notag\\
    &~~~~~~~   \bbalpha_{\ell,i}^t\!\!\!\!\! \!\!\!\!\!\!&&=&&\!\!\!\!\!\!\!\!\! \Phi_{*,\ell}^{\balpha}\left(\bbTheta_{i}\right)(t), \quad \ell\in[2:L],~ i\in [m],  ~t\in[0,T], \notag\\   
    &~~~~~~~   \bbv_{1,i}^t\!\!\!\!\! \!\!\!\!\!\!&&=&&\!\!\!\!\!\!\!\!\! \Phi_{*,1}^{\bv}\left(\bbTheta_{i}\right)(t), \quad i\in [m],  ~t\in[0,T], \notag\\  
    &~~~~~~~   \bar{v}_{\ell,i,j}^t\!\!\!\!\! \!\!\!\!\!\!&&=&&\!\!\!\!\!\!\!\!\! \Phi_{*,\ell}^{\bv}\left(\bbTheta_{i} ,\bbTheta_{j}\right)(t), \quad \ell\in[2:L],~ i\in [m],i\in[m],  ~t\in[0,T], \notag\\  
    &~~~~~~~   \bar{v}_{L+1,i,1}^t\!\!\!\!\! \!\!\!\!\!\!&&=&&\!\!\!\!\!\!\!\!\! \Phi_{*,L+1}^{\bv}\left(\bbTheta_{i}\right)(t), \quad i\in [m],  ~t\in[0,T].  \notag
    \end{align}
\end{itemize}
We also compare the discrete and the continuous losses denoted by $\hat{\bL}^{k}_R:= \frac{1}{n}\sum_{n=1}^N\phi(\hbbeta_{L+1,1}^{k}(n), y^n )$ and $\bL^{t}_R:=\frac{1}{N}\sum_{n=1}^N\phi\left(\bbeta_{L+1}(\Phi_*,t)(n), y^n \right)$, respectively.

% \begin{enumerate}[(A)]
%     \item Actual process \label{realres} $\left(\hbv^{[0:K]}, \hbalpha^{[0:K]},\hbbeta^{[0:K]}\right)$:  Execute  Algorithm \ref{algo:RGD} in $K$ steps on the discrete Res-Net initialized by $\big(\hbv, \hbalpha,\hbbeta\big)$, where $K =\frac{T}{\eta}$. 
%     \item Ideal process \label{idealres} $\left(\bbv^{[0,T]}, \bbalpha^{[0,T]},\bbbeta^{[0,T]}\right)$: Let  $\left(\bbv^{[0,T]}, \bbalpha^{[0,T]},\bbbeta^{[0,T]}\right)$ evolves as neural feature flow. Namely, we define 
%     \begin{align}
%         &~~~~~~~   \bbbeta_{\ell,i}^t\!\!\!\!\! \!\!\!\!\!\!&&=&&\!\!\!\!\!\!\!\!\! \Phi_{*,\ell}^{\bbeta}\left(\bbTheta_{i}\right)(t), \quad\ell\in[L],~ i\in [m],  ~t\in[0,T], \notag\\
%           &~~~~~~~   \bbalpha_{\ell,i}^t\!\!\!\!\! \!\!\!\!\!\!&&=&&\!\!\!\!\!\!\!\!\! \Phi_{*,\ell}^{\balpha}\left(\bbTheta_{i}\right)(t), \quad \ell\in[2:L],~ i\in [m],  ~t\in[0,T], \notag\\   
%         &~~~~~~~   \bbv_{1,i}^t\!\!\!\!\! \!\!\!\!\!\!&&=&&\!\!\!\!\!\!\!\!\! \Phi_{*,1}^{\bv}\left(\bbTheta_{i}\right)(t), \quad i\in [m],  ~t\in[0,T], \notag\\  
%     &~~~~~~~   \bar{v}_{\ell,i,j}^t\!\!\!\!\! \!\!\!\!\!\!&&=&&\!\!\!\!\!\!\!\!\! \Phi_{*,\ell}^{\bv}\left(\bbTheta_{i} ,\bbTheta_{j}\right)(t), \quad \ell\in[2:L],~ i\in [m],i\in[m],  ~t\in[0,T], \notag\\  
%     &~~~~~~~   \bar{v}_{L+1,i,1}^t\!\!\!\!\! \!\!\!\!\!\!&&=&&\!\!\!\!\!\!\!\!\! \Phi_{*,L+1}^{\bv}\left(\bbTheta_{i}\right)(t), \quad i\in [m],  ~t\in[0,T].  \notag
%     \end{align}
% \end{enumerate}
% We have the following  theorem which bounds the approximation errors using the finite neurons.  
\begin{theorem}\label{theorm:appres}
Under Assumption \ref{ass:5}, suppose $\ep_1\leq \cO(1)$ and $m\geq\tilde{\Omega}(\ep_1^{-2})$, and teat the  parameters in assumptions and $T$ as constants.
Consider  the  actual  process  from an $\ep_1$-independent initialization in Definition \ref{Res-Net condition} with step size $\eta \leq \tO(\ep_1) $. 
Then, the following holds with probability $1-\delta$:
%Under Assumptions \ref{ass:5} and \ref{ass:6}, 
% Assume $\ep_1\leq \cO(1)$ and $\delta\leq1$. Suppose there is an actual discrete Res-Net that satisfies  initial condition in Definition \ref{Res-Net condition}. Let the actual and ideal processes defined  in (\ref{realres}) and (\ref{idealres}), respectively. Teat   the  parameters in Assumptions \ref{ass:5} and \ref{ass:6}  and $T$ as constants.
% Then by setting  $m\geq\tilde{\Omega}(\ep^{-2})$ and the step size $\eta \leq \tO(\ep) $ with probability $1-\delta$, we have
\begin{itemize}
    \item The two processes are close to each other:
$$    \sup_{k\in[0:K]}\bigg\{~\sup_{ i\in [m]}\left\| \hbv^k_{1,i} -\linebv^{k\eta}_{1,i} \right\|_{\infty},~\sup_{ \ell\in[2:L], ~i, j\in[m]}\left| \hv^k_{\ell,i,j} -\linev^{k\eta}_{\ell,i,j} \right|   \bigg\}\leq \tO(\ep_1),$$
$$   \sup_{k\in[0:K],~i\in[m]}\bigg\{\left| \hv^k_{L+1,i,1} -\bar{v}^{k\eta}_{L+1,i,1} \right|  ,~\sup_{\ell\in[2:L]}\left\| \hbalpha^k_{\ell,i} -\bbalpha^{k\eta}_{\ell,i} \right\|_{\infty},
~\sup_{\ell\in[L]}\left\| \hbbeta^k_{\ell,i} -\bbbeta^{k\eta}_{\ell,i} \right\|_{\infty}\bigg\}\leq \tO(\ep_1).$$

\item  The training losses are also close to each other:
% Moreover,  let $\hat{\bL}^{k}_R:= \frac{1}{n}\sum_{n=1}^N\phi(\hbbeta_{L+1,1}^{k}(n), y^n )$ be the loss of running scaled Gradient Descent Algorithm \ref{algo:RGD} at $k$-th step and $\bL^{t}_R:=\frac{1}{N}\sum_{n=1}^n\phi\left(\bbeta_{L+1}(\Phi_*,t)(n), y^n \right)$ be the loss of neural feature flow at time $t$, we have
    $$ \sup_{k\in[0:K]}\left|\hat{\bL}^k_R -  \bL^{k\eta}_R\ \right|  \leq \cO(\ep_1). $$
\end{itemize}
\end{theorem}

One can directly verify that Algorithm \ref{algo:init:sim} produces a discrete Res-Net satisfying the initial condition in Definition \ref{Res-Net condition}.  
Next we show that Algorithm \ref{algo:init:resnet} also produces an $\ep_1$-independent initialization. 
\begin{theorem}\label{ini lemma2}
Under Assumptions \ref{ass:5} and \ref{ass:8}, treat the parameters in assumptions as constants. 
With probability  at least $1-\delta$,     Algorithm \ref{algo:init:resnet} produces an $\ep_1$-independent initial discrete Res-Net with $\ep_1\le \tO(\frac{1}{\sqrt{m}})$. %that satisfies Assumptions \ref{ass:6} and \ref{ass:7} .
% Under Assumption \ref{ass:5},  assume that $h_1$ can achieve Assumption \ref{ass:8}.  Then there is a continuous Res-Net $\left(\left\{v_\ell\right\}^{L+1}_{\ell=2}, p\right)$ that satisfies Assumptions \ref{ass:6} and \ref{ass:7} and   an ideal discrete DNN that satisfies initial condition in Definition \ref{Res-Net condition}. Moreover,    treat the problem-dependent parameters, i.e.,  $L_1$, $L_2$, $L_3$, $N$, $d$, and $\bar{\lambda}$ as  constants. For $\ep_1\leq \tO(1)$ and $\delta\leq 1$, with probability  at least $1-\delta$,     Algorithm \ref{algo:init:resnet} produces a  discrete Res-Net that satisfies the initial condition in Definition  \ref{Res-Net condition} when $m\geq \tilde{\Omega}(\ep^{-2})$.
\end{theorem}

\subsection{Finding Global Minimal Solution}
We study the converge of neural feature flow. In fact, using the same technique as Subsection \ref{sec:convex}, we can also transform the learning problem in \eqref{problem resnet} to a convex optimization.
\begin{theorem}\label{theo:convex res}
 Suppose  $R_R$ can be written in form of
 $$ R_R\left( \left\{v_{\ell}\right\}_{\ell=2}^{L+1}, p\right)  :=   \sum_{\ell=2}^{L+1}  \lambda^{v}_{\ell} {R}_{R,\ell}^{v}( v_{\ell}, p)+ \lambda^{p} R^{p}_R\left(p\right), $$
 where $\{ \lambda^{v}_{\ell}\}_{\ell=2}^{L+1}$  and $\lambda^{p}$   are non-negative and
 \begin{equation}\label{rv}
     R_{R,\ell}^{v} :=
 \begin{cases}
 \int\left[\int  \left|v_{\ell}(\bTheta, \bbTheta)\right| d p(\bbTheta) \right]^r d p(\bTheta),&\quad \ell\in[2:L], \\
 \int\left[v_{L+1}(\bTheta)\right]^r d p(\bTheta),&\quad \ell=L+1, \\
 \end{cases}
 \end{equation}
 $r\geq1$ and $R^p_R$ is convex on $p$.
If  $p$ is equivalent to the Lebesgue measure and $\phi$ is convex in the first argument,  then \eqref{problem resnet} is convex under suitable changes of variables. 
\end{theorem}

However,  we consider a relatively simple case to achieve a global minimal solution here. We assume that  $h_1$ satisfies the strong universal approximation property in Assumption \ref{ass:8}. We show in Theorem \ref{global conver} that the neural feature flow always finds a globally optimal solution when it converges.

\begin{theorem}\label{global conver}
Under Assumptions   \ref{ass:5} -- \ref{ass:8},  assume that the loss function $\phi$ is convex in the first argument.  
%Let , $\Phi_*$, and $\bL^t$ be from Theorem \ref{theorm:appres}.
Let $\Phi_*$ and $\bL^t_R$ be the solution and loss of the neural feature flow in Theorem \ref{theorm:appres}, respectively.
If $\Phi^{\bbeta}_{*, L}(\bTheta)(t)$ converges in $\ell_{\infty}(p)$ and  $\Phi_{*, L+1}^{\bv}(\bTheta)(t)$  converges in $\ell_1(p)$ as $t\to\infty$, where $\bTheta\sim p$,
\iffalse
\begin{enumerate}[(A)]
    \item\label{theo3：A}  there is a function $\Phi_{\infty, L}^{\bbeta}: \supp(p) \to \RR^N$, such that $\Phi^{\bbeta}_{*, L}(\cdot)(t)$ converges to  $\Phi_{\infty, L}^{\bbeta}$ in $\ell_{\infty}$, i.e., $\lim_{t\to\infty}\mathrm{ess}\text{-}\mathrm{sup}_{\bTheta\sim p}\|\Phi^{\bbeta}_{*, L}(\bTheta)(t) -\Phi^{\bbeta}_{\infty, L}(\bTheta)\|_\infty=0$.   
    \item\label{theo3:B} there is a function $\Phi_{\infty, L+1}^{\bv}: \supp(p) \to \RR$, such that  $\Phi^{\bv}_{*, L+1}(\cdot)(t)$ converges to  $\Phi_{\infty, L+1}^{\bv}$ in $\ell_1$, i.e., $\lim_{t\to\infty} \int  \left|\Phi^{\bv}_{*, L+1}(\bTheta)(t) - \Phi^{\bv}_{\infty, L+1}(\bTheta)  \right| d p(\bTheta)  = 0  $.
\end{enumerate}
\fi
% then for $L^{t}$ defined Theorem \ref{theorm:appres},  we have
then we have
$$ \lim_{t\to\infty} \bL^{t}_R  = \sum_{n=1}^N\left[\min_{y'}\phi \left(y', y^{n}\right) \right]. $$
\end{theorem}

Theorem \ref{global conver} is an important application of our  mean-field framework, which shows that neural feature flow can find a global minimizer after it converges. 
We prove that the distribution of the weights in the first layer always has a full support in any finite time by Brouwer's fixed-point theorem. 
Then, using a similar argument to \cite{chizat2018global}, we show that all bad local minima are unstable.
Note that under Assumptions \ref{ass:5} and \ref{ass:8}, the continuous limits of the Res-Nets generated from Algorithms \ref{algo:init:sim} and \ref{algo:init:resnet}, respectively,  can   achieve Assumptions \ref{ass:6} and \ref{ass:7}.  We also note that 
our global convergence   holds for Res-Nets with arbitrary (finite) depth.   Before us, the global convergence result was proved only for two-level NNs \cite{MeiE7665, chizat2018global}, and more recently for three-level ones \cite{nguyen2020rigorous} under a similar convergence assumption on the weights in the second layer.   % It is not clear how to extend their analysis to deeper NNs.

\section{Conclusions and Future Directions}\label{sec:conclu}
This paper proposed a new mean-field framework for DNNs where features in hidden layers have non-vanishing variance. 
We constructed a continuous dynamic called neural feature flow that captures the evolution of sufficiently over-parametrized DNNs trained by Gradient Descent.   We study both the standard DNN and the Res-Net architectures.
Furthermore, for Res-Net, we show that the neural feature flow reaches a globally optimal solution after it converges. 
% Extensions and open problems are discussed in the extended version of this paper in the supplementary material. 
% And it was shown that   
%   the process  finds a global minimum under suitable conditions.  
We hope that our new analytical tool pioneers better understandings for DNN training.

\iffalse
In this paper,  we propose   a new mean-field framework which is based on features to model  over-parameterized DNNs.   We  study both DNNs and Res-Nets and show that the evolutional processes of properly initialized DNNs and Res-Nets  can be captured by a special kind of dynamics equations named as neural feature flow.   We prove the existence and uniqueness of their solutions under mild conditions. Furthermore,  we show that the global minimum  is achievable for the Res-Net architecture  under suitable conditions. 
\fi

% There are still some important questions left.
There are many interesting questions under this framework to be further investigated:
\begin{enumerate}[(A)]
\item It is not clear whether the dynamics of DNNs trained by Gradient Descent can be  characterized by  PDEs of Mckean-Vlason type. 
Recently \cite{araujo2019mean} pointed out the difficulty lied in the potential discontinuity of the conditional distribution under Wasserstein metric. 
From the viewpoint of our framework, the features of the hidden units potentially collide with others along the evolution.
% there potentially exists a worse case when the hidden units encounter with others along the evolution. 
% Is it possible to impose proper regularizers to avoid the bad situation?

\item It is  not answered in this paper  how to analyze the evolution of DNN with  special regularizers such as relative entropy regularizer. 
Can we prove that  Gradient Descent find a global minimum  under such regularizers?
% with  complete mathematical  rigor?

\item  The approximation error bounds established in Theorems \ref{theorm:app} and \ref{theorm:appres}  follow the   ``propagation of chaos'' technique. Such type of  analyses result in  complexities with exponential dependency on time $T$.  It is still not known how to sharpen the complexities even under simple settings. 

\item It would be encouraging to  conduct a deeper analysis on the strong  universal approximation property in Assumption \ref{ass:8}.  
\end{enumerate}

\bibliographystyle{alpha}
\bibliography{overbib2}

\pagebreak\appendix
\pb
\section{Proofs of Theorems \ref{theorm:flow p0} and \ref{theorem:conpsi}}
\subsection{Proof of Theorem \ref{theorm:flow p0}}
In the proof, we first show that our neural feature flow in Definition~\ref{NFLp0} necessarily satisfies several continuity properties in Lemma~\ref{lmm:flow-condition}, which allows us to narrow down the search space for the solution. 
Then we construct a contraction mapping (also known as Picard iteration) to show the existence of uniqueness of solution in that search space. 
Recall that the a trajectory $\Psi$ consists of trajectories of weights $\Psi_\ell^\bw$ for $\ell\in[L+1]$ and features $\Psi_\ell^{\btheta}$ for $\ell\in[2:L]$.
In the proof, we also abbreviate the notations for individual trajectories as
\begin{align*}
 \Psi_\ell^{\bw}(\bu_\ell)(t)&=\Psi(w_\ell(\bu_\ell),t)=w_\ell^t(\bu_\ell),\\
 \Psi_\ell^{\btheta}(\btheta_\ell)(t)&=\Psi(\btheta_\ell,t)=\btheta_\ell^t,
\end{align*}
where $\bu_\ell$ stands for $\bw_1$, $(\bw_1,\btheta_2)$, $(\btheta_{\ell-1},\btheta_{\ell})$, $\btheta_L$ for $\ell=1$, $\ell=2$, $3\le\ell\le L$, $\ell=L+1$, respectively. 
Throughout the proof, we fix $T$ as a constant. 

For a precise statement of the continuity property of the neural feature flow, we first define the set of continuous trajectories: 
\begin{definition}[$(\bC,\bC')$-Continuous Trajectory]\label{special Function 2}
Given $\bC:=(\bC_1,\dots,\bC_{L+1})\in\RR_+^{L+1}$ and $\bC':=(\bC_2',\dots,\bC_{L}')\in\RR_+^{L-1}$, we say $\Psi$ is $(\bC,\bC')$-continuous 
if $\Psi_\ell^{\bw}(\bu_\ell)(t)$ is $\bC_\ell$-Lipschitz continuous in $t\in[0,T]$ for $\ell\in[L+1]$, 
and $\Psi_\ell^{\btheta}(\btheta_\ell)(t)$ is $\bC_\ell'(1+\|\btheta_\ell\|_\infty)$-Lipschitz continuous in $t\in[0,T]$ for $\ell\in[2:L-1]$. 
The set of all $(\bC,\bC')$-continuous trajectories is denoted as $\mathbf{\Psi}^{(\bC,\bC')}$.
\end{definition}

\begin{lemma}
\label{lmm:flow-condition}
There exists constants $\bC\in\RR_+^{L+1}$ and $\bC'\in\RR_+^{L-1}$ such that every solution $\Psi$ of the neural feature flow is $(\bC,\bC')$-continuous.
\end{lemma}

In the remaining of the proof we let $\bC$ and $\bC'$ be constants in Lemma~\ref{lmm:flow-condition}, and let $\mathbf{\Psi}:=\mathbf{\Psi}^{(\bC,\bC')}$, which will serve as the search space. 
The solution can be equivalently characterized as the fixed-point of a mapping from $\mathbf{\Psi}$ to itself that we introduce next:

\begin{definition}\label{def:F}
Define $F: \mathbf{\Psi} \to  \mathbf{\Psi}$ as follows: for all $t\in[0,T]$,
\begin{enumerate}[(1)]
    \item for all $\ell\in[2:L-1]$ and all $\btheta_{\ell}$,  $$F(\Psi)_{\ell}^{\btheta}(\btheta_{\ell})(t) = \btheta_{\ell} -\int_{0}^t \lineGrad^{\btheta}_{\ell}\left( \btheta_\ell; \Psi,s\right)ds; $$
    \item for all $\ell\in[L+1]$ and all $\bu_\ell$,
    $$F(\Psi)_{\ell}^\bw(\bu_\ell)(t) = w_{\ell}(\bu_\ell)-\int_{0}^t\lineGrad^{\bw}_{\ell}( \bu_\ell;\Psi,s)ds.$$ 
\end{enumerate}
\end{definition}
It follows from the same argument as Lemma~\ref{lmm:flow-condition} that the image of $\mathbf{\Psi}$ under $F$ is indeed contained in $\mathbf{\Psi}$. 
Comparing the definition of neural feature flow in Definition~\ref{NFLp0}, it is clear that finding a solution of neural feature flow in $\mathbf{\Psi}$ is equivalent to finding a fixed-point of $F$.
We will show in Lemma~\ref{lmm:contraction} the contraction property of $F$ under an appropriate metric defined below:

\begin{definition}
\label{def:metric}
For a pair $\Psi_1,\Psi_2\in  \mathbf{\Psi}$, we define the normalized distance between each trajectories over $[0,t]$ as 
\begin{align*}
& \rho_\bw^{[0,t]}(\Psi_{1,\ell}^\bw,\Psi_{2,\ell}^\bw)
:=\sup_{s\in[0,t],\bu_\ell}
\frac{\|\Psi_{1,\ell}^\bw(\bu_\ell)(s)
-\Psi_{2,\ell}^\bw(\bu_\ell)(s)\|_\infty}
{1+\|\bu_\ell\|_\infty},\\
& \rho_{\btheta}^{[0,t]}(\Psi_{1,\ell}^{\btheta},\Psi_{2,\ell}^{\btheta})
:=\sup_{s\in[0,t],\btheta_\ell}
\frac{\|\Psi_{1,\ell}^{\btheta}(\btheta_\ell)(s)
-\Psi_{2,\ell}^{\btheta}(\btheta_\ell)(s)\|_\infty}
{1+\|\btheta_{\ell}\|_\infty}.
\end{align*}
Finally we define the distance between $\Psi_1$ and $\Psi_2$ as 
\[
\Dis^{[0,t]}(\Psi_1,\Psi_2)
:=\max\left\{
\max_{\ell\in[L+1]}\rho_\bw^{[0,t]}(\Psi_{1,\ell}^\bw,\Psi_{2,\ell}^\bw), \max_{\ell\in[2:L-1]}\rho_{\btheta}^{[0,t]}(\Psi_{1,\ell}^{\btheta},\Psi_{2,\ell}^{\btheta})
\right\}.
\]
\end{definition}

\begin{lemma}
\label{lmm:contraction}
There exists a constant $C$ such that
\begin{eqnarray}
\Dis^{[0,t]}(F(\Psi_1),F(\Psi_2) )\leq C \int_{0}^t \Dis^{[0,s]}(\Psi_1,\Psi_2 ) ds. \notag
\end{eqnarray}
\end{lemma}

\begin{proof}[Proof of Theorem \ref{theorm:flow p0}]
Firstly, it is clear that $\mathbf{\Psi}$ contains the constant trajectory and thus is nonempty.
Applying Lemma~\ref{lmm:contraction}, the proof of existence and uniqueness  follows from a similar argument of Picard–Lindel\"of theorem.
Specifically, iteratively applying Lemma~\ref{lmm:contraction} yields that
\[
\Dis^{[0,T]}(F^m(\Psi_1),F^m(\Psi_2) )\leq \frac{(CT)^m}{m!} \Dis^{[0,T]}(\Psi_1,\Psi_2 ).
\]
 Let $\Psi$ be  the constant trajectory,  for any $\tPsi\in \mathbf{\Psi}$,  by  the upper bounds of  $\lineGrad^{\btheta}_{\ell}$ and $\lineGrad^{\bw}_{\ell}$ in Lemma \ref{lmm:Psi-property} and the Definition of $\Dis^{[0,T]}$ in Definition \ref{def:metric}, there is a constant $C$ such that 
$$ \Dis^{[0,T]}(F(\tPsi),\Psi) \leq CT<\infty.$$

We first show the uniqueness. 
For two fixed points of $F$ denoted by $\Psi_1$ and $\Psi_2$, we have 
\[
\Dis^{[0,T]}(\Psi_1,\Psi_2)
=\Dis^{[0,T]}(F^m(\Psi_1),F^m(\Psi_2))
\le \frac{(CT)^{m-1}}{(m-1)!} \Dis^{[0,T]}(F(\Psi_1),F(\Psi_2)),
\]
By the triangle inequality $\Dis^{[0,T]}(F(\Psi_1),F(\Psi_2))\leq  \Dis^{[0,T]}(F(\Psi_1),\Psi) +  \Dis^{[0,T]}(F(\Psi_2),\Psi)<\infty $,  hence
the right-hand side of the above inequality vanishes as $m$ diverges.
For the existence, we consider the sequence $\{F^i(\Psi):i\ge 0\}$ that satisfies 
\[
\Dis^{[0,T]}(F^{m+1}(\Psi),F^m(\Psi))
\le \frac{(CT)^m}{m!} \Dis^{[0,T]}(F(\Psi),\Psi),
\]
Because $\Dis^{[0,T]}(F(\Psi),\Psi)<\infty$,   $\{F^i(\Psi):i\ge 0\}$ is a Cauchy sequence. 
Since $\mathbf{\Psi}$ is complete under $\Dis^{[0,T]}$ by Lemma~\ref{lmm:psi-complete}, the limit point $\Psi_*\in \mathbf{\Psi}$, which is a fixed-point of $F$.
\end{proof}

\subsection{Proof of Theorem \ref{theorem:conpsi}}
Theorem~\ref{theorem:conpsi} is a Gr\"onwall-type of result. 
However, it is not straightforward to directly derive a simple differential inequality due to the involved relations among the parameters of deep neural networks. 
Again we turn to the technique of Picard iterations used in the proof of Theorem~\ref{theorm:flow p0}. This approach has also been used to prove the abstract Gr\"onwall inequality in \cite{turinici1986abstract}.

Recall the set $\mathbf{\Psi}$ in the proof of Theorem~\ref{theorm:flow p0}, and the mapping $F:\mathbf{\Psi}\mapsto\mathbf{\Psi}$ in Definition~\ref{def:F}. 
It is shown that $F$ is a contraction mapping and thus there exists a unique solution $\Psi_*\in \mathbf{\Psi}$. 
We will construct a closed nonemtpy subset $\btPsi \subseteq \bfPsi$ with the desired properties in Theorem~\ref{theorem:conpsi} such that $F(\btPsi)\subseteq \btPsi$.
By the same argument as the proof of Theorem~\ref{theorm:flow p0}, there exists a solution in $\btPsi$, thereby proving $\Psi_* \in \btPsi$.

Next we introduce the set of $\beta$-locally Lipschitz trajectories with the desired properties in Theorem \ref{theorem:conpsi}.
We use similar notations as in the proof of Theorem~\ref{theorm:flow p0} by letting $\bu_\ell$ denote $\bw_1$, $(\bw_1,\btheta_2)$, $(\btheta_{\ell-1},\btheta_{\ell})$, $\btheta_L$ for $\ell=1$, $\ell=2$, $3\le\ell\le L$, $\ell=L+1$, respectively.

\begin{definition}[$\beta$-Locally Lipschitz Trajectory]
\label{def:local-lip}
Recall the constants $C_2$ and $C_4$ in Assumption~\ref{ass:3} for the locally Lipschitz continuity at $t=0$.
We say $\Psi$ is $\beta$-locally Lipschitz if for all $t\in[0,T]$,  $\bw_1$, $\bbw_1\in\mathcal{B}_{\infty}(\bw_1,1) $, $\btheta_{\ell}$, and $\bbtheta_{\ell}\in \mathcal{B}_{\infty}(\btheta_{\ell},1) $ with $\ell\in[2:L]$, we have
\begin{subequations}
\begin{align}
&\left\|\Psi^\bw_{1}(\bw_1)(t)-\Psi^\bw_{1}(\bbw_1)(t) \right\|_{\infty}
\leq  e^{\beta t}( \|\bw_1\|_{\infty}+1 )\|\bw_1 -\bbw_1 \|_{\infty},\label{local lip-w1}\\
&\left|\Psi^\bw_{L+1}( \btheta_{L})(t) - \Psi^\bw_{L+1}( \bbtheta_{L})(t)  \right|
\leq (1+C_4)e^{\beta t} (  \|\btheta_L\|_{\infty}+1 )  \|\btheta_{L}-\bbtheta_{L} \|_{\infty},\label{local lip-wL+1}\\
&\left|\Psi^\bw_{\ell}(\bu_\ell)(t) - \Psi^\bw_{\ell}(\bbtheta_{\ell-1}, \btheta_\ell)(t)  \right|
\leq (1+C_2)e^{\beta t}( \|\bu_{\ell}\|_{\infty}+1)  \|\btheta_{\ell-1}-\bbtheta_{\ell-1} \|_{\infty},\label{local lip-wl1}\\
&\left|\Psi^\bw_{\ell}(\bu_\ell)(t) - \Psi^\bw_{\ell}(\btheta_{\ell-1}, \bbtheta_\ell)(t)  \right|
\leq (1+C_2)e^{\beta t}( \|\bu_{\ell}\|_{\infty}+1) \|\btheta_{\ell}-\bbtheta_{\ell} \|_{\infty},\label{local lip-wl2}\\
&\left\|\Psi^{\btheta}_{\ell}(\btheta_{\ell})(t)-\Psi^{\btheta}_{\ell}(\bbtheta_{\ell})(t) \right\|_{\infty}
\leq  e^{\beta t}( \|\btheta_\ell\|_{\infty}+1 )\|\btheta_\ell -\bbtheta_{\ell} \|_{\infty},\label{local lip-tl}
\end{align}
\end{subequations}
for $\ell\in[2:L]$.
Denote the set of all $\beta$-locally Lipschitz trajectories as $\bfPsi_{\beta}$.
\end{definition}

\begin{lemma}
\label{const Lstar}
There exists a constant $\beta_*$ such that $F(\bfPsi\cap \bfPsi_{\beta_*})\subseteq   \bfPsi_{\beta_*}$.
\end{lemma}

\begin{proof}[Proof of Theorem~\ref{theorem:conpsi}]
Let $\beta_*$ be the constant in Lemma~\ref{const Lstar} and $\bfPsi':=\bfPsi\cap \bfPsi_{\beta_*}\subseteq \bfPsi$, which clearly contains the constant trajectory and thus is nonempty. 
It follows from Lemma~\ref{const Lstar} that $F(\bfPsi')\subseteq \bfPsi'$. 
Since $F$ is a contraction mapping by Lemma~\ref{lmm:contraction} and $\bfPsi'$ is a closed set by Lemma~\ref{psi closed}, by the same argument as the proof of Theorem~\ref{theorm:flow p0}, there exists a unique solution in $\bfPsi'$, which is necessarily $\Psi_*$ by the uniqueness of the solution in Theorem~\ref{theorm:flow p0}. 
\end{proof}

\subsection{Proofs of Lemmas}
\begin{proof}[Proof of Lemma~\ref{lmm:flow-condition}]
To prove the Lipschitz continuity of $\Psi$ in time, by the definition of neural feature flow in Definition~\ref{NFLp0}, it suffices to show upper bounds of $\lineGrad^{\bw}_{\ell}$ and $\lineGrad^{\btheta}_{\ell}$ for each layer $\ell$. 
In the following, we use the backward equations to inductively upper bound $\lineDel_{\ell}$ and thus $\lineGrad^{\bw}_{\ell}$ from $\ell=L+1$ to $1$, and then use the forward equations to upper bound $\lineGrad^{\btheta}_{\ell}$ from $\ell=1$ to $L$.

We first consider the backward steps. 
We will focus on the proof of $\|\lineDel_{\ell}\|_\infty\le \tilde 
\bC_\ell$ for constants $\tilde 
\bC_\ell$ to be specified, which immediately yield upper bounds $\|\lineGrad^{\bw}_\ell\|_\infty\le \bC_{\ell}$ for constants $\bC_\ell$ since both $h$ and $\bX$ are bounded.
For the top layer $\ell=L+1$, by Assumption \ref{ass:2} that $|\phi_1'|\leq L_4$, we have
\[
\|\lineDel_{L+1}(\Psi,t)\|_{\infty} 
\le L_4 := \tilde \bC_{L+1}.\]
At layer $\ell=L$, since $|h'|\le L_2$, 
\begin{eqnarray}
  \left\|\lineDel_{L}(\btheta_{L}; \Psi,t) \right\|_{\infty} 
  = \left|w_{L+1}^t(\btheta_L)|\|\lineDel_{L+1}(\Psi,t)\right\|_{\infty} \left\|\th'\left(\btheta_{L}^t\right)\right\|_{\infty}
  \le \tilde \bC_{L},\notag
\end{eqnarray}
where $\tilde \bC_{L}:= (C_3+\bC_{L+1}T)L_2\tilde \bC_{L+1}$ and $|w_{L+1}^t|\le C_3+\bC_{L+1}t\le C_3+\bC_{L+1}T$ by the upper bound of initialization \eqref{ass45} and the $\bC_{L+1}$-Lipschitz continuity of $w_{L+1}^t$ in $t$. 
For each $\ell=L-1,\dots,2$, we similarly apply the upper bounds of initialization in \eqref{ass42} and the $\bC_{\ell+1}$-Lipschitz continuity of $w_{\ell+1}^t$ in $t$ and obtain that
\begin{eqnarray}
 &&\int  |w_{\ell+1}^t(\btheta_{\ell},\btheta_{\ell+1})| \|\lineDel_{\ell+1}(\btheta_{\ell+1}; \Psi,t)\|_\infty  dp_{\ell+1}(\btheta_{\ell+1}) \label{Del bound}\\
 &\leq&\tilde \bC_{\ell+1}\int\left|w_{\ell+1}^t(\btheta_{\ell},\btheta_{\ell+1})\right| dp_{\ell+1}(\btheta_{\ell+1})  \notag\\
  &\leq&  \tilde \bC_{\ell+1}\left(\bC_{\ell+1}t+  \int |w_{\ell+1}(\btheta_{\ell},\btheta_{\ell+1})|dp_{\ell+1}(\btheta_{\ell+1})  \right)   \notag\\
  &\leq&\tilde \bC_{\ell+1}(\bC_{\ell+1}T + C'),\notag
\end{eqnarray}
for a constant $C'$,
where in the last inequality we used the upper bound of $w_{\ell+1}$ in \eqref{ass42}, the sub-gaussian property of $p_{\ell+1}$ in Assumption~\ref{ass:4}, and Corollary~\ref{cor:subg-q}.
Consequently,
\begin{eqnarray}
  \left\|\lineDel_{\ell}(\btheta_{\ell}; \Psi,t)\right\|_{\infty} &\leq& \left\|\int  w_{\ell+1}^t(\btheta_{\ell},\btheta_{\ell+1})~ \lineDel_{\ell+1}(\btheta_{\ell+1}; \Psi,t)  dp_{\ell+1}(\btheta_{\ell+1})\right\|_{\infty} \left\|\th'\left(\btheta_{\ell}^t\right)\right\|_{\infty}
  \leq \tilde \bC_{\ell}, \notag
\end{eqnarray}
where $\tilde \bC_{\ell}:= ( C'+\bC_{\ell+1}T)L_2\tilde \bC_{\ell+1}$. 
For $\ell=1$, the upper bound can be obtained by replacing $\btheta_1$ by $\bw_1$ in \eqref{Del bound}.

Now we consider the forward steps and upper bound $\lineGrad^{\btheta}_{\ell}$.
For the first layer $\ell=1$, since $\bX$ is bounded, it follows from \eqref{grad11} that 
\begin{eqnarray}\label{qq18}
  \left\|\lineGrad^{\btheta}_{1}(\bw_1;\Psi,t)\right\|_{\infty} 
  \leq \bC_1'
  \le \bC_1'(1+\|\bw_1\|_\infty),
\end{eqnarray}
for a constant $\bC_1'$.  Next we prove for $\ell\ge 2$. The  analysis uses the notations for $\ell\ge 3$, and for the $\ell=2$ case $\btheta_1$ should be replaced by $\bw_1$.  By a similar argument to \eqref{Del bound}, we have
\begin{eqnarray}
 && \int |w_{\ell}^t(\btheta_{\ell-1}, \btheta_{\ell})| \left\|\lineGrad^{\btheta}_{\ell-1}(\btheta_{\ell-1}; \Psi,t)\right\|_{\infty}  d p_{\ell-1}(\btheta_{\ell-1})\notag\\
    &\leq& \bC_{\ell-1}'\int  (\| \btheta_{\ell-1}\|_{\infty}+1 )(\bC_\ell T + \left|w_{\ell}(\btheta_{\ell-1}, \btheta_{\ell}) \right| ) d p_{\ell-1}(\btheta_{\ell-1})   \notag\\
    &\overset{\eqref{ass42}}{\leq}& \bC_{\ell-1}'(\bC_\ell T +  C_1 (\| \btheta_{\ell}\|_{\infty}+1)) \int  (\| \btheta_{\ell-1}\|_{\infty}+1)   d p_{\ell-1}(\btheta_{\ell-1})     \notag\\
    & \le& \tilde\bC_\ell'(\| \btheta_{\ell}\|_{\infty}+1),\notag
\end{eqnarray}
for some constant $\tilde\bC_\ell'$.
Therefore, applying \eqref{grad12} yields that
\begin{align}
 \left\|\lineGrad^{\btheta}_{\ell}( \btheta_{\ell} ;\Psi, t)\right\|_{\infty}
&\leq L_2\tilde\bC_\ell'(\| \btheta_{\ell}\|_{\infty}+1) 
+ \int \underbrace{\left\|\th\left(\btheta_{\ell-1}^t\right)\right\|_{\infty}}_{\leq L_1}\underbrace{\left| \lineGrad^{\bw}_{\ell}(\btheta_{\ell-1}, \btheta_{\ell};\Psi,t)\right|}_{\leq\bC_\ell} d p_{\ell-1}(\btheta_{\ell-1}) \nonumber\\
&\leq \bC_\ell'(\| \btheta_{\ell}\|_{\infty}+1),\label{tend1}
\end{align}
for some constant $\bC_\ell'$.
\end{proof}

Before proving Lemma~\ref{lmm:contraction}, we first present in Lemma~\ref{lmm:Psi-property} properties of $\Psi \in \mathbf{\Psi}$ that will be used to prove the contraction lemma. 
The proof is exactly the same as Lemma~\ref{lmm:flow-condition} and is omitted. 
\begin{lemma}[Property of $\mathbf{\Psi}$]
\label{lmm:Psi-property}
There exist constants $\tilde\bC_\ell$, $\bC_\ell$, and $\bC_\ell'$ such that, for any $\Psi\in\mathbf{\Psi}$, we have
\begin{itemize}
    \item $\|\lineDel_{L+1}(\Psi,t)\|_\infty \le \tilde\bC_{L+1}$ and $\|\lineDel_{\ell}(\btheta_\ell;\Psi,t)\|_\infty \le \tilde\bC_\ell$ for $\ell\in[L]$;
    \item $\|\lineGrad^{\bw}_{\ell}( \bu_\ell;\Psi,t)\|_\infty \le \bC_\ell$ and $\|w_{\ell}^t(\bu_\ell)\|_\infty \le \|w_{\ell}^0(\bu_\ell)\|_\infty + \bC_\ell t$ for $\ell\in[L+1]$;
    \item $\|\lineGrad^{\btheta}_{\ell}( \btheta_\ell; \Psi,t)\|_\infty \le\bC_\ell'(\| \btheta_\ell\|_{\infty} +1)$ for $\ell\in[2:L-1]$.
\end{itemize}
\end{lemma}

\begin{proof}[Proof of Lemma~\ref{lmm:contraction}]
The proof entails upper bounds of the gradient differences 
$\| \lineGrad^{\btheta}_{\ell}( \btheta_\ell; \Psi_1,t)-\lineGrad^{\btheta}_{\ell}( \btheta_\ell; \Psi_2,t)\|_\infty$ 
and 
$\|\lineGrad^{\bw}_{\ell}( \bu_\ell;\Psi_1,t)-\lineGrad^{\bw}_{\ell}( \bu_\ell;\Psi_2,t)\|_{\infty}$ in terms of the differences $\|\btheta_{1,\ell}^t-\btheta_{2,\ell}^t\|_\infty$ for $\ell\in[2:L]$ and $|w_{1,\ell}^t-w_{2,\ell}^t|$ for $\ell\in[L+1]$, which can all be further upper bounded in terms of $d_t:=\Dis^{[0,t]}(\Psi_1,\Psi_2 )$, that is, by definition, 
\begin{align}
&\left\|\btheta_{1,\ell}^t - \btheta_{2,\ell}^t\right\|_{\infty}\leq(\| \btheta_\ell\|_{\infty} +1)d_t,\quad \ell\in[2:L],\label{eq:dt-theta}\\
&\left\| w_{1,\ell}^t(\bu_\ell) - w_{2,\ell}^t(\bu_\ell)\right\|_{\infty} \leq(\| \bu_\ell\|_{\infty} +1)d_t\label{eq:dt-w},\quad \ell\in [L+1].    
\end{align}
Analogous to the proof of Lemma~\ref{lmm:flow-condition}, we will use the backward equations to inductively upper bound the differences between $\lineDel_{\ell}$ and thus between $\lineGrad^{\bw}_{\ell}$ from $\ell=L+1$ to $1$, and then use the forward equations to upper bound the difference between $\lineGrad^{\btheta}_{\ell}$ from $\ell=1$ to $L$. 
Specifically, we will prove that, for some constant $C$, 
\begin{align*}
&\|\lineDel_{\ell}(\btheta_\ell;\Psi_1,t) -\lineDel_{\ell}(\btheta_\ell; \Psi_2,t)\|_\infty
\le C (1+\|\btheta_\ell\|_\infty) d_t,\quad \ell\in[L],\\
&\|\lineGrad^{\bw}_{\ell}( \bu_\ell;\Psi_1,t)-\lineGrad^{\bw}_{\ell}( \bu_\ell;\Psi_2,t)\|_{\infty}
\le C (1+\|\bu_\ell\|_\infty) d_t,\quad \ell\in[L+1],\\
&\| \lineGrad^{\btheta}_{\ell}( \btheta_\ell; \Psi_1,t)-\lineGrad^{\btheta}_{\ell}( \btheta_\ell; \Psi_2,t)\|_\infty
\le C (1+\|\btheta_\ell\|_\infty)d_t,\quad \ell\in[2:L].
\end{align*}
Then the conclusion follows from the definition of $F$ and $\Dis^{[0,t]}$ in Definitions~\ref{def:F} and \ref{def:metric}, respectively.

We first consider the backward steps. 
Again we focus on the upper bound of the difference between $\lineDel_{\ell}$. 
Since both $h$ and $\bX$ are bounded, $h$ is Lipschitz continuous by assumption, and $\lineDel_{\ell}$ is bounded by Lemma~\ref{lmm:Psi-property}, the upper bound of the difference between $\lineGrad^{\bw}_{\ell}$ follows immediately. 
For the top layer $\ell=L+1$, the Lipschitz continuity of $\phi'_1$ in Assumption~\ref{ass:2} implies that,
\begin{eqnarray}
&&\|\lineDel_{L+1}(\Psi_1,t) -\lineDel_{L+1}(\Psi_2,t)\|_\infty\label{qwe}\\
&\le& L_5 \|\btheta_{1,L+1}^t - \btheta_{2,L+1}^t\|_\infty \notag\\
&\le& L_5 \int \| h(\btheta_{1,L}^t)w^t_{1,L+1}(\btheta_{L}) -  h(\btheta_{2,L}^t)w^t_{2,L+1}(\btheta_{L})\|_\infty d p_L(\btheta_L) \notag.
\end{eqnarray}
Since $h$ is bounded and Lipschitz continuous, $w_{i,L+1}^t$ is bounded for $t\le T$ by Lemma~\ref{lmm:Psi-property}, we have
\[
\|\lineDel_{L+1}(\Psi_1,t) -\lineDel_{L+1}(\Psi_2,t)\|_\infty
\le \tilde \bC_{L+1} d_t,
\]
for a constant $\tilde \bC_{L+1}$.
At layer $\ell=L$, recall that
$$\lineDel_{L}(\btheta_{L}; \Psi,t) =    w_{L+1}^t(\btheta_L)~ \lineDel_{L+1}(\Psi,t)\cdot    \th'\left(\btheta_{L}^t\right).$$
Since the three terms in the product are all bounded, and $h'$ is $L_3$-Lipschitz continuous, we have
\begin{eqnarray}\label{cos r}
 \|\lineDel_{L}(\btheta_{L}; \Psi_1,t)-\lineDel_{L}(\btheta_{L}; \Psi_2,t)\|_\infty
\le \tilde\bC_{L}(1+\|\btheta_L\|_\infty) d_t,
\end{eqnarray}
for a constant $\tilde \bC_{L}$.
For each $\ell=L-1,\dots,1$, 
\begin{eqnarray}
 &&\int  \|w_{1,\ell+1}^t(\btheta_{\ell},\btheta_{\ell+1})  ~\lineDel_{\ell+1}(\btheta_{\ell+1}; \Psi_1,t) -  w_{2,\ell+1}^t(\btheta_{\ell},\btheta_{\ell+1})  ~\lineDel_{\ell+1}(\btheta_{\ell+1}; \Psi_2,t)\|_{\infty}dp_{\ell+1}(\btheta_{\ell+1})  \notag\\
 &\leq&\int \underbrace{\left| w_{1,\ell+1}^t(\btheta_{\ell},\btheta_{\ell+1}) -w_{2,\ell+1}^t(\btheta_{\ell},\btheta_{\ell+1})\right|}_{\leq  (\|\btheta_{\ell} \|_{\infty}+\|\btheta_{\ell+1} \|_{\infty}+1 )d_t } \left\| \lineDel_{\ell+1}(\btheta_{\ell+1}; \Psi_1,t) \right\|_{\infty}  \notag\\
 &&\quad\quad\quad\quad\quad\quad+ \left|w_{2,\ell+1}^t(\btheta_{\ell},\btheta_{\ell+1}) \right| \underbrace{\left\|\lineDel_{\ell+1}(\btheta_{\ell+1}; \Psi_1,t) -  \lineDel_{\ell+1}(\btheta_{\ell+1}; \Psi_2,t)\right\|_{\infty}}_{\leq \tilde \bC_{\ell+1}(\|\btheta_{\ell+1} \|_{\infty}+1 )d_t  }d p_{\ell+1}(\btheta_{\ell+1})\notag\\
 &\leq& C' (\| \btheta_{\ell}\|_{\infty} +1 )  d_t, \label{qq1}
\end{eqnarray}
for a constant $C'$, where the last step is due to the sub-gaussianness of $p_{\ell+1}$, Corollary~\ref{cor:subg-q}, and the upper bounds of $\lineDel_{\ell+1}$ and $w_{\ell+1}^t$ in Lemma~\ref{lmm:Psi-property}.
Consequently, 
\[
\|\lineDel_{\ell}(\btheta_{\ell}; \Psi_1,t)-\lineDel_{\ell}(\btheta_{\ell}; \Psi_2,t)\|_\infty
\le \tilde\bC_{\ell}(1+\|\btheta_\ell\|_\infty) d_t,
\]
for a constant $\tilde\bC_{\ell}$.

Now we turn to the forward steps and upper bound $\| \lineGrad^{\btheta}_{\ell}( \btheta_\ell; \Psi_1,t)-\lineGrad^{\btheta}_{\ell}( \btheta_\ell; \Psi_2,t)\|_\infty$.
The case $\ell=1$ follows from the boundedness of $\bX$.   Next we prove for $\ell\ge 2$. The following analysis uses the notations for $\ell\ge 3$, and for the $\ell=2$ case $\btheta_1$ should be replaced by $\bw_1$. For each layer $\ell\in [3:L]$, we consider the two terms in \eqref{grad12} separately. 
For the first term, since $h'$ is bounded and Lipschitiz, we apply the upper bound of $\lineGrad^{\btheta}_{\ell-1}$ in Lemma~\ref{lmm:Psi-property} and obtain that 
\begin{equation}
\left\|\th'\left(\btheta_{1,\ell-1}^t\right)\cdot \lineGrad^{\btheta}_{\ell-1}(\btheta_{\ell-1}; \Psi_1,t)  -\th'\left(\btheta_{2,\ell-1}^t\right)\cdot \lineGrad^{\btheta}_{\ell-1}(\btheta_{\ell-1}; \Psi_2,t)    \right\|_{\infty}
\le C' (\| \btheta_{\ell-1}\|_{\infty}+1)^2 d_t,\label{tu2}
\end{equation}
which further implies that 
\begin{eqnarray}
 && \int \|w_{1,\ell}^t(\btheta_{\ell-1}, \btheta_{\ell})\th'\left(\btheta_{1,\ell-1}^t\right)\cdot \lineGrad^{\btheta}_{\ell-1}(\btheta_{\ell-1}^t; \Psi_1,t)  \notag\\
 &&\quad\quad\quad\quad\quad\quad\quad\quad\quad\quad\quad
 -w_{2,\ell}^t(\btheta_{\ell-1}, \btheta_{\ell})\th'\left(\btheta_{2,\ell-1}^t\right)\cdot \lineGrad^{\btheta}_{\ell-1}(\btheta_{\ell-1}; \Psi_2,t)\|_{\infty}  dp_{\ell-1}(\btheta_{\ell-1}) \notag\\
 &\leq&  \int C' (\| \btheta_{\ell-1}\|_{\infty}+1)^2 d_t \left|w_{1,\ell}^t(\btheta_{\ell-1}, \btheta_{\ell}) \right|dp_{\ell-1}(\btheta_{\ell-1})\notag\\
 &&\quad\quad\quad\quad+\int 
 \underbrace{\left\| \th'\left(\btheta_{2,\ell-1}^t\right)\cdot \lineGrad^{\btheta}_{\ell-1}(\btheta_{\ell-1}; \Psi_2,t) \right\|_{\infty}}_{\leq C' (\|\btheta_{\ell-1}\|_{\infty}+1) }\underbrace{\left| w_{1,\ell}^t(\btheta_{\ell-1}, \btheta_{\ell}) -w_{2,\ell}^t(\btheta_{\ell-1}, \btheta_{\ell}) \right|}_{\leq (\|\btheta_{\ell-1}\|_{\infty}+\|\btheta_{\ell}\|_{\infty}+1)d_t } dp_{\ell-1}(\btheta_{\ell-1})\notag\\
 &\leq& C'' (\|\btheta_{\ell} \|_{\infty}+1)d_t, \notag
\end{eqnarray}
for a constant $C''$, where in the last step we used the sub-gaussianness of $p_{\ell-1}$, Corollary~\ref{cor:subg-q}, and the upper bound of $\w^t_{\ell}$ in Lemma~\ref{lmm:Psi-property}.
For the second term of \eqref{grad12}, we apply the upper bound of $\lineGrad^{\bw}_{\ell-1}$ in Lemma~\ref{lmm:Psi-property} and obtain that
\begin{eqnarray}
&& \int \|\lineGrad^{\bw}_{\ell-1}(\btheta_{\ell-1}, \btheta_{\ell};\Psi_1,t)~ \th\left(\btheta_{1,\ell-1}^t\right)  -  \lineGrad^{\bw}_{\ell-1}(\btheta_{\ell-1}, \btheta_{\ell};\Psi_2,t)~ \th\left(\btheta_{2,\ell-1}^t\right)\|_{\infty} dp_{\ell-1}(\btheta_{\ell-1}) \notag\\
&\leq& \tilde C'\int (\| \btheta_{\ell-1}\|_{\infty}+\|\btheta_{\ell} \|_{\infty}+1) + (\| \btheta_{\ell-1}\|_{\infty}+1)d_t dp_{\ell-1}(\btheta_{\ell-1})\notag\\
&\leq&\tilde C'' (\|\btheta_{\ell} \|_{\infty}+1)d_t. \notag
\end{eqnarray}
We conclude that
\[
\| \lineGrad^{\btheta}_{\ell}( \btheta_\ell; \Psi_1,t)-\lineGrad^{\btheta}_{\ell}( \btheta_\ell; \Psi_2,t)\|_\infty
\le \tilde \bC_\ell'(\|\btheta_{\ell} \|_{\infty}+1)d_t,
\]
for a constant $\tilde \bC_\ell'$. %For $\ell=2$, the upper bound can be obtained by replacing $\btheta_1$ by $\bw_1$.

\end{proof}

\begin{lemma}
\label{lmm:psi-complete}
$\mathbf{\Psi}$ is complete under $\Dis^{[0,T]}$.    
\end{lemma}
\begin{proof}
Let $\{\Psi_n:n\ge 0\}$ be a Cauchy sequence under $\Dis^{[0,T]}$. 
Then $\frac{\Psi_{n,\ell}^\bw(\bu_\ell)(t)}{1+\|\bu_\ell\|_\infty}$ and $\frac{\Psi_{n,\ell}^{\btheta}(\btheta_\ell)(t)}{1+\|\btheta_\ell\|_\infty}$ converge uniformly under the $\ell_\infty$-norm. 
Let 
\begin{align*}
& \Psi_{*,\ell}^\bw(\bu_\ell)(t)=\lim_{n\to\infty}\Psi_{n,\ell}^\bw(\bu_\ell)(t),\\
& \Psi_{*,\ell}^{\btheta}(\btheta_\ell)(t)=\lim_{n\to\infty}\Psi_{n,\ell}^{\btheta}(\btheta_\ell)(t).
\end{align*}
Then $\Psi_*$ is a limit point of $\{\Psi_n:n\ge 0\}$ under $\Dis^{[0,T]}$.
Since the Lipschitz continuity is preserved under the pointwise convergence, we have $\Psi_*\in\mathbf{\Psi}$.
\end{proof}

% \subsection{Proof of lemmas}
Next we prove lemmas for Theorem~\ref{theorem:conpsi}.
Analogous to the notation of $\bu_\ell$, for the convenience of presenting continuity of $\Psi^\bw_{\ell}$, we introduce notations $\bbfu_\ell$ and $\bbfu'_\ell$ by letting
\[
\bbfu_\ell=
\begin{cases}
\bbw_1,\\
(\bbw_1,\btheta_2),\\
(\bbtheta_{\ell-1},\btheta_{\ell}),\\
\bbtheta_L
\end{cases}
\quad 
\bbfu_\ell'=
\begin{cases}
\bbw_1,&\ell=1,\\
(\bw_1,\bbtheta_2),&\ell=2,\\
(\btheta_{\ell-1},\bbtheta_{\ell}),&\ell\in[3:L],\\
\bbtheta_L&\ell=L+1.
\end{cases}
\]
We also abbreviate the notations for the individual trajectories as:
$$
 \Psi_\ell^{\bw}(\bu_\ell)(t)=w_\ell^t(\bu_\ell),\quad   \Psi_\ell^{\bw}(\bbfu_\ell)(t)=w_\ell^t(\bbfu_\ell),\quad   \Psi_\ell^{\bw}(\bbfu_\ell')(t)=w_\ell^t(\bbfu_\ell'),\quad \ell\in[L+1],
$$
and 
$$  \Psi_\ell^{\btheta}(\btheta_\ell)(t)=\btheta_\ell^t,\quad\quad\quad ~ \Psi_\ell^{\btheta}(\bbtheta_\ell)(t)=\bbtheta_\ell^t,\quad \quad\quad~ \ell\in[2:L]. $$

\begin{proof}[Proof of Lemma~\ref{const Lstar}]
We first investigate the set $F(\bfPsi\cap \bfPsi_{\beta})$ for a general $\beta$.
We follow similar steps as the proof of Lemma~\ref{lmm:contraction} by inductively showing upper bound for the differences between $\lineDel_{\ell}$ and $\lineGrad^{\bw}_{\ell}$ from $\ell=L+1$ to $1$ using backward equations, and then for the differences between $\lineGrad^{\btheta}_{\ell}$ from $\ell=1$ to $L$ using forward equations.
Specifically, we will prove that (cf. Definition~\ref{def:local-lip}) there exists a constant $C$ independent of $\beta$ such that for any $\Psi\in\bfPsi\cap \bfPsi_{\beta}$,
\begin{align*}
&\|\lineDel_{\ell}(\btheta_\ell;\Psi,t) -\lineDel_{\ell}(\bbtheta_\ell; \Psi,t)\|_\infty
\le C e^{\beta t}(1+\|\btheta_\ell\|_\infty)\| \btheta_\ell - \bbtheta_\ell\|_{\infty} ,& \ell\in[L],\\
&\|\lineGrad^{\bw}_{\ell}(\bu_\ell;\Psi,t) - \lineGrad^{\bw}_{\ell}(\bbfu_{\ell};\Psi,t) \|_\infty
\le C e^{\beta t}(1+\|\bu_\ell\|_\infty)\|\bu_{\ell}-\bbfu_{\ell}\|_\infty,& \ell\in[L+1],\\
&\|\lineGrad^{\bw}_{\ell}(\bu_\ell;\Psi,t) - \lineGrad^{\bw}_{\ell}(\bbfu'_{\ell};\Psi,t) \|_\infty
\le C e^{\beta t}(1+\|\bu_\ell\|_\infty)\|\bu_{\ell}-\bbfu'_{\ell}\|_\infty,& \ell\in[L+1],\\
&\| \lineGrad^{\btheta}_{\ell}( \btheta_\ell; \Psi,t)-\lineGrad^{\btheta}_{\ell}( \bbtheta_\ell; \Psi,t)\|_\infty
\le C e^{\beta t}(1+\|\btheta_\ell\|_\infty)\|\btheta_\ell-\bbtheta_\ell\|_\infty,& \ell\in[2:L].
\end{align*}

We first consider the backward steps. 
Again we focus on the difference between $\lineDel_{\ell}$. 
Then the upper bound for the difference between $\lineGrad^{\bw}_{\ell}$ follows immediately. 
In particular, for the top layer $\ell=L+1$, since $h$ is Lipschitz continuous and $\Psi\in\bfPsi_\beta$, applying \eqref{local lip-tl} and the formula of $\lineGrad^{\bw}_{L+1}$ in \eqref{linegrad wL+1} yields that
$$ \left|\lineGrad^{\bw}_{L+1}(\btheta_{L};\Psi,t) - \lineGrad^{\bw}_{L+1}(\bbtheta_{L};\Psi,t) \right|\leq  Ce^{\beta t}
(\|\btheta_L\|_{\infty}+1)\| \btheta_L - \bbtheta_L\|_{\infty}. $$
Other layers can be analogously obtained.
At layer $\ell=L$, 
recall that 
$$\lineDel_{L}(\btheta_{L}; \Psi,t) 
=    w_{L+1}^t(\btheta_L)~ \lineDel_{L+1}(\Psi,t) \cdot    \th'\left(\btheta_{L}^t\right).$$
Since $\Psi\in\bfPsi_\beta$, we have an upper bound for $|w_{L+1}^t(\btheta_L) - w_{L+1}^t(\bbtheta_L)  |$ from \eqref{local lip-wL+1}.
Applying the Lipschitz continuity of $h'$ and \eqref{local lip-tl} yields that
$$ \left\|\th'\left(\btheta_{L}^t\right) -\th'\left(\bbtheta_{L}^t\right)\right\|_{\infty}
\leq C e^{\beta t}( \|\btheta_{L}\|_{\infty}+1  ) \|\btheta_{L} - \bbtheta_{L}\|_{\infty}.  $$
Since $h'$ is bounded, we apply the upper bound of $w_{L+1}^t$ and $\lineDel_{L+1}$ in Lemma~\ref{lmm:Psi-property} and obtain that
\begin{eqnarray}\label{qq+1}
\left\|\lineDel_{L}(\btheta_L;\Psi,t) - \lineDel_{L}(\bbtheta_L;\Psi,t) \right\|_{\infty}\leq C e^{\beta t}( \|\btheta_{L}\|_{\infty}+1  ) \|\btheta_{L} - \bbtheta_{L}\|_{\infty}.
\end{eqnarray}
For each layer $\ell=L-1,\dots,1$, we have
\begin{eqnarray}
 &&\int \underbrace{\left| w_{\ell+1}^t(\btheta_{\ell},\btheta_{\ell+1}) - w_{\ell+1}^t(\bbtheta_{\ell},\btheta_{\ell+1})\right|}_{\leq  C e^{\beta t}(\|\btheta_{\ell} \|_{\infty}+\|\btheta_{\ell+1} \|_{\infty}+1 )\|\btheta_{\ell} -\bbtheta_{\ell}  \|_{\infty} } \left\| \lineDel_{\ell+1}(\btheta_{\ell+1}; \Psi,t) \right\|_{\infty}dp_{\ell+1}(\btheta_{\ell+1}) \notag\\
 &\leq& Ce^{\beta t}(\| \btheta_{\ell}\|_{\infty} +1 ) \|\btheta_{\ell} -\bbtheta_{\ell}  \|_{\infty},\label{cqq1}
\end{eqnarray}
where in the last step we used the sub-gaussianness of $p_{\ell+1}$, Corollary~\ref{cor:subg-q}, and upper bound of $\lineDel_{\ell+1}$ in Lemma~\ref{lmm:Psi-property}.
Then, by the upper bound in \eqref{Del bound}, boundedness and Lipschitz continuity of $h'$, we obtain that
\begin{eqnarray}\label{dellater}
\left\|\lineDel_{\ell}(\btheta_\ell;\Psi,t) -\lineDel_{\ell}(\bbtheta_\ell;\Psi,t) \right\|_{\infty}\leq C e^{\beta t} ( \|\btheta_{\ell}\|_{\infty}+1  ) \|\btheta_{\ell} - \bbtheta_{\ell}\|_{\infty}. 
\end{eqnarray}

Now we turn to the forward steps.
For $\ell=1$, the boundedness of $\bX$ yields that
$$   \left\|\lineGrad^{\bw}_{1}(\bbw_1;\Psi,t) - \lineGrad^{\bw}_{1}(\bbw_1;\Psi,t) \right\|_{\infty}\leq 
C e^{\beta t} (\|\bw_{1}\|_{\infty}+1)\| \bw_1 - \bbw_1\|_{\infty}.  $$
For $\ell\in[2:L]$, we consider the two terms in \eqref{grad12} separately. 
For the first term, we have 
\begin{eqnarray}\label{cqq13}
 &&\int \underbrace{\left\| \th'(\btheta_{\ell-1}) \cdot \lineGrad^{\btheta}_{\ell-1}(\btheta_{\ell-1}; \Psi,t) \right\|}_{\leq C (\|\btheta_{\ell-1}\|_{\infty}+1) }
 \underbrace{\left| w_{\ell}^t(\btheta_{\ell-1}, \btheta_{\ell}) -w_{\ell}^t(\btheta_{\ell-1}, \bbtheta_{\ell}) \right|}_{\leq C e^{\beta t}(\|\btheta_{\ell-1}\|_{\infty}+\|\btheta_{\ell}\|_{\infty}+1)\|\btheta_{\ell}-\bbtheta_{\ell} \|_{\infty}  } dp_{\ell-1}(\btheta_{\ell-1})\notag\\
 &\leq& C e^{\beta t}(\|\btheta_{\ell} \|_{\infty}+1)  \|\btheta_{\ell}-\bbtheta_{\ell} \|_{\infty},\notag
\end{eqnarray}
by the sub-gaussianness of $p_{\ell-1}$ and Corollary~\ref{cor:subg-q}.
Similarly, for the second term, applying the boundedness of $h$ yields that
\begin{eqnarray}
&& \int \left\|\lineGrad^{\bw}_{\ell-1}(\btheta_{\ell-1}, \btheta_{\ell};\Psi,t)  -  \lineGrad^{\bw}_{\ell-1}(\btheta_{\ell-1}, \bbtheta_{\ell};\Psi,t)\right\|_{\infty} \|\th(\btheta_{\ell-1}^t)\|_{\infty} dp_{\ell-1}(\btheta_{\ell-1}) \notag\\
&\leq& \int C e^{\beta t} (\| \btheta_{\ell-1}\|_{\infty}+\|\btheta_{\ell} \|_{\infty}+1) \|\btheta_{\ell}-\bbtheta_{\ell} \|_{\infty} dp_{\ell-1}(\btheta_{\ell-1})\notag\\
&\leq& C e^{\beta t}(\|\btheta_{\ell}\|_{\infty}+1)\| \btheta_{\ell} -\bbtheta_{\ell}\|_{\infty}.   \notag
\end{eqnarray}
Therefore, we obtain that
$$   \left\|\lineGrad^{\btheta}_{\ell}(\btheta_{\ell};\Psi,t) - \lineGrad^{\btheta}_{\ell}(\bbtheta_{\ell};\Psi,t) \right\|_\infty \leq C e^{\beta t}(\|\btheta_{\ell}\|_{\infty}+1)\| \btheta_{\ell} -\bbtheta_{\ell}\|_{\infty}. $$

Finally, let $\beta_*=C$. 
It remains to verify that $F(\Psi)\in \bfPsi_{\beta_*}$ for any $\Psi\in \bfPsi\cap \bfPsi_{\beta_*}$, that is, to verify the conditions in Definition~\ref{def:local-lip}. 
For $F(\Psi)_1^{\bw}$, we have 
\begin{eqnarray}
 &&\left\|F(\Psi)^\bw_{1}(\bw_1)(t)-F(\Psi)^\bw_{1}(\bbw_1)(t) \right\|_{\infty}\notag\\
 &\leq&  \left\| \bw_1 -\bbw_1 \right\|_{\infty}+   \int_{0}^t\left\| \lineGrad^{\bw}_{1}(\bw_1;\Psi,s)- \lineGrad^{\bw}_{1}(\bbw_1;\Psi,s)\right\|_{\infty}ds \notag\\
 &\leq&    \left\| \bw_1 -\bbw_1 \right\|_{\infty} +\int_{0}^t   C (\|\bw_1 \|_{\infty}+1) e^{\beta_* s} \left\| \bw_1 -\bbw_1 \right\|_{\infty}  ds\notag\\
 &\leq& e^{\beta_* s}(\|\bw_1  \|_{\infty}+1)  \|\bw_1 -\bbw_1 \|_{\infty}.\label{ftheo4}
\end{eqnarray}
The verification of other cases are entirely analogous and is omitted. 
\end{proof}

\begin{lemma}
\label{psi closed}
$\bfPsi\cap\bfPsi_{\beta}$ is a closed set.
\end{lemma}
\begin{proof}
Given a convergent sequence $\{\Psi_n:n\ge 0\}\subseteq \bfPsi\cap\bfPsi_{\beta}$, it follows  from Lemma~\ref{lmm:psi-complete} that the limit point $\Psi_*\in \bfPsi$.
Since Lipschitz property is preserved under pointwise convergence, we also have $\Psi_*\in \bfPsi_{\beta}$.
\end{proof}
\section{Proofs of Theorems \ref{theorm:app} -- \ref{theo:con}}
\subsection{Proof of Theorem \ref{theorm:app}}
In the proof, we fix $\Psi_*$ and the initialization $\{\bbw_{1,i}\}_{i\in[m]}$, $\{\bbtheta_{\ell,i}\}_{\ell\in[2:L],i\in[m]}$ of the ideal process.
Similar to the notation $\bu_\ell$ in the proof of Theorem~\ref{theorm:flow p0}, we introduce the notations $\bbfu_{\ell,i,j}$ that stands for $\bbw_{1,j}$, $(\bbw_{1,i},\bbtheta_{2,j})$, $(\bbtheta_{\ell-1,i},\bbtheta_{\ell,j})$, $\bbtheta_{L,i}$ for $\ell=1$, $\ell=2$, $3\le\ell\le L$, $\ell=L+1$, respectively. 
% $\bu_\ell$ stands for $\bw_1$, $(\bw_1,\btheta_2)$, $(\btheta_{\ell-1},\btheta_{\ell})$, $\btheta_L$ for $\ell=1$, $\ell=2$, $3\le\ell\le L$, $\ell=L+1$, respectively. 
We also abbreviate the gradients of the ideal process as 
\[
\lineDel_{\ell,i}^{t}:=\lineDel_{\ell}(\linetheta_{\ell,i}, \Psi_*, t),
\qquad \lineGrad^{t}_{\ell,i,j}:= \lineGrad_{\ell}^{\bw}(\bbfu_{\ell,i,j};  \Psi_*, t).
\]
We use a common notation $\linew_{\ell,i,j}^t$ to the weights at layer $\ell$; for $\ell=1$ let $\linew_{1,i,j}^t=\bbw_{1,j}^t$.
% ; for $\ell=L+1$ let $\linew_{L+1,i,j}^t=w_{L+1,i}^t$ 
To compare the discrete and continuous trajectories on the same time scale, we normalize discrete gradients by 
\[
\NhDel^k_{\ell,i} := m_{\ell}~\hDel^k_{\ell,i},
\qquad \NhGrad^{k}_{\ell,i,j}  := [m_{\ell-1}m_{\ell}]~\hGrad^k_{\ell,i,j}.
\]
When $m$ is finite, the forward and backward propagation for the ideal process is no long exact. Nevertheless, for sufficiently large $m$, those propagations relations approximately holds by the following events that happen with high probability:
% Throughout the proof, we condition on the following high probability events on the ideal process:
\begin{align}
& \left\| \frac{1}{m}\sum_{i=1}^m \left[\th\left(\bbtheta_{\ell-1,i}^{k\eta}\right) \linew_{\ell,i,j}^{k\eta}\right] - \linetheta_{\ell,j}^{k\eta}\right\|_{\infty} 
\leq C(\|\linetheta_{\ell,j}\|_\infty+1)\ep_1,\qquad \ell\in[2:L], ~j\in[m], \label{event forward}\\
& \left\| \frac{1}{m}\sum_{j=1}^m \left[\linew_{\ell+1,i,j}^{k\eta} \lineDel_{\ell+1,j}^{k\eta}  \right]\cdot \th'\left(\linetheta_{\ell,i}^{k\eta}\right)  - \lineDel_{\ell,i}^{k\eta}\right\|_{\infty} 
\leq \ep_1,\quad \ell\in[L-1],~i\in[m],\label{event backward}\\
& \max_{i}\left\|\bbw_{1,i} \right\|_{\infty} \leq C\sqrt{\log\frac{m}{\delta}},
% \sigma_2 \sqrt{\log(4em L \delta^{-1})}, 
\qquad \max_{i}\left\|\bbtheta_{\ell,i} \right\|_{\infty} \leq C\sqrt{\log\frac{m}{\delta}},
% \sigma_2 \sqrt{\log(4em L \delta^{-1})} , \quad \ell\in[2:L],~ i\in[m],
\label{event max}\\
& \frac{1}{m}\sum_{i=1}^m \left\| \bbw_{1,i}\right\|_{\infty}^j \leq C,  
\quad  
\frac{1}{m}\sum_{i=1}^m \| \linetheta_{\ell,i}\|_{\infty}^j\leq C,  
\quad j\in[2],~\ell\in[2:L],\label{event moment}
\end{align}
for a constant $C$.
In the proofs of this section, we condition on those events. 
\begin{lemma}
\label{lmm:high-prob}
% If $m\ge \tilde{\Omega}(\frac{1}{\ep_1^2}\log\frac{1}{\delta})$, then 
The events \eqref{event forward} -- \eqref{event moment} happen with probability $1-\delta$.
\end{lemma}

The proof consists of the deviation of the actual discrete trajectory from the ideal trajectory over the iteration $k\in[0:K]$. 
We will upper bound the deviation by induction on $k$.
For $k=0$, we have the deviation of weights $\|\linew_{\ell,i,j}^{0} - \hw_{\ell,i,j}^{0}\|_\infty$ from the initial conditions in Definition~\ref{DNN condition}. 
The induction proceeds as follows. 
In Lemma~\ref{lmm:induction-k}, we first upper bound the deviation of features using the forward propagation, and then upper bound the deviation of gradients using the backward propagation.
% $\| \linetheta_{\ell,i}^{k\eta} - \hbtheta_{\ell,i}^{k} \|_{\infty}$ using the forward propagation; then we upper bound $|\lineGrad^{k\eta}_{\ell,i,j}-\NhGrad^k_{\ell, i,j}|$ using the backward propagation. 
Note that
\begin{equation}
\label{induction w^k}
\left\| \linew_{\ell,i,j}^{(k+1)\eta} - \hw_{\ell,i,j}^{k+1}   \right\|_\infty
\le \|\linew_{\ell,i,j}^{k\eta} - \hw_{\ell,i,j}^{k} \|_\infty 
    + \int_{k\eta}^{(k+1)\eta} \left\| \lineGrad^{s}_{\ell,i,j} - \NhGrad^k_{\ell, i,j}\right \|_\infty ds.
\end{equation}
Combining with the Lipschitz continuity of $\lineGrad^{t}_{\ell,i,j}$ in Lemma~\ref{lmm:gradw-lip-t}, we complete the inductive step.
% we obtain the upper bound for $|\linew_{\ell,i,j}^{(k+1)\eta} - \hw_{\ell,i,j}^{k+1}|$. 

\begin{lemma}
\label{lmm:induction-k}
Given $k\in[0:K]$ and $\ep<1$. 
Suppose
\begin{equation}
\label{eq:induction weights}
\left\| \linew_{\ell,i,j}^{k\eta} - \hw_{\ell,i,j}^{k}   \right\|_\infty
\le (\| \bbfu_{\ell,i,j}\|_\infty+1) \ep,\qquad \forall~\ell\in[L+1],~i\in[m_{\ell-1}],~j\in [m_\ell].
\end{equation}
Then there exists a constant $C$ such that
\begin{align}
& \left\| \linetheta_{L+1,1}^{k\eta} - \hbtheta_{L+1,1}^{k} \right\|_{\infty} \!\!\!\leq C\left(\ep+\ep_1\right), \label{eq:induction forward last}\\
& \left\| \linetheta_{\ell,i}^{k\eta} - \hbtheta_{\ell,i}^{k} \right\|_{\infty} \!\!\!\leq  C\left(  \left\| \linetheta_{\ell,i} \right\|_{\infty}+1 \right)\left(\ep+\ep_1\right),~ \forall~\ell \in[L],~ i\in[m_\ell],\label{eq:induction forward}\\
& \left\|\lineGrad_{\ell,i,j}^{k\eta} -  \NhGrad_{\ell,i,j}^k  \right\|_\infty 
\!\!\!\leq  C\left(\| \bbfu_{\ell,i,j}\|_{\infty}+1\right)\left(\ep+\ep_1\right),~ \forall~\ell\in[L+1],~i\in[m_{\ell-1}],~j\in [m_\ell]\label{eq:induction backward}.
\end{align}
\end{lemma}

\begin{lemma}
\label{lmm:gradw-lip-t}
There exists a constant $C$ such that, for all $\ell\in[L+1]$, $t_1,t_2\in[0,T]$, and $\bu_\ell$,
\[
\left\|\lineGrad_{\ell,i,j}^{t_1}-\lineGrad_{\ell,i,j}^{t_2}\right\|_\infty \le C (\|\bbfu_{\ell,i,j}\|_\infty+1)|t_1-t_2|.
\]
% \[
% \left\|\lineGrad_\ell^{\bw}(\bu_\ell;\Psi_*,t_1)-\lineGrad_\ell^{\bw}(\bu_\ell;\Psi_*,t_2)\right\|_\infty \le C (1+\|\bu_\ell\|_\infty)|t_1-t_2|.
% \]
\end{lemma}

\begin{proof}[Proof of Theorem \ref{theorm:app}]
By Lemma~\ref{lmm:high-prob}, the events in \eqref{event forward} -- \eqref{event moment} happen with probability $1-\delta$.
Conditioned on those events, we prove by induction on $k\in[0:K]$ that
\begin{equation}
\label{eq:w-diff}
\left\|\linew_{\ell,i,j}^{k\eta} - \hw_{\ell,i,j}^k  \right\|_\infty 
\leq  \left(\|\bbfu_{\ell,i,j} \|_{\infty}+1\right) e^{  C  k \eta} \ep_1,\quad \forall~\ell\in[L+1],~~i\in[m_{\ell-1}],~j\in [m_\ell],
\end{equation}
for some constant $C$ to be specified. 
The base case $k=0$ follows from Definition~\ref{DNN condition}. 
Suppose that \eqref{eq:w-diff} holds for $k\in[0:K-1]$. 
By Lemmas~\ref{lmm:induction-k} and \ref{lmm:gradw-lip-t}, for $s\in[k\eta, (k+1)\eta]$,  
\[
\left\| \lineGrad^{s}_{\ell,i,j} - \NhGrad^k_{\ell, i,j}\right \|_\infty
\le C'\left(\| \bbfu_{\ell,i,j}\|_{\infty}+1\right)\left(e^{Ck\eta}\ep_1+\ep_1+s-k\eta\right).
\]
Applying \eqref{induction w^k} yields that
\begin{align*}
\left\| \linew_{\ell,i,j}^{(k+1)\eta} - \hw_{\ell,i,j}^{k+1}   \right\|_\infty
&\le \left\|\linew_{\ell,i,j}^{k\eta} - \hw_{\ell,i,j}^{k} \right\|_\infty 
    + \int_{k\eta}^{(k+1)\eta} \left\| \lineGrad^{s}_{\ell,i,j} - \NhGrad^k_{\ell, i,j}\right \|_\infty ds\\
&\le \left(\|\bbfu_{\ell,i,j} \|_{\infty}+1\right) 
\left(e^{  C  k \eta} \ep_1 + 2C'e^{Ck\eta}\ep_1 \eta + C'\frac{\eta^2}{2} \right)\\
&\le \left(\|\bbfu_{\ell,i,j} \|_{\infty}+1\right) e^{  C  k \eta} \ep_1(1+C'' \eta),
\end{align*}
for a constant $C''$. 
By letting $C=C''$, we arrive at \eqref{eq:w-diff} for $k+1$ using $1+C\eta \le e^{C\eta}$.
Note that $k\eta\le T$ for $k\in[0:K]$, $\ep_1\le\tO(\frac{1}{\sqrt{m}})$, and $\|\bbfu_{\ell,i,j} \|_{\infty}\le \cO(\sqrt{\log\frac{m}{\delta}})$ by \eqref{event max}.
The conclusion follows from Lemma~\ref{lmm:induction-k} and the Lipschitz continuity of $\phi$.
\end{proof}

\subsection{Proof of Theorem \ref{ini lemma}}
We first introduce the initialization of the continuous DNN: 
\[
p_1 := \mathcal{N}\left(\mathbf{0}^d,d\sigma_1^2~ \mathbf{I}^{d}\right),
\qquad
p_{\ell} := \mathcal{N}\left(\mathbf{0}^N, \sigma_1^2~ \bK_{\ell-1}\right), \quad \ell\in[2:L].
\]
The connecting weights between consecutive layers are given by
\[
w_{2}(\bw_1, \btheta_{2}) := \btheta_{2}^\top \bK_1^{-1} \th\big(\btheta_{1}(\bw_1) \big),
\qquad
w_{\ell+1}(\btheta_{\ell}, \btheta_{\ell+1}) :=    \btheta_{\ell+1}^\top \bK_{\ell}^{-1} \th\left(\btheta_{\ell} \right), \quad \ell\in[2:L-1].
\]
The weights at the output layer are initialized as a constant $C_3$ given in Algorithm~\ref{algo:init}.
Then the forward propagation constraints \eqref{forward l2} and \eqref{forward l} are satisfied by the definitions of $\bK_\ell$.
The weights also satisfy the conditions in  Assumption \ref{ass:4} since $\|\bK_\ell^{-1} \|_2\le \blambda^{-1} $ and $h$ is bounded and Lipschitz continuous.

Next we construct the initialization for ideal discrete DNN $(\linebw, \linetheta)$ that are mutually independent with $\linebw_{1,i} \iiddistr p_1$, $\linetheta_{\ell,i} \iiddistr p_\ell$. The  closeness to the actual discrete DNN will be shown in Lemma~\ref{lmm:hK-K}.
% $\linebw_{1,i}=\hbw_{1,i}$ for $i\in[m]$, and with probability $1-\delta$,
% \begin{align}
% \|\linetheta_{\ell,i}-\hbtheta_{\ell,i}\|_2\le .. \left\|\linetheta_{\ell,i}\right\|_2,\qquad \forall~\ell\in[2:L],i\in[m].\label{eq:diff-ideal-theta}  
% \end{align}
% We first introduce the ideal discrete DNN $(\linebw, \linetheta)$ and show the distributions at the initialization. 
Let $\linebw_{1,i} := \hbw_{i,1}$ for  $i\in[m]$.
For $\ell\in[L-1]$, define the empirical Gram matrix as
    $$  \hbK_{\ell} = \frac{1}{m}\sum_{i=1}^m \th\left( \hbtheta_{\ell,i}\right)\th^\top\left(\hbtheta_{\ell,i}\right),$$
where $\hbtheta_{1,i}=\btheta_1(\hbw_{1,i})$.
Let $\linetheta_{\ell+1,j} :=  \bK_{\ell}^{1/2} \hbK_{\ell}^{-1/2}  \hbtheta_{\ell+1,j}$ for all $j\in[m]$ when $\hbK_{\ell}$ is invertible, and otherwise let $\linetheta_{\ell+1,j}\iiddistr p_{\ell+1}$.
Here $\linetheta_{\ell+1,j}$ are determined by the outputs of previous layer $\hbtheta_{\ell,i}$ and the connecting weights $\hw_{\ell+1,i,j}$.
Hence, they are independent of $\linebw_{1,i}$ and $\linetheta_{2,i},\dots,\linetheta_{\ell,i}$ for $i\in[m]$ given $\{\hbtheta_{\ell,i}\}_{i\in[m]}$.
% ..dependency relations ..
Since $\hw_{\ell+1,i,j}$ are independent Gaussian, $\hbtheta_{\ell+1,j}$ and thus $\linetheta_{\ell+1,j}$ are conditionally independent Gaussian given $\{\hbtheta_{\ell,i}\}_{i\in[m]}$.
Furthermore, the conditional distribution of $\linetheta_{\ell+1,j}$ given  $\{\hbtheta_{\ell,i}\}_{i\in[m]}$ is $\mathcal{N}\left(\mathbf{0}^N, \sigma_1^2\bK_{\ell}\right) = p_{\ell+1}$. %for all $\{\hbtheta_{\ell,i}\}_{i\in[m_\ell]}$. 
Therefore, marginally $\linetheta_{\ell+1,j}\iiddistr p_{\ell+1}$ and they are independent of $\linebw_{1,i}$ and $\linetheta_{2,i},\dots,\linetheta_{\ell,i}$ for $i\in[m]$.
% We prove \eqref{eq:diff-ideal-theta} in Lemma~\ref{lmm:hK-K}.
% For \eqref{eq:diff-ideal-theta}, we use the following upper bound 
% \begin{align*}
% \left\|\hbtheta_{\ell+1,i} - \linetheta_{\ell+1,i}  \right\|_2
% & = \left\|\left( \hbK^{1/2}_{\ell}\bK^{-1/2}_{\ell}  - \mathbf{I}^N \right)     \linetheta_{\ell+1,i} \right\|_2\\
% & \le \left\|\hbK^{1/2}_{\ell}  - \bK^{1/2}_{\ell}\right\|_2
% \left\|\bK^{-1/2}_{\ell}\right\|_2
% \left\|\linetheta_{\ell+1,i}\right\|_2,
% \end{align*}
% and the concentration of $\hbK_{\ell}$ in Lemma~\ref{lmm:hK-K}.
\begin{lemma}
\label{lmm:hK-K}
Let $\ep_2:=\cO(\ep_1)$ such that $\ep_2< \blambda^{-1}$. 
If $m\ge \tilde\Omega(\ep_2^{-2})$, then, with probability $1-\delta$, for all $\ell\in[L-1]$,
\begin{align*}
& \left\|\hbK_{\ell}  - \bK_{\ell}\right\|_2
\le \ep_2,
\qquad
\norm{\linetheta_{\ell+1,i}-\hbtheta_{\ell+1,i} }_2\le \ep_2\left\|\linetheta_{\ell+1,i}\right\|_2,\\
& \norm{\bbtheta_{\ell+1,i} }_2
\le B_3:= C\sqrt{\log(m/\delta)}.
\end{align*}
% where $\tO(\cdot)$ hides poly-logarithmic factors of $m$ and $\frac{1}{\delta}$.
\end{lemma}

Finally we show that the initial connecting weights are also close to the actual discrete DNN as specified by the upper bound of $| \linew_{\ell+1,i, j} -\hw_{\ell+1, i,j}|$ in Definition~\ref{DNN condition}. Under Lemma \ref{lmm:hK-K}, $\hbK_{\ell}$ is invertible. Then we have
 the following formula for $\hw_{\ell+1, i, j}$ of Algorithm~\ref{algo:init} (see  Lemma~\ref{lmm:hat-w}):
\[
\hw_{\ell+1, i, j} = \hbtheta_{\ell+1,j}^\top\hbK^{-1}_{\ell}\th(\hbtheta_{\ell,i}), \quad \ell\in[L-1], ~i,j\in[m].
\]
By the triangle inequality,
\begin{align*}
&\phantomeq\left| \linew_{\ell+1,i, j} -\hw_{\ell+1, i,j}\right|\\
&=\norm{ \linetheta_{\ell+1,j}^\top\bK^{-1}_{\ell}\th(\linetheta_{\ell,i}) 
- \hbtheta_{\ell+1,j}^\top\hbK^{-1}_{\ell}\th(\hbtheta_{\ell,i})}_2\\
&\le \Norm{\bbtheta_{\ell+1,j}}_2 \norm{\bK_\ell^{-1}}_2 \norm{\th(\linetheta_{\ell,i})- \th(\hbtheta_{\ell,i})}_2
+ \Norm{\bbtheta_{\ell+1,j}}_2 \norm{\bK_\ell^{-1} - \hbK^{-1}_{\ell}}_2 \norm{\th(\hbtheta_{\ell,i})}_2\\
&\qquad\qquad + \Norm{\bbtheta_{\ell+1,j}- \hbtheta_{\ell+1,j} }_2 \Norm{\hbK^{-1}_{\ell}}_2 \norm{\th(\hbtheta_{\ell,i})}_2.
% &\le 2\sqrt{N}L_1\blambda^{-1}\Norm{\bbtheta_{\ell+1,j}}_2
% + \Norm{\bbtheta_{\ell+1,j}}_2 .. \sqrt{N}L_1 
% + .. .. \sqrt{N}L_1 
\end{align*}
Under the same event in Lemma~\ref{lmm:hK-K}, we upper bound three terms separately. 
By the Lipschitz continuity of $h$, the first term is at most $CB_3\blambda^{-1}\ep_2\Norm{\bbtheta_{\ell+1,j}}_2$;
for the second term, since $h$ is bounded and
\[
\Norm{\bK_\ell^{-1} - \hbK^{-1}_{\ell}}_2
\le \Norm{\bK_\ell^{-1}}_2 
\Norm{\bK_\ell - \hbK_{\ell}}_2 \Norm{\hbK^{-1}_{\ell}}_2
\le 2\blambda^{-2}\ep_2,
\]
we have an upper bound $C\blambda^{-2}\ep_2\Norm{\bbtheta_{\ell+1,j}}_2$;
the third term is at most $C\lambda^{-1}\ep_2\Norm{\bbtheta_{\ell+1,j}}_2$.

\subsection{Proof of Theorem \ref{theo:con}}
The proof is based on the lemma below.
\begin{lemma}\label{property of convex}
	$\frac{x^{a}}{y^b}$ is convex on $(x,y)\in [0, +\infty)\otimes(0, +\infty)  $ when $a-1\geq b  \geq 0$.
\end{lemma}
\begin{proof}
One can   verify that the Hessian matrix of $\frac{x^{a}}{y^b}$ is positive semi-definite when $a-1\geq b  \geq 0$.
\end{proof}
\begin{proof}[Proof of Theorem \ref{theo:con}]
We can observe that the Problem \eqref{problem2} only has linear constraints. We prove that the objective function is convex. It is sufficient to show that $ {R}_{\ell}^{\ttw}( \cdot, \cdot)$ is convex for all $\ell\in[2:L+1]$. For $\ell\in[2:L]$,  for  any $\left( \ttw_{\ell}^{(1)}, \ddp_{\ell}^{(1)} \right)$ and   $\left( \ttw_{\ell}^{(2)}, \ddp_{\ell}^{(2)} \right)$, we define $$( \ttw_{\ell}^{(3)}, \ddp_{\ell}^{(3)} )= \left(  \alpha \ttw_{\ell}^{(1)} +(1-\alpha)\ttw_{\ell}^{(2)},  \alpha \ddp_{\ell}^{(1)} +(1-\alpha) \ddp_{\ell}^{(2)}\right),$$ where $0\leq \alpha\leq 1$. Because $\ell_1$-norm is convex, we have
\begin{eqnarray}\label{1333}
  \int \left|\ttw_{\ell}^{(3)}(\btheta_{\ell-1}, \btheta_{\ell})  \right|d \btheta_{\ell}\leq  \alpha  \int \left|\ttw_{\ell}^{(1)}(\btheta_{\ell-1}, \btheta_{\ell})  \right| d \btheta_{\ell}+ (1-\alpha)  \int |\ttw_{\ell}^{(2)}(\btheta_{\ell-1}, \btheta_{\ell})  |d \btheta_{\ell}.
\end{eqnarray}
Let us introduce 
$$ \tu( \btheta_{\ell-1} ) := \alpha\int \left|\ttw_{\ell}^{(1)}( \btheta_{\ell-1},\btheta_{\ell} )\right|  d\btheta_{\ell} +(1-\alpha)\int \left|\ttw_{\ell}^{(2)}( \btheta_{\ell-1},\btheta_{\ell} )\right|  d\btheta_{\ell}.  $$
Because $r\geq1$,  from Lemma \ref{property of convex}, we have
\begin{eqnarray}\label{13332}
   \frac{\left(\tu( \btheta_{\ell-1} )\right)^r}{\left(\ddp^{(3)}(\btheta_{\ell-1})\right)^{r-1}}\leq \alpha    \frac{\left(\int \left|\ttw_{\ell}^{(1)}( \btheta_{\ell-1},\btheta_{\ell} )\right|  d\btheta_{\ell} \right)^r}{\left(\ddp^{(1)}(\btheta_{\ell-1})\right)^{r-1}} + (1-\alpha)\frac{\left(\int \left|\ttw_{\ell}^{(2)}( \btheta_{\ell-1},\btheta_{\ell} )\right|  d\btheta_{\ell} \right)^r}{\left(\ddp^{(2)}(\btheta_{\ell-1})\right)^{r-1}}.  
\end{eqnarray}
Plugging \eqref{1333} into \eqref{13332}, using that $|x |^r$  is monotonically increasing when $x\geq0$, we have
$$    \frac{\left(\int \left|\ttw_{\ell}^{(3)}( \btheta_{\ell-1},\btheta_{\ell} )\right|  d\btheta_{\ell} \right)^r}{\left(\ddp^{(3)}(\btheta_{\ell-1})\right)^{r-1}}\leq \alpha    \frac{\left(\int \left|\ttw_{\ell}^{(1)}( \btheta_{\ell-1},\btheta_{\ell} )\right|  d\btheta_{\ell} \right)^r}{\left(\ddp^{(1)}(\btheta_{\ell-1})\right)^{r-1}} + (1-\alpha)\frac{\left(\int \left|\ttw_{\ell}^{(2)}( \btheta_{\ell-1},\btheta_{\ell} )\right|  d\btheta_{\ell} \right)^r}{\left(\ddp^{(2)}(\btheta_{\ell-1})\right)^{r-1}}.   $$
Integrating the above inequality on $\btheta_{\ell}$, we have that   $ {R}_{\ell}^{\ttw}$ is convex.   In the same way, we can obtain  the convexity of   $ {R}_{L+1}^{\ttw}$. We achieve Theorem \ref{theo:con}.
\end{proof}

\subsection{Proofs of Lemmas}
\begin{proof}[Proof of Lemma~\ref{lmm:high-prob}]
We prove each of the four events happens with probability $1-\frac{\delta}{4}$ by standard concentration inequalities thanks to mutual independence of the ideal process. 
For \eqref{event forward} with a given $k,\ell,j,n$, consider random vectors 
\[
\xi_i := \frac{h\left( \linetheta_{\ell-1,j}^{k\eta}(n) \right) \linew_{\ell,i,j}^{k\eta}}{\left\|\linetheta_{\ell,j}\right\|_{\infty}+1},
\]
which are bounded by a constant $C'$ due to the upper bound of $\linew_\ell$ in Lemma~\ref{lmm:Psi-property}.
Conditioned on $\bbtheta_{\ell,j}$, those $\xi_i$ are independent and $\EE [ \xi_i | \bbtheta_{\ell,j}]= \frac{\linetheta_{\ell,j}^{k\eta}(n)}{\|\linetheta_{\ell,j}\|_{\infty}+1}$. 
By Hoeffding's inequality and the union bound, we have 
\[
\left| \frac{1}{m}\sum_{i=1}^m \xi_i - \frac{\linetheta_{\ell,j}^{k\eta}(n)}{\|\linetheta_{\ell,j}\|_{\infty}+1} \right| 
< \ep_1,
\]
with probability $1-\frac{\delta}{4mL(K+1)N}$.
Therefore, applying the union bound again over $k\in[0:K],~\ell\in[L],~j\in[m]$ and $n\in[N]$, we have \eqref{event forward} with probability $1-\frac{\delta}{4}$.

For \eqref{event backward} with a given $k,\ell,i,n$, consider the random vectors 
$$\xi_j'  :=  [\linew_{\ell+1,i,j}^{k\eta}~ \lineDel_{\ell+1,j}^{k\eta}(n)] ~ h'\left( \bbtheta_{\ell,i}^{k\eta}(n)\right).  $$
Conditioned on $\bbtheta_{\ell,i}$, those $\xi_j'$ are independent and $\EE[\xi_j'|\bbtheta_{\ell,i}]=\lineDel_{\ell,i}^{k\eta}(n)$.
By the boundedness of $h'$ and the upper bound of $\lineDel_{\ell+1}$ in Lemma~\ref{lmm:Psi-property}, we have 
\[
| \xi_j'|
\le C' | \linew_{\ell+1,i,j}^{k\eta} |
\le C (1+\|\bbtheta_{\ell+1,j} \|_\infty),
\]
and thus $\xi_j'$ is sub-gaussian. 
Applying Lemma~\ref{azuma}, we obtain that
\[
\left| \frac{1}{m}\sum_{j=1}^m \xi_j' - \lineDel_{\ell,i}^{k\eta}(n) \right| 
< \ep_1,
\]
with probability $1-\frac{\delta}{4mL(K+1)N}$.
Therefore, applying the union bound again over $k\in[0:K],~\ell\in[L], ~j\in[m]$, and $n\in[N]$, we have \eqref{event backward} with probability $1-\frac{\delta}{4}$.

Finally both \eqref{event max} and \eqref{event moment} happen with probability $1-\frac{\delta}{4}$ by the concentration of sub-gaussian random variables; in particular, \eqref{event max} follows from Lemma \ref{theo:subgau} and \eqref{event moment} follows from Lemmas~\ref{azuma} and \ref{azumaexp}.
\end{proof}

\begin{proof}[Proof of Lemma~\ref{lmm:induction-k}]
We first consider the forward propagation and prove \eqref{eq:induction forward last} and \eqref{eq:induction forward}.
For $\ell=1$, since $\bX$ is bounded, 
\begin{eqnarray}
  \left\| \linetheta_{1,i}^{k\eta} - \hbtheta_{1,i}^{k} \right\|_{\infty} \leq C \|\linebw_{1,i}^{k\eta} - \hat{\bw}_{1,i}^k \|_{\infty} \leq   C (\|\linebw_{1,i}  \|_{\infty}+1) \ep. \notag
\end{eqnarray}
For $\ell\in[2:L]$, by the triangle inequality,
\begin{eqnarray}
&&\left\| \linetheta_{\ell,j}^{k\eta} - \hbtheta_{\ell,j}^{k}  \right\|_{\infty}\label{fordisa}\\
&\leq&\left\| \linetheta_{\ell,j}^{k\eta} -   \frac{1}{m}\sum_{i=1}^m \left[ \linew_{\ell,i,j}^{k\eta}~\th\left(\bbtheta_{\ell-1,i}^{k\eta}\right)\right]  \right\|_{\infty} 
+ \left\| \frac{1}{m}\sum_{i=1}^m \left[\linew_{\ell,i,j}^{k\eta}~\th\left(\linetheta_{\ell-1,i}^{k\eta}\right)  -\hw_{\ell,i,j}^{k}~\th\left(\hbtheta_{\ell-1,i}^{k}\right)  \right] \right\|_{\infty}. \notag
\end{eqnarray}
The first term is approximately the forward propagation that is at most $(\|\bbtheta_{\ell,j}\|_\infty+1)\ep_1$ by \eqref{event forward}. 
For the second term, since $h$ is bounded and Lipschitz continuous and the weights $\linew_{\ell,i,j}$ are upper bounded by Lemma~\ref{lmm:Psi-property} and Assumption~\ref{ass:3}, we have a further upper bound
\begin{align*}
&\phantomeq\frac{1}{m}\sum_{i=1}^m 
\underbrace{\left|\linew^{k\eta}_{\ell,i,j}\right|}_{\leq C(\| \linetheta_{\ell,j}\|_{\infty}+1)}~ \left\|\th\left(\linetheta^{k\eta}_{\ell-1,i}\right) -\th\left(\hbtheta^{k}_{\ell-1,i}\right)  \right\|_{\infty} +\frac{1}{m}\sum_{i=1}^m \underbrace{\left|\linew^{k\eta}_{\ell,i,j} - \hw^k_{\ell,i,j} \right|}_{\eqref{eq:induction weights}} ~\left\|\th (\hbtheta^{k}_{\ell-1,i}) \right\|_{\infty}\\
&\le C(\| \linetheta_{\ell,j}\|_{\infty}+1)(\ep+\ep_1),
\end{align*}
where in the last step we used \eqref{event moment}.
The output layer $\ell=L+1$ is similar by applying the upper bound of $w_{L+1}$ in Assumption~\ref{ass:3}.

Next we consider the backward propagation and prove \eqref{eq:induction backward}.
Since $\bX$ is bounded, $h$ is bounded and Lipschitz continuous, and $\lineDel_{\ell}$ is bounded by Lemma~\ref{lmm:Psi-property}, it suffices to prove that
\begin{equation}
\label{lineDel approx back}    
\left\|\lineDel_{\ell,i}^{k\eta} - \hDel_{\ell, i}^k \right\|_{\infty} 
\leq C\left(\| \linetheta_{\ell,i}\|_{\infty}+1\right)\left(\ep + \ep_1\right), \quad \forall~\ell\in[L+1],~i\in[m_\ell].
\end{equation}
At the output layer $\ell=L+1$, since $\phi_1'$ is Lipschitz continuous on the first argument, 
\[
\left\|\lineDel_{L+1,1}^{k\eta} -  \NhDel_{L+1,1}^k \right\|_{\infty}\leq L_5  \left\| \linetheta_{L+1,1}^{k\eta} - \hbtheta_{L+1,1}^{k} \right\|_{\infty}\leq C\left(\ep +\ep_1\right).
\]
At layer $\ell=L$, since $h'$ is bounded and Lipschitz continuous and $\lineDel_{L+1}$ is bounded by Lemma~\ref{lmm:Psi-property}, applying \eqref{eq:induction forward} yields that
\[
\left\| \lineDel_{L+1,1}^{k\eta}~ h'\left(\linetheta^{k \eta}_{L,i}\right) -   \NhDel_{L+1,1}^k~  h'\left(\hbtheta^{k \eta}_{L,i}\right) \right\|_\infty 
\le C\left(\left\|\linetheta_{L,i}\right\|_{\infty}+1\right)\left(\ep+\ep_1\right).
\]
Applying \eqref{eq:induction weights} and the upper bound of $\linew_{L,i}$ in Lemma~\ref{lmm:Psi-property}, we obtain that
\begin{equation}\label{laqqa}
\left\|\lineDel_{L,i}^{k\eta} -  \NhDel_{L,i}^k \right\|_\infty
\le C \left(\|\linetheta_{L,i} \|_{\infty}+ 1\right) \left(\ep+\ep_1\right).
\end{equation}
For each layer $\ell$ from $L-1$ to $1$, by the triangle inequality, 
\begin{eqnarray}
&&\!\!\!\!\left\|\lineDel_{\ell,i}^{k\eta} -  \NhDel_{\ell,i}^k \right\|_{\infty} \label{laqqa1}\\
&\leq& \!\!\!\!\left\| \lineDel_{\ell,i}^{k\eta} - \frac{1}{m}\sum_{j=1}^m \barw^{k\eta}_{\ell+1,i,j}~ \lineDel_{\ell+1,j}^{k\eta}\cdot \th'\left(\linetheta_{\ell+1,i}^{k\eta}\right) \right\|_{\infty} \notag\\
&+&\!\!\!\!\left\| \frac{1}{m}\sum_{j=1}^m \linew_{\ell+1,i,j}^{k\eta}\left[ \lineDel_{\ell+1,j}^{k\eta}\cdot \th'\left(\linetheta_{\ell,i}^{k\eta}\right)\right] -\frac{1}{m}\sum_{j=1}^m\hw_{\ell+1,i,j}^k\left[ \NhDel_{\ell+1,j}\cdot \th'\left(\hbtheta_{\ell,i}^{k\eta}\right)\right]\right\|_{\infty}. \notag
\end{eqnarray}
The first term is approximately backward propagation and is at most $\ep_1$ by \eqref{event backward}.
For the second term, note that $h'$ is bounded and Lipschitz continuous, $\linew_{\ell+1,i,j}$ and $\lineDel_{\ell,j}$ are upper bounded by Lemma~\ref{lmm:Psi-property}.
Applying \eqref{eq:induction weights}, \eqref{eq:induction forward}, and \eqref{lineDel approx back} at layer $\ell+1$ yields that
\begin{align*}
&\phantomeq \left\|\lineDel_{\ell+1,j}^{k\eta}\cdot\left[\linew_{\ell+1,i,j}^{k\eta}~ \th'\left(\linetheta_{\ell,i}^{k\eta}\right)\right] -\NhDel_{\ell+1,j}\cdot\left[\hw_{\ell+1,i,j}^k ~\th'\left(\hbtheta_{\ell,i}^{k\eta}\right)\right]\right\|_{\infty}\\
&\le C(\| \linetheta_{\ell+1,j}\|_{\infty}+1)\left(\|\linetheta_{\ell+1,j} \|_{\infty}+\|\linetheta_{\ell,i} \|_{\infty}+ 1\right) \left(\ep+\ep_1\right).
\end{align*}
Therefore, by \eqref{event moment}, we obtain \eqref{lineDel approx back} at layer $\ell$.
\end{proof}

\begin{proof}[Proof of Lemma~\ref{lmm:gradw-lip-t}]
The proof is similar to the backward steps in Lemma~\ref{lmm:induction-k}.
Since $h$ is Lipschitz continuous, by the upper bound of $\lineGrad_{\ell}^{\btheta}$ in Lemma~\ref{lmm:Psi-property}, we have
\begin{equation}
\label{eq:diff htheta}
\left\|\th\left(\btheta_\ell^{t_1}\right) -\th\left(\btheta_\ell^{t_2}\right) \right\|_{\infty}
\leq C\left(\|\btheta_\ell\|_{\infty}+1 \right)|t_1 -t_2 |.
\end{equation}
By the boundedness of $h$ and $\bX$, it suffices to prove the Lipschitz continuity $\lineDel_\ell$ that
\begin{align}
&\left\|\lineDel_{L+1}\left(\Psi_*, t_1\right) -\lineDel_{L+1}\left(\Psi_*, t_2\right)  \right\|_{\infty}\le C|t_1-t_2|,\label{lip-t-del-L+1}\\
&\left\|\lineDel_{\ell}(\btheta_{\ell};\Psi_*,t_1)-\lineDel_{\ell}(\btheta_{\ell};\Psi_*,t_2)\right\|_\infty
\le C\left(\|\btheta_\ell\|_{\infty}+1 \right)|t_1 -t_2 |,\quad \ell\in[L]\label{lip-t-del}.
\end{align}
At the output layer $\ell=L+1$, by the Lipschitz continuity of $\phi'$, we have 
\begin{eqnarray}
&&\left\|\lineDel_{L+1}\left(\Psi_*, t_1\right) -\lineDel_{L+1}\left(\Psi_*, t_2\right)  \right\|_{\infty}\label{qfrt}\\
&\leq& L_5\left\| \btheta_{L+1}^{t_1} -\btheta_{L+1}^{t_2} \right\|_{\infty}\notag\\
&\leq&L_5\left\|  \int w_{L+1}^{t_1}~ \th\left(\btheta_L^{t_1}\right) -w_{L+1}^{t_2}~ \th\left(\btheta_L^{t_2}\right) d p_L(\btheta_L)\right\|_{\infty}\notag.
\end{eqnarray}
By the upper bound and Lipschitz continuity of $w_{L+1}$ in Lemma~\ref{lmm:Psi-property}, we obtain \eqref{lip-t-del-L+1}.
At layer $\ell=L$, using \eqref{def linedel}, we obtain \eqref{lip-t-del} from the upper bounds and the Lipschitz continuity of $\lineDel_{L+1}(\btheta_L, \Psi_*, t)$, $ w^t_{L+1}(\btheta_L)$, and $\th'\left(\btheta_L^t\right)$.
For each layer $\ell$ from $L-1$ to $1$, we have 
\begin{eqnarray}
 &&\!\!\!\!\!\!\int  \left\|w_{\ell+1}^{t_1}(\btheta_{\ell},\btheta_{\ell+1}) ~ \lineDel_{\ell+1}\left(\btheta_{\ell+1}; \Psi_*,t_1\right) \!- \! w_{\ell+1}^{t_2}(\btheta_{\ell},\btheta_{\ell+1}) ~ \lineDel_{\ell+1}\left(\btheta_{\ell+1}; \Psi_*,t_2\right)\right\|_{\infty}dp_{\ell+1}(\btheta_{\ell+1})  \notag\\
 &&\leq\int \underbrace{\left| w_{\ell+1}^{t_1}(\btheta_{\ell},\btheta_{\ell+1}) -w_{\ell+1}^{t_2}(\btheta_{\ell},\btheta_{\ell+1})\right|}_{\leq  C| t_1 -t_2| } \left\| \lineDel_{\ell+1}(\btheta_{\ell+1}; \Psi_*,t) \right\|_{\infty} dp_{\ell+1}(\btheta_{\ell+1})  \notag\\
 &&\quad\quad\quad\quad + \int  \left|w_{\ell+1}^{t_2}(\btheta_{\ell},\btheta_{\ell+1}) \right| \underbrace{\left\|\lineDel_{\ell+1}(\btheta_{\ell+1}; \Psi_*,t_1) -  \lineDel_{\ell+1}(\btheta_{\ell+1}; \Psi,t_2)\right\|_{\infty}}_{\leq C\left(\|\btheta_{\ell+1} \|_{\infty}+1 \right)~|t_1 -t_2|  }dp_{\ell+1}(\btheta_{\ell+1})\notag\\
 &&\leq C' |t_1 -t_2|, \label{btcon}
\end{eqnarray}
where in the last step we used the upper bounds of $w_{\ell+1}$ and $\lineDel_{\ell+1}$ in Lemma~\ref{lmm:Psi-property}, sub-gaussianness of $p_{\ell+1}$, and Corollary~\ref{cor:subg-q}.
Then, combining \eqref{eq:diff htheta}, we obtain \eqref{lip-t-del} at layer $\ell$.
\end{proof}

\begin{proof}[Proof of Lemma~\ref{lmm:hK-K}]
In the proof of  Lemma~\ref{lmm:hK-K}, we treat the parameters in Assumptions \ref{ass:1} -- \ref{ass:3} as constants  and focus on the dependency on $N$, $\delta$, and $\ep_2$.

Recall that $\linetheta_{\ell,i}\iiddistr p_\ell$. 
Consider auxiliary matrices
\[
\bbK_\ell:= \frac{1}{m}\sum_{i=1}^m \th(\linetheta_{\ell,i})\th^\top(\linetheta_{\ell,i}),
\]
whose entry consists of i.i.d.~summation of the form
\[
\bbK_\ell(i,j)
= \frac{1}{m}\sum_{k=1}^m h(\bbtheta_{\ell,k}(i))h(\bbtheta_{\ell,k}(j)).
\]
Since $h$ is bounded, by Hoeffding's inequality, with probability $1-\frac{\delta}{3N^2 (L-1)}$,
\[
|\bbK_\ell(i,j)-\bK_\ell(i,j)|\le C\sqrt{\frac{1}{m}\log\frac{3N^2 (L-1)}{\delta}}.
\]
By the union bound, with probability $1- \delta/3$,
\begin{eqnarray}\label{eq fl1}
\max_{\ell\in[L-1]}\norm{\bbK_\ell-\bK_\ell}_2
\le N \max_{\ell\in[L-1]}\norm{\bbK_\ell-\bK_\ell}_\infty
\le \ep_3:=CN\sqrt{\frac{1}{m}\log\frac{N}{\delta}}.
\end{eqnarray}
The upper bounds of $\Norm{\bbtheta_{\ell+1,i} }_2$ happen with probability $1-\delta/3$ due to the sub-gaussianness of $p_{\ell+1}$ and Lemma~\ref{theo:subgau}.
We will also use the following upper bound that happen with probability $1-\delta/3$ by the sub-gaussianness of $p_\ell$ and Lemma~\ref{azuma}:
\[
\frac{1}{m}\sum_{i=1}^m\Norm{\bbtheta_{\ell,i}}_2
\le \beta:=C\sqrt{N}\log(eL/\delta),\qquad \forall~\ell\in[2:L].
\]

Next we inductively prove that, for $\ell\in[L-1]$, 
\begin{align}
\Norm{\bK_\ell-\hbK_\ell}_2
&\le (CN^{3/2}\blambda^{-1}\beta)^{\ell-1}\ep_3,\label{eq:init-bk-induction}\\
\left\|\hbtheta_{\ell+1,i} - \linetheta_{\ell+1,i}  \right\|_2
&\le (CN^{3/2}\blambda^{-1}\beta)^{\ell-1}N\blambda^{-1}\ep_3 \norm{\linetheta_{\ell+1,i}}_2.\label{eq:init-theta-induction}
\end{align}
% upper bound $\Norm{\bbK_\ell-\hbK_\ell}_2$ and prove the lemma for $\ell\in[L-1]$.
For $\ell=1$, by definition $\bbK_1=\hbK_1$.
The upper bound of $\Norm{\hbK_1^{1/2}-\bK_1^{1/2}}_2$ is achieved by matrix calculus \cite[Section V.3]{bhatia2013matrix}. 
Since $\Norm{\hbK_1-\bK_1}_2\le \frac{\blambda}{2}$, then the eigenvalues of $\hbK_1$ are at least $\frac{\blambda}{2}$. 
Let $f(x):=\sqrt{x}$. Then $|f'(x)|\ge \frac{1}{\sqrt{2\blambda}}$ when $x$ is the eigenvalue of $\hbK_1$. Applying \cite[(V.20)]{bhatia2013matrix} yields that
\begin{equation}
\label{matrix sqrt}
\left\|\hbK^{1/2}_{1}  - \bK^{1/2}_{1}\right\|_2
\le \frac{N}{\sqrt{2\blambda}}\norm{\hbK_1-\bK_1}_2.
\end{equation}
Consequently, 
\begin{align}
\left\|\hbtheta_{2,i} - \linetheta_{2,i}  \right\|_2
& = \left\|\left( \hbK^{1/2}_{1}\bK^{-1/2}_{1}  - \mathbf{I}^N \right)     \linetheta_{2,i} \right\|_2\nonumber\\
& \le \left\|\hbK^{1/2}_{1}  - \bK^{1/2}_{1}\right\|_2
\left\|\bK^{-1/2}_{1}\right\|_2
\left\|\linetheta_{2,i}\right\|_2\nonumber\\
&\le \frac{N\blambda^{-1}\ep_3}{\sqrt{2}}\Norm{\linetheta_{2,i}}_2.\label{theta-ub-cmp}
\end{align}
For $\ell\in[2:L-1]$,  suppose we have 
\[ \left\|\hbtheta_{\ell,i} - \linetheta_{\ell,i}  \right\|_2
\le (CN^{3/2}\blambda^{-1}\beta)^{\ell-2}N\blambda^{-1}\ep_3 \norm{\linetheta_{\ell,i}}_2.\]
By the boundedness and Lipschitz continuity of $h$, we have 
\begin{equation}\label{eq fl2}
    \norm{\hbK_\ell-\bbK_\ell}_2
\le \frac{C'N^{1/2}}{m} \sum_{i=1}^m \norm{\th(\linetheta_{\ell,i}) - \th(\hbtheta_{\ell,i})}_2
\le (C'' N^{3/2}\blambda^{-1}\beta)^{\ell-1}\ep_3.
\end{equation}
Then we obtain \eqref{eq:init-bk-induction} by triangle inequality from \eqref{eq fl1} and \eqref{eq fl2}. 
The upper bound in \eqref{eq:init-theta-induction} for $\ell+1$ follows from a similar argument of \eqref{matrix sqrt} and \eqref{theta-ub-cmp}.
\end{proof}

\begin{lemma}
\label{lmm:hat-w}
If $\hbK_\ell$ is invertible, then
\[
\hw_{\ell+1, i, j} = \hbtheta_{\ell+1,j}^\top\hbK^{-1}_{\ell}\th(\hbtheta_{\ell,i}), \quad \ell\in[L-1], ~i,j\in[m].
\]
\end{lemma}
\begin{proof}
For a given layer $\ell$ and $j$, the $\ell_2$-regression problem in Algorithm \ref{algo:init} can be equivalently written as
\begin{eqnarray}
\min_{\hbw } &\quad&\frac{1}{2}\left\|\hbw \right\|^2\label{rrop}\\
\mathrm{s.t.} &\quad& \frac{1}{m}\hbH \hbw = \hbtheta_{\ell+1,j},  \notag
\end{eqnarray}
where $\hbw=(\hw_{\ell+1,1,j}, \dots, \hw_{\ell+1,m,j})^\top$ and $\hbH = [\th(\hbtheta_{\ell,1}),\dots,\th(\hbtheta_{\ell,m})] $.
Decompose $\hbw$ as 
\[
\hbw = \hbH^\top \mathbf{z} +  \hbw',
\]
where $\mathbf{z}\in\RR^m$ and $\hbH \hbw'=0$.
Then \eqref{rrop} is equivalent to 
\begin{eqnarray}
\min_{~\mathbf{z},~ \hbw' } &\quad&\frac{1}{2} \left\| \hbH^\top \mathbf{z}\right\|_{2}^2 + \frac{1}{2}\left\| \hbw' \right\|_2^2 \notag\\
\mathrm{s.t.} &\quad& \frac{1}{m}\hbH \hbH^\top \mathbf{z} =  \hbtheta_{\ell+1,j}.\notag
\end{eqnarray}
Since $\frac{1}{m}\hbH \hbH^\top = \hbK_\ell$ is invertible, the optimal solution is $\mathbf{z}   =   \hbK_{\ell}^{-1}\hbtheta_{\ell+1,j}$ and $\hbw'=\mathbf{0}^N$.
\end{proof}

\section{Proofs of Theorems \ref{theorm:resflow p0} and \ref{theorm:phi}}\label{sec:c}
%We prove  Theorems \ref{theorm:resflow p0} and \ref{theorm:phi} in a unified scheme. 
\subsection{Proof of  Theorem \ref{theorm:resflow p0} }
The proof of Theorem \ref{theorm:resflow p0} is similar to that of Theorem \ref{theorm:flow p0} with a special consideration  on the weights.  We also first show that our neural feature flow  in Definition \ref{resNFLp0} satisfies several continuity properties, which allows us to narrow down the search space for the solution. Recall that a trajectory $\Phi$ consists of trajectories of weights $\Phi_\ell^\bv$ for $\ell\in[L+1]$, features $\Phi_\ell^{\bbeta}$ for $\ell\in[L]$, and residuals $\Phi_\ell^{\balpha}$ for $\ell\in[2:L]$. For $\bTheta, \bbTheta\in \supp(p)$,
we also abbreviate the notations for individual trajectories as
$$ \Phi_\ell^{\bv}(\bu_\ell)(t)=v_\ell^t(\bu_{\ell}), \quad \Phi_\ell^{\bbeta}(\bTheta)(t)=\bbeta_\ell^t(\bTheta),\quad    \Phi_\ell^{\balpha}(\bTheta)(t)=\balpha_\ell^t(\bTheta), $$
where $\bu_\ell$ stands for $\bTheta$, $(\bTheta,\bbTheta)$, $\bTheta$  for $\ell=1$, $2\le\ell\le L$, $\ell=L+1$, respectively.  

Throughout the proof, we fix $T$ as a constant. We define the set of continuous restricted trajectories below.
\begin{definition}[$\bR$-Continuous Restricted Trajectory]\label{special Function R}
Given $\bR:=(\bR_1,\dots,\bR_{L+1})\in\RR_+^{L+1}$, we say $\Phi$ is a $\bR$-continuous restricted trajectory  if
$\Phi_{\ell}^{\bv}(\bu_{\ell})(t)$ is $\bR_{\ell}$-Lipschitz continuous in $t\in[0,T]$ for $\ell\in[L+1]$, and  $\Phi_{\ell}^{\balpha}(\bu_{\ell})(t)$ and $\Phi_{\ell}^{\bbeta}(\bu_{\ell})(t)$  are determined by the forward-propagation process, i.e., 
$\bbeta_1^t(\bTheta) =  \frac{1}d \bX v_1^t(\bTheta)$,  $\balpha^t_{\ell+1}(\bTheta) = \int v_{\ell+1}^t(\bTheta, \bbTheta) \th_1\left(\bbeta_{\ell}^t(\bbTheta) \right) d p(\bbTheta)$,  $\bbeta_{\ell+1}^t(\bTheta) = \bbeta_{\ell}^t(\bTheta)+ \th_2\left( \balpha^t_{\ell+1}(\bTheta)\right)$ for $\ell\in[L-1]$ and $\bTheta\in\supp(p)$.
The set of  $\bR$-continuous restricted  trajectories is denoted as $\mathbf{\Phi}^{\bR}$.
\end{definition}
We can find that given the trajectories of weights, the trajectories of residuals  and features are determined by the forward-propagation process.  Lemma \ref{lmm:flow-condition R} below shows that it suffices to consider a restricted search space.
\begin{lemma}
\label{lmm:flow-condition R}
There exists constants $\bR\in\RR_+^{L+1}$ such that every solution $\Phi$ of the neural feature flow is a  $\bR$-continuous restricted  trajectory.
\end{lemma}

In the remaining of the proof we let $\bR$ be the constants in Lemma~\ref{lmm:flow-condition R}, and let $\mathbf{\Phi}:=\mathbf{\Phi}^{\bR}$, which will serve as the search space.  We introduce the mapping $F_2$ below. In fact, the fixed-point of $F_2$ is equivalent to the solution of neural feature flow.  
\begin{definition}\label{def:F2}
Define $F_2: \mathbf{\Phi} \to  \mathbf{\Phi}$ as follows:
for all $t\in[0,T]$,
\begin{enumerate}[(1)]
    \item for all $\ell\in[L+1]$ and all $\bu_\ell$,
    $$F_2(\Phi)_{\ell}^\bv(\bu_\ell)(t) = v_{\ell}^0(\bu_\ell)-\int_{0}^t\ulineGrad^{\bv}_{\ell}( \bu_\ell;\Phi,s)ds,$$
%where  for all $\bTheta = (\bv_1, \balpha_2, \dots, \balpha_L), \bbTheta$, we have $ v_{1}^0(\bTheta) = \bv_1$,  $v_{\ell}^0(\bTheta,\bbTheta) = v_{\ell}(\bTheta,\bbTheta)$ for $\ell\in[2:L]$, and $v_{L+1}^0(\bTheta) = v_{L+1}(\bTheta)$,
\item   for all $\bTheta$, 
$$F_2(\Phi)_{1}^{\bbeta}(\bTheta)(t) = \frac{1}{d}\left[ \bX F_2(\Phi)_{1}^{\bv}(\bTheta)(t)\right],$$
\item   for all $\ell\in[L-1]$ and $\bTheta$, 
\begin{eqnarray}
    F_2(\Phi)_{\ell+1}^{\balpha}(\bTheta)(t)  &=&  \int F_2(\Phi)_{\ell+1}^{\bv}(\bTheta, \bbTheta)(t) ~\th_1\left(\Phi_{\ell}^{\bbeta}(\bbTheta)(t)\right) d p(\bbTheta), \notag\\
    F_2(\Phi)_{\ell+1}^{\bbeta}(\bTheta)(t)&=&\th_2\left(F_2(\Phi)_{\ell+1}^{\balpha}(\bTheta)(t)\right)+  F_2(\Phi)_{\ell}^{\bbeta}(\bTheta)(t), \notag
\end{eqnarray}
\end{enumerate}
where for $\bTheta=(\bv_1,\balpha_2,\dots, \balpha_L )\in \supp(p)$,  $\bbTheta\in\supp(p)$, $v^0_1(\bTheta) = \bv_1$, $v^0_{\ell}(\bTheta, \bbTheta) = v_{\ell}(\bTheta, \bbTheta)$ with $\ell\in[2:L]$, and $v^0_{L+1}(\bTheta) = v_{L+1}(\bTheta)$.
\end{definition}

Following the same argument as Lemma~\ref{lmm:flow-condition R}, we have that the image of $\mathbf{\Phi}$ under $F_2$ is indeed contained in $\mathbf{\Phi}$.  We then show in Lemma~\ref{lmm:contraction R} the contraction property of $F_2$ under an appropriate metric defined below:

\begin{definition}
\label{def:metric R}
For a pair $\Phi_1,\Phi_2\in  \mathbf{\Phi}$, we define the normalized distance between each trajectories over $[0,t]$ as 
\[
\Dis^{[0,t]}(\Phi_1,\Phi_2)
:=\sup_{s\in[0,t],~ \ell\in[L+1],~\bu_\ell}
\frac{\|\Phi_{1,\ell}^\bv(\bu_\ell)(s)
-\Phi_{2,\ell}^\bv(\bu_\ell)(s)\|_\infty}
{1+\|\bu_\ell\|_\infty}.
\]
\end{definition}

\begin{lemma}
\label{lmm:contraction R}
There exists a constant $R$ such that
\begin{eqnarray}
\Dis^{[0,t]}(F_2(\Phi_1),F_2(\Phi_2) )\leq R \int_{0}^t \Dis^{[0,s]}(\Phi_1,\Phi_2 ) ds. \notag
\end{eqnarray}
\end{lemma}

\begin{proof}[Proof of Theorem \ref{theorm:resflow p0}]
 Firstly, $\mathbf{\Phi}$ contains the constant trajectory and thus is nonempty.
Applying Lemma~\ref{lmm:contraction R},  we have that
\[
\Dis^{[0,T]}(F_2^m(\Phi_1),F_2^m(\Phi_2) )\leq \frac{(CT)^m}{m!} \Dis^{[0,T]}(\Phi_1,\Phi_2 ).
\]
 Let $\Phi$ be  the constant trajectory,  for any $\tPhi\in \mathbf{\Phi}$,  by  the upper bounds of   $\ulineGrad^{\bv}_{\ell}$ in Lemma \ref{lmm:Phi-property} and the Definition of $\Dis^{[0,T]}$ in Definition \ref{def:metric R}, there is a constant $C$ such that 
$$ \Dis^{[0,T]}(F_2(\tPhi),\Phi) \leq CT<\infty.$$

From the  argument in Theorem \ref{theorm:flow p0}, Lemma \ref{lmm:flow-condition R} implies the
 uniqueness claim.  
For the existence, we can also consider the sequence $\{F^i_2(\Phi):i\ge 0\}$ that satisfies 
\[
\Dis^{[0,T]}(F^{m+1}_2(\Phi),F^m_2(\Phi))
\le \frac{(CT)^m}{m!} \Dis^{[0,T]}(F_2(\Phi),\Phi),
\]
which shows that $\{F^i_2(\Phi):i\ge 0\}$ is a Cauchy sequence. 
Since $\mathbf{\Phi}$ is complete under $\Dis^{[0,T]}$ by Lemma~\ref{lmm:psi-complete R}, the limit point $\Phi_*\in \mathbf{\Phi}$, which is a fixed-point of $F_2$. Finally, by 
dominated convergence theorem, we can directly verify that $\Phi_*$ is the solution of neural feature flow. 
\end{proof}

\subsection{Proof of Theorem \ref{theorm:phi}}
We follow the same technique used in Theorem \ref{theorem:conpsi} to prove Theorem \ref{theorm:phi}. %Recall the set $\mathbf{\Phi}$ in the proof of Theorem~\ref{theorm:resflow p0}, and the mapping $F_2:\mathbf{\Phi}\mapsto\mathbf{\Phi}$ in Definition~\ref{def:F2}. 
We will construct a closed nonemtpy subset $\btPhi \subseteq \bfPhi$ with the desired properties in Theorem~\ref{theorm:phi} such that $F_2(\btPhi)\subseteq \btPhi$. Then
by the same argument as the proof of Theorem~\ref{theorm:resflow p0}, the Picard iteration guarantees the solution in  $\btPhi$, thereby proving $\Phi_* \in \btPhi$. 

We introduce the set of $b$-locally Lipschitz trajectories with the desired properties in Theorem \ref{theorm:phi}.
We use similar notations as in the proof of Theorem~\ref{theorm:resflow p0} by letting $\bu_{\ell}$ denote $\bTheta$, $(\bTheta,\bbTheta)$,  $\bTheta$ for $\ell=1$,  $2\le\ell\le L$, $\ell=L+1$, respectively.

\begin{definition}[$b$-Locally Lipschitz Trajectory]
\label{def:local-lip2}
Recall the constants $C_6$ and $C_8$ in Assumption~\ref{ass:6} for the locally Lipschitz continuity at $t=0$.
We say $\Phi$ is $b$-locally Lipschitz if for all $t\in[0,T]$,  $\bTheta_1$, $\bbTheta_1\in\mathcal{B}_{\infty}(\bTheta_1,1) $, $\bbTheta_2$, and $\bbTheta_{2}\in \mathcal{B}_{\infty}(\bbTheta_2,1) $, we have
\begin{subequations}
\begin{align}
&\left\|\Phi^\bv_{1}(\bTheta_1)(t)-\Phi^\bv_{1}(\bbTheta_1)(t) \right\|_{\infty}
\leq  e^{b t}( \|\bTheta_1\|_{\infty}+1 )\|\bTheta_1 -\bbTheta_1 \|_{\infty},\label{local lip-w1 R}\\
&\left|\Phi^\bv_{L+1}( \bTheta_1)(t) - \Phi^\bv_{L+1}( \bbTheta_1)(t)  \right|
\leq (1+C_8)e^{b t} ( \|\bTheta_1\|_{\infty}+1 )\|\bTheta_1 -\bbTheta_1 \|_{\infty},\label{local lip-wL+1 R}\\
&\left|\Phi^\bv_{\ell}(\bu_{\ell})(t) - \Phi^\bv_{\ell}(\bbTheta_1, \bTheta_2)(t)  \right|
\leq (1+C_6)e^{b t}( \|\bu_{\ell}\|_{\infty}+1)  \|\bTheta_1-\bbTheta_1 \|_{\infty},\label{local lip-wl1 R}\\
&\left|\Phi^\bv_{\ell}(\bu_{\ell})(t) - \Phi^\bv_{\ell}(\bTheta_1, \bbTheta_2)(t)  \right|
\leq (1+C_6)e^{b t}( \|\bu_{\ell}\|_{\infty}+1) \|\bTheta_2-\bbTheta_2 \|_{\infty},\label{local lip-wl2 R}
\end{align}
\end{subequations}
for $\ell\in[2:L]$.
Denote the set of all $b$-locally Lipschitz trajectories as $\bfPhi_{b}$.
\end{definition}

\begin{lemma}
\label{const Lstar R}
There exists a constant $b_*$ such that $F_2(\bfPhi\cap \bfPhi_{b_*})\subseteq   \bfPhi_{b_*}$.
\end{lemma}

\begin{proof}[Proof of Theorem~\ref{theorm:phi}]
Let $b_*$ be the constant in Lemma~\ref{const Lstar R} and $\bfPhi':=\bfPhi\cap \bfPhi_{b_*}\subseteq \bfPhi$, which clearly contains the constant trajectory and thus is nonempty.  From Lemma~\ref{const Lstar R}, Lemma~\ref{lmm:contraction R} that $F_2$ is a contraction mapping, and the fact that  $\bfPhi'$ is a closed set (which can be directly obtained using the same argument as Lemma~\ref{psi closed}), by the same argument as the proof of Theorem~\ref{theorm:resflow p0}, there exists a  solution in $\bfPhi'$, which is  $\Phi_*$ in Theorem~\ref{theorm:resflow p0} due to the uniqueness of solution. 
\end{proof}

\subsection{Proofs of Lemmas}
\begin{proof}[Proof of Lemma~\ref{lmm:flow-condition R}]
We first prove the  Lipschitz continuity of $\Phi$ for weight. It suffices to show upper bounds of $\ulineGrad^{\bv}_{\ell}$  for each layer $\ell$.   We use the backward equations to inductively upper bound $\ulineDel_{\ell}^{\bbeta}$ and $\ulineDel_{\ell}^{\balpha}$,  which immediately yield upper bounds $\|\ulineGrad^{\bv}_{\ell}\|_\infty\le \bR_{\ell}$ for constants $\bR_\ell$.

For the top layer $\ell=L+1$, by Assumption \ref{ass:5} that $|\phi_1'|\leq L_4$, we have
\[
\|\ulineDel_{L+1}(\Phi,t)\|_{\infty} 
\le L_4 := \tilde \bR_{L+1}.\]

At layer $\ell=L$, since $|h'_1|\le L_2$, 
\begin{eqnarray}
 \left\|\ulineDel^{\bbeta}_L(\bTheta; \Phi,t) \right\|_{\infty} \leq  \underbrace{\left\|\th_1'\left(\bbeta_{L}^t(\bTheta)\right)\right\|_{\infty}}_{\leq L_2} \underbrace{\left\|\ulineDel_{L+1}(\Phi,t)\right\|_{\infty}}_{\leq \tilde \bR_{L+1}}~~  \underbrace{\left|v_{L+1}^t(\bTheta)\right|}_{\leq C_7+T\bR_{L+1}  }  \leq\tilde \bR_{L},\label{usedel}
\end{eqnarray}
where $\tilde \bR_{L}:= (C_7+\bR_{L+1}T)L_2\tilde \bR_{L+1}$ and $|v_{L+1}^t|\le C_7+\bR_{L+1}T$ by the upper bound of initialization \eqref{ass:6} and the $\bR_{L+1}:= L_1 \tilde \bR_{L+1}$-Lipschitz continuity of $v_{L+1}^t$ in $t$.   For each $\ell=L-1,\dots,1$,  suppose  $ \ulineDel_{\ell+1}^{\bbeta}$ is uniform bounded by $\tilde \bR_{\ell+1}$. Then we  have
\begin{eqnarray}
 \left\|\ulineDel^{\balpha}_{\ell+1}(\bTheta; \Phi,t) \right\|_{\infty} \leq   \left\|\ulineDel^{\bbeta}_{\ell+1}(\bTheta; \Phi,t) \right\|_{\infty}   \left\|\th_2'\left(\balpha_{{\ell+1}}^t(\bTheta)\right)\right\|_{\infty}\leq  \tilde \bR_{\ell+1}L_2:=  \tilde \bR_{\ell+1}'. \notag
\end{eqnarray}
then 
$$\left|\ulineGrad^{\bv}_{\ell+1}(\bTheta, \bbTheta ;\Phi,t)\right| \leq \left\|\ulineDel^{\balpha}_{\ell+1}(\bTheta;\Phi,t)\right\|_{\infty} \left\|\th_1\left(\bbeta_{\ell}^t(\bTheta)\right)\right\|_{\infty}\leq \bR_{\ell+1}.$$ By the sub-gaussian property of $p$ and  the upper bound of $v_{\ell+1}$ in Assumption \ref{ass:6}, we have,  by the  same  argument as \eqref{Del bound} that
\[
\left\|\int  v_{\ell+1}^t(\bTheta,\bbTheta)~ \ulineDel_{\ell+1}(\bbTheta; \Phi,t) ~ dp(\bbTheta)\right\|_{\infty} 
  \leq R',
  \]
for a constant $R'$.  We achieve that 
 \begin{eqnarray}
 && \left\|\ulineDel_{\ell}^{\bbeta}(\bTheta; \Phi,t)\right\|_{\infty}\notag\\
 &\leq&    \left\|\ulineDel_{\ell+1}^{\bbeta}(\bTheta; \Phi,t)\right\|_{\infty}+ \left\|\th_1'\left(\bbeta^t_{\ell}(\bTheta)\right) \cdot \int  v_{\ell+1}^t(\bTheta,\bbTheta)~ \ulineDel_{\ell+1}(\bbTheta; \Phi,t)~  dp(\bbTheta)\right\|_{\infty}\notag\\
 &\leq& \tilde \bR_{\ell+1}+  R'L_2:=  \tilde \bR_{\ell}.
\end{eqnarray}
 
Now we turn to the forward steps. We 
prove that there is a constant $R$ such that for $\ell\in[L]$ and $\bTheta$, 
\begin{eqnarray}\label{Gra beta}
\left\|\ulineGrad^{\bbeta}_{\ell}(\bTheta;\Phi,t)\right\|_{\infty}\leq R\left(\| \bTheta\|_{\infty} +1\right) , 
\end{eqnarray}
and for all $\ell\in[2:L]$ and $\bTheta$,
\begin{eqnarray}\label{Grad alpha}
\left\|\ulineGrad^{\balpha}_{\ell}(\bTheta;\Phi,t)\right\|_{\infty}\leq R\left(\| \bTheta\|_{\infty} +1\right).\notag
\end{eqnarray}
Because $p$ has bounded finite moment (Corollary~\ref{cor:subg-q}), by the
dominated convergence theorem,  we have that $\Phi^{\alpha}$ and $\Phi^{\bbeta}$ satisfy  the forward equations in Definition \ref{special Function R}, which is our desired result.

For the first layer $\ell=1$, since $\bX$ is bounded,  we have
\[ \left\|\ulineGrad^{\bbeta}_{\ell}(\bTheta;\Phi,t)\right\|_{\infty}\leq R\leq R\left(\| \bTheta\|_{\infty} +1\right). \]
   Suppose that at layer $\ell\in[L-1]$,
$$ \left\|\ulineGrad^{\bbeta}_{\ell}(\bTheta;\Phi,t)\right\|_{\infty}\leq R\left(\| \bTheta\|_{\infty} +1\right).  $$
 Using a  same similar argument as  \eqref{tend1}, we  have 
  $$ \left\|\ulineGrad^{\balpha}_{\ell+1}(\bTheta;\Phi,t)\right\|_{\infty}\leq R'\left(\| \bTheta\|_{\infty} +1\right) ,\notag $$
  for a  constant $R'$. Then we  obtain
\begin{eqnarray}
   &&\left\|\ulineGrad^{\bbeta}_{\ell+1}(\bTheta;\Phi,t)\right\|_\infty\notag\\
   &\leq& \left\|\th_2'\left(\balpha_{\ell+1}^{t}(\bTheta)\right) \cdot \ulineGrad^{\balpha}_{\ell+1}(\bTheta;\Phi,t)
   \right\|_\infty+ \left\|\ulineGrad^{\bbeta}_{\ell}(\bTheta;\Phi,t)\right\|_\infty\notag\\
   &\leq& (L_2R'+R) \left(\| \bTheta\|_{\infty} +1\right). \notag
\end{eqnarray}
We achieve Lemma \ref{lmm:flow-condition R}.
\end{proof}

Before proving Lemma~\ref{lmm:contraction R}, we first present in Lemma~\ref{lmm:Phi-property} properties of $\Phi \in \mathbf{\Phi}$ that will be used to prove the contraction lemma. 
The proof is exactly the same as Lemma~\ref{lmm:flow-condition R} and is omitted. 
\begin{lemma}[Property of $\mathbf{\Phi}$]
\label{lmm:Phi-property}
There exist a generic constant $ R$   such that, for any $\Phi\in\mathbf{\Phi}$, we have
\begin{itemize}
    \item $\|\ulineDel_{L+1}(\Phi,t)\|_\infty \le R$ and $\|\ulineDel_{\ell}^{\bbeta}(\bTheta;\Phi,t)\|_\infty \le R$ for $\ell\in[L]$;
        \item  $\|\ulineDel_{\ell}^{\balpha}(\bTheta;\Phi,t)\|_\infty \le R$ for $\ell\in[2:L]$;
    \item $\|\ulineGrad^{\bv}_{\ell}( \bu_\ell;\Phi,t)\|_\infty \le R$ and $\|v_{\ell}^t(\bu_\ell)\|_\infty \le \|v_{\ell}^0(\bu_\ell)\|_\infty + R~ t$ for $\ell\in[L+1]$;
    \item $\|\ulineGrad^{\bbeta}_{\ell}( \bTheta; \Phi,t)\|_\infty \le R\left(\| \bTheta\|_{\infty} +1\right)$ for $\ell\in[L]$;
    \item $\|\ulineGrad^{\balpha}_{\ell}( \bTheta; \Phi,t)\|_\infty \le R\left(\| \bTheta\|_{\infty} +1\right)$ for $\ell\in[2: L]$.
\end{itemize}
\end{lemma}

\begin{proof}[Proof of Lemma~\ref{lmm:contraction R}]
The proof entails upper bounds of the gradient differences 
$\|\ulineGrad^{\bv}_{\ell}( \bu_\ell;\Phi_1,t)-\ulineGrad^{\bv}_{\ell}( \bu_\ell;\Phi_2,t)\|_{\infty}$ in terms of the differences  $|v_{1,\ell}^t-v_{2,\ell}^t|$ for $\ell\in[L+1]$, which can  be further upper bounded in terms of $d_t:=\Dis^{[0,t]}(\Phi_1,\Phi_2 )$ by Definition \ref{def:metric R}:
\begin{align}
&\left\| v_{1,\ell}^t(\bu_\ell) - v_{2,\ell}^t(\bu_\ell)\right\|_{\infty} \leq(\| \bu_\ell\|_{\infty} +1)d_t\label{eq:dt-w q},\quad \ell\in [L+1].
\end{align}
We first prove that
\begin{equation}
  \begin{split}\label{lip res pi}
   \left\| \bbeta^{t}_{1,\ell}(\bTheta) - \bbeta^{t}_{2,\ell}(\bTheta) \right\|_{\infty} &\leq R\left(\|\bTheta \|_{\infty}+1\right)d_t,\quad \ell\in[L], \\
   \left\| \balpha^{t}_{1,\ell}(\bTheta) - \balpha^{t}_{2, \ell}(\bTheta) \right\|_{\infty} &\leq R\left(\|\bTheta \|_{\infty}+1\right)d_t,\quad \ell\in[2:L]. 
  \end{split}
\end{equation}
and then prove that 
\begin{align}
&\left\|\ulineDel_{\ell}^{\bbeta}(\bTheta;\Phi_1,t) -\ulineDel_{\ell}^{\bbeta}(\bTheta; \Phi_2,t)\right\|_\infty
\le R (1+\|\bTheta\|_\infty) d_t,\quad \ell\in[L],\label{ugrad lips2}\\
&\left\|\ulineDel_{\ell}^{\balpha}(\bTheta;\Phi_1,t) -\ulineDel_{\ell}^{\balpha}(\bTheta; \Phi_2,t)\right\|_\infty
\le R (1+\|\bTheta\|_\infty) d_t,\quad \ell\in[2:L],\label{ugrad lips1}\\
&\left\|\ulineGrad^{\bv}_{\ell}( \bu_\ell;\Phi_1,t)-\ulineGrad^{\bv}_{\ell}( \bu_\ell;\Phi_2,t)\right\|_{\infty}
\le R (1+\|\bu_\ell\|_\infty) d_t,\quad \ell\in[L+1], \label{ugrad lips}
\end{align}
Finally, Lemma~\ref{lmm:contraction R} directly follows from \eqref{ugrad lips}, and the definitions of $F_2$ and $\Dis^{[0,t]}$ in Definitions~\ref{def:F2} and \ref{def:metric R}, respectively.

We consider forward steps to obtain \eqref{lip res pi}.  When $\ell=1$, because $\bX$ is bounded, we have $ \left\| \bbeta^{t}_{1,1}(\bTheta) - \bbeta^{t}_{2,1}(\bTheta) \right\|_{\infty} \leq R\left(\|\bTheta \|_{\infty}+1\right)d_t$. 
Suppose at layer $\ell\in[L-1]$, we have   $ \left\| \bbeta^{t}_{1,\ell}(\bTheta) - \bbeta^{t}_{2,\ell}(\bTheta) \right\|_{\infty} \leq R\left(\|\bTheta \|_{\infty}+1\right)d_t$. Then 
    \begin{eqnarray}
&&\left\|\balpha_{1,\ell+1}^{t}(\bTheta) - \balpha_{2,\ell+1}^{t}(\bTheta)\right\|_\infty\notag\\
&\leq&\left\| \int \th_1\big(\bbeta_{1,\ell}^{t}(\bbTheta)\big)v^t_{1,\ell+1}(\bbTheta, \bTheta)~ -\th_1\big(\bbeta_{2,\ell}^{t}(\bbTheta)\big)~v^t_{2,\ell+1}(\bbTheta, \bTheta)   d p(\bbTheta)   \right\|_{\infty}\notag\\
&\leq&  \int \underbrace{\left\|\th_1\big(\bbeta_{1,\ell}^{t}(\bbTheta)\big)-\th_1\big(\bbeta_{2,\ell}^{t}(\bbTheta)\big)\right\|_{\infty}}_{\leq L_2  R \left(\| \bbTheta\|+1\right) d_t }\left|v_{1,\ell+1}^t(\bbTheta, \bTheta)\right|    d p(\bbTheta)  \notag\\
&& \quad\quad\quad\quad\quad\quad\quad\quad\quad\quad + \int \underbrace{\left\|\th_1\big(\bbeta_{2,\ell}^{t}(\bbTheta)\big)\right\|_{\infty}}_{\leq L_1} \underbrace{\left| v_{1,\ell+1}^t(\bbTheta, \bTheta) -v_{2,\ell+1}^t(\bbTheta, \bTheta)\right|}_{ \leq \left( \| \bbTheta\|_{\infty}+ \| \bTheta\|_{\infty}+1\right)\Dis^{[0,t]}(\Pi, \tPi  )}    d p(\bbTheta)  \notag\\
&\leq& R'\left( \| \bTheta\|_{\infty}+1\right) d_t,\notag
    \end{eqnarray}
for a constant $R'$, where the last step is due to the sub-gaussianness of $p$, Corollary~\ref{cor:subg-q}, and the upper bound of  $v_{\ell+1}^t$ in Lemma~\ref{lmm:Phi-property}.  Then it follows to have that
\begin{eqnarray}
   &&\left\|\bbeta_{1,\ell+1}^{t}(\bTheta) - \bbeta_{2,\ell+1}^{t}(\bTheta)\right\|_\infty\notag\\
   &\leq& \left\|\th_2\big(\balpha_{1,\ell+1}^{t}(\bTheta)\big) - \th_2\big(\balpha_{2,\ell+1}^{t}(\bTheta)\big)\right\|_\infty+ \left\|\bbeta_{1,\ell}^{t}(\bTheta) - \bbeta_{2,\ell}^{t}(\bTheta)\right\|_\infty\notag\\
   &\leq&R''\left( \| \bTheta\|_{\infty}+1\right)  d_t \notag
\end{eqnarray}
for constant $R''$. We achieve \eqref{lip res pi}. 

We  turn to the backward steps. 
We prove \eqref{ugrad lips2} and \eqref{ugrad lips1}, then  \eqref{ugrad lips} can be obtained accordingly. We introduce
\begin{equation}\label{def gamma}
 \ulineDel^{\bgamma}_{\ell}(\bTheta;\Phi,t) := \int v_{\ell}^t(\bTheta, \bbTheta)~ \ulineDel_{\ell}^{\balpha}(\bbTheta;\Phi,t)dp(\bbTheta)\cdot \th_1'\big(\bbeta_{\ell-1}^t(\bTheta)\big), \quad \ell\in[2:L] . 
\end{equation}
and have 
$$   \ulineDel_{\ell-1}^{\bbeta}(\bTheta; \Phi,t) =  \ulineDel_{\ell}^{\bbeta}(\bTheta; \Phi,t) + \ulineDel^{\bgamma}_{\ell}(\bTheta;\Phi,t),\quad \ell\in[2:L]. $$

When $\ell=L+1$, using the same argument as \eqref{qwe}, we have
\[
\|\ulineDel_{L+1}(\Phi_1,t) -\ulineDel_{L+1}(\Phi_2,t)\|_\infty
\le R\cdot d_t,
\]
for a constant $R$.
Then it follows from \eqref{cos r} to have 
$$\left\|\ulineDel_{L}^{\bbeta}(\bTheta;\Phi_1,t)-\ulineDel_{L}^{\bbeta}(\bTheta;\Phi_2,t) \right\|_{\infty}\leq R'\left(\| \bTheta\|_{\infty}+1\right) d_t,  $$
for a constant $R'$. For each $\ell=L-1,\dots,1$, suppose there is a constant $R$, such that   
$$ \left\|\ulineDel_{\ell+1}^{\bbeta}(\bTheta;\Phi_1,t)-\ulineDel_{\ell+1}^{\bbeta}(\bTheta;\Phi_2,t) \right\|_{\infty}\leq R\left(\| \bTheta\|_{\infty}+1\right) d_t.  $$
Because $h_2'$ is $L_3$-Lipschitz continuous,  using  the boundedness of $h_2$ and $\ulineDel^{\bbeta}_{\ell+1}$  in Lemma \ref{lmm:Phi-property}, we  have
  \begin{eqnarray}
  && \left\|\ulineDel_{\ell+1}^{\balpha}(\bTheta;\Phi_1,t)-\ulineDel_{\ell+1}^{\balpha}(\bTheta;\Phi_2,t) \right\|_{\infty}\leq  R'\left(\| \bTheta\|_{\infty}+1\right) d_t,\notag
  \end{eqnarray}
  for a constant $R'$. Then following the same argument as  \eqref{qq1}, we can obtain
\begin{eqnarray}
  &&\left\|   \ulineDel^{\bgamma}_{\ell+1}(\bTheta;\Phi_1,t) -\ulineDel^{\bgamma}_{\ell+1}(\bTheta;\Phi_2,t)\right\|_{\infty}\leq  R''\left(\| \bTheta\|_{\infty}+1\right) d_t, \notag
\end{eqnarray}
for a constant $R''$, which further implies 
\begin{eqnarray}
    && \left\|\ulineDel_{\ell}^{\bbeta}(\bTheta;\Phi_1,t)-\ulineDel_{\ell}^{\bbeta}(\bTheta;\Phi_2,t) \right\|_{\infty}\notag\\
   &\leq &\left\|\ulineDel_{\ell+1}^{\bbeta}(\bTheta;\Phi_1,t)-\ulineDel_{\ell+1}^{\bbeta}(\bTheta;\Phi_2,t)   \right\|_{\infty}+\left\|   \ulineDel^{\bgamma}_{\ell+1}(\bTheta;\Phi_1,t) -\ulineDel^{\bgamma}_{\ell+1}(\bTheta;\Phi_2,t)\right\|_{\infty}\notag\\
&\leq&  R'''\left(\| \bTheta\|_{\infty}+1\right) d_t,\notag
\end{eqnarray}
for a constant $R'''$.  We finish the proof. 
\end{proof}

\begin{lemma}
\label{lmm:psi-complete R}
$\mathbf{\Phi}$ is complete under $\Dis^{[0,T]}$.    
\end{lemma}
\begin{proof}
Let $\{\Phi_n:n\ge 0\}$ be a Cauchy sequence under $\Dis^{[0,T]}$. 
Then $\frac{\Phi_{n,\ell}^\bv(\bu_\ell)(t)}{1+\|\bu_\ell\|_\infty}$  converges uniformly under the $\ell_\infty$-norm. Let  $ \Phi_{*,\ell}^\bv(\bu_\ell)(t)=\lim_{n\to\infty}\Phi_{n,\ell}^\bv(\bu_\ell)(t)$ for $\ell\in[L+1]$.  Since the Lipschitz continuity is preserved under the pointwise convergence, we have  $\Phi_{*,\ell}^\bv$ is $\bR$-Lipschitz continuous in $t$. Let 
\begin{align*}
  \Phi_{*,1}^{\bbeta}(\bTheta)(t)  &= \frac{1}{d} \bX\Phi_{*,1}^{\bv}(\bTheta)(t),  \\
   \Phi_{*,\ell+1}^{\balpha}(\bTheta)(t) &= \int \Phi_{*,\ell+1}^{\bv}(\bTheta, \bbTheta)(t)    ~\th_{1}\left(\Phi^{\bbeta}_{*,\ell}(\bbTheta)(t) \right) d p(\bbTheta),\quad \ell\in[L],\\
   \Phi^{\bbeta}_{*,\ell+1}(\bTheta)(t) &= \Phi^{\bbeta}_{*, \ell}(\bTheta)(t)+ \th_2\left( \Phi^{\balpha}_{*,\ell+1}(\bTheta)(t)\right), \quad \ell\in[L].
\end{align*}
By  the dominated convergence theorem, we have $\Phi_{*,\ell}^{\bbeta}(\bTheta)(t)=\lim_{n\to\infty}\Phi_{n,\ell}^{\bbeta}(\bTheta)(t)$ and   $\Phi_{*,\ell}^{\balpha}(\bTheta)(t)=\lim_{n\to\infty}\Phi_{n,\ell}^{\balpha}(\bTheta)(t)$. Then  $\Phi_*$ is a limit point of $\{\Phi_n:n\ge 0\}$ under $\Dis^{[0,T]}$ and  $\Phi_*\in\mathbf{\Phi}$.
\end{proof}

\begin{proof}[Proof of Lemma~\ref{const Lstar R}]
Analogous to the notation of $\bu_{\ell}$, for the convenience of presenting continuity of $\Phi^\bv_{\ell}$, we introduce notations $\bbfu$ and $\bbfu'_\ell$ by letting
\[
\bbfu_\ell=
\begin{cases}
\bbTheta_1,\\
(\bbTheta_1,\bTheta_2),\\
\bbTheta_1
\end{cases}
\quad 
\bbfu_\ell'=
\begin{cases}
\bbTheta_1,&\ell=1,\\
(\bTheta_1,\bbTheta_2),&\ell\in[2:L],\\
\bbTheta_1&\ell=L+1.
\end{cases}
\]
We also abbreviate the notations for the individual trajectories as:
$$
 \Phi_\ell^{\bv}(\bu_\ell)(t)=v_\ell^t(\bu_\ell),\quad ~~  \Phi_\ell^{\bv}(\bbfu_\ell)(t)=v_\ell^t(\bbfu_\ell),\quad~~   \Phi_\ell^{\bv}(\bbfu_\ell')(t)=v_\ell^t(\bbfu_\ell'),
$$
$$  \Phi_{\ell_1}^{\bbeta}(\bTheta_1)(t)=\bbeta_{\ell_1}^t(\bTheta_1),~\Phi_{\ell_1}^{\bbeta}(\bbTheta_1)(t)=\bbeta_{\ell_1}^t(\bbTheta_1), ~ \Phi_{\ell_2}^{\balpha}(\bTheta_1)(t)=\balpha_{\ell_2}^t(\bTheta_1),~ \Phi_{\ell_2}^{\balpha}(\bbTheta_1)(t)=\balpha_{\ell_2}^t(\bbTheta_1), $$
for $\ell\in[L+1]$, $\ell_1\in[L]$, and $\ell_2\in[2:L]$, respectively.

We first investigate the set $F_2(\bfPhi\cap \bfPhi_{b})$ for a general $b$. We follow similar steps as the proof of Lemma~\ref{lmm:Phi-property}.  We first consider forward steps and prove that there is a constant $R$ such that for any  $\Phi\in\bfPhi\cap \bfPhi_{\beta}$,
    \begin{align}
         \left\|\bbeta_{\ell}^t(\bTheta_1) -\bbeta^t_{\ell}(\bbTheta_1)\right\|_{\infty} \leq R e^{b t}  \left( \|\bTheta_1\|_{\infty}+1 \right)\|\bTheta_1 -\bbTheta_1 \|_{\infty}, \quad &\ell\in[L],\label{asdd1}\\
  \left\|\balpha_{\ell}^t(\bTheta_1) -\balpha^t_{\ell}(\bbTheta_1)\right\|_{\infty} \leq  R e^{b t}\left( \|\bTheta_1\|_{\infty}+1 \right)\|\bTheta_1 -\bbTheta_1 \|_{\infty}, \quad &\ell\in[2:L]. \label{asdd2}
    \end{align}
Then we study the backward steps, and prove that there is  a constant $R'$ such that for any  $\Phi\in\bfPhi\cap \bfPhi_{\beta}$, we have
\begin{align}
&\|\ulineDel_{\ell}^{\bbeta}(\bTheta_1;\Phi,t) -\ulineDel_{\ell}^{\bbeta}(\bbTheta_1; \Phi,t)\|_\infty
\le R' e^{b t}\left( \|\bTheta_1\|_{\infty}+1 \right)\|\bTheta_1 -\bbTheta_1 \|_{\infty} ,& \ell\in[L],\label{asddd3}\\
&\|\ulineDel_{\ell}^{\balpha}(\bTheta_1;\Phi,t) -\ulineDel_{\ell}^{\balpha}(\bbTheta_1; \Phi,t)\|_\infty
\le R' e^{b t}\left( \|\bTheta_1\|_{\infty}+1 \right)\|\bTheta_1 -\bbTheta_1 \|_{\infty} ,& \ell\in[2:L],\label{asddd4}\\
&\|\ulineGrad^{\bv}_{\ell}(\bu_\ell;\Phi,t) - \ulineGrad^{\bv}_{\ell}(\bbfu_{\ell};\Phi,t) \|_\infty
\le R' e^{ b t}\left(1+\|\bu_\ell\|_\infty\right)\|\bu_{\ell}-\bbfu_{\ell}\|_\infty,& \ell\in[L+1], \label{asddd5}\\
&\|\ulineGrad^{\bv}_{\ell}(\bu_\ell;\Phi,t) - \ulineGrad^{\bv}_{\ell}(\bbfu'_{\ell};\Phi,t) \|_\infty
\le R' e^{b t}\left(1+\|\bu_\ell\|_\infty\right)\|\bu_{\ell}-\bbfu'_{\ell}\|_\infty,& \ell\in[L+1]. \label{asddd6}
\end{align}
Note that once we obtain \eqref{asddd5} and \eqref{asddd6},   letting $b_*=R'$,  by the same argument as \eqref{ftheo4} in Lemma \ref{const Lstar}, we  achieve Lemma \ref{const Lstar R} immediately, which finishes the proof. 

In the following, we consider forward steps to prove \eqref{asdd1} and \eqref{asdd2}. For the $1$-st layer, because $\bX$ is bounded, we have \eqref{asdd1} from \eqref{local lip-w1 R}. Suppose at layer $\ell\in[L-1]$, we have that
\[\left\|\bbeta_{\ell}^t(\bTheta_1) -\bbeta^t_{\ell}(\bbTheta_1)\right\|_{\infty} \leq R' e^{b t}  \left( \|\bTheta_1\|_{\infty}+1 \right)\|\bTheta_1 -\bbTheta_1 \|_{\infty}\]
holds for a constant $R'$. Then it follows    that
    \begin{eqnarray}
\left\|\balpha_{\ell+1}^{t}(\bTheta_2) - \balpha_{\ell+1}^{t}(\bbTheta_2)\right\|_\infty
&\leq& \int \underbrace{\left\|\th_1\left(\bbeta_{\ell}^{t}(\bTheta_1)\right)\right\|_{\infty}}_{\leq L_1}~ \underbrace{\left| v_{\ell+1}^t(\bTheta_1, \bTheta_2) -v_{\ell+1}^t(\bTheta_1, \bbTheta_2)\right|}_{(1+C_6)e^{b t}( \|\bTheta_1\|_{\infty} +\|\bTheta_2\|_{\infty}+1) \|\bTheta_2-\bbTheta_2 \|_{\infty}   }    d p(\bTheta_1)  \notag\\
&\leq& R'' e^{b t}( \|\bTheta_2\|_{\infty}+1) \|\bTheta_2-\bbTheta_2 \|_{\infty},\notag
    \end{eqnarray}
for a constant $R''$, where we use sub-gaussianness of $p$ and Corollary~\ref{cor:subg-q}. We further have
\begin{eqnarray}
   \left\|\bbeta_{\ell+1}^{t}(\bTheta_1) - \bbeta_{\ell+1}^{t}(\bbTheta_1)\right\|_\infty&\leq& \left\|\th_2\left(\balpha_{\ell+1}^{t}(\bTheta_1)\right) - \th_2\left(\balpha_{\ell+1}^{t}(\bbTheta_1)\right)\right\|_\infty+ \left\|\bbeta_{\ell}^{t}(\bTheta_1) - \bbeta_{\ell}^{t}(\bbTheta_1)\right\|_\infty\notag\\
   &\leq& R''' e^{b t}( \|\bTheta_1\|_{\infty}+1) \|\bTheta_1-\bbTheta_1 \|_{\infty}. \notag
\end{eqnarray}
We achieve \eqref{asdd1} and \eqref{asdd2}.

We turn to backward process.  We focus on \eqref{asddd3} and \eqref{asddd4}. By the boundedness of $h_1$, $\ulineDel_{L+1}$, $\ulineDel^{\balpha}_{\ell}$ (recall Lemma \ref{lmm:Phi-property}),  \eqref{asddd5} and \eqref{asddd6} can be obtained immediately from \eqref{asdd1} and \eqref{asddd4}.

When $\ell=L$, following the same argument as \eqref{qq+1} in Lemma \ref{const Lstar}, we have \eqref{asddd3} holds.  For each $\ell=L-1, L-2,\cdots,1$, suppose we have
$$\left\|\ulineDel_{\ell+1}^{\bbeta}(\bTheta_1;\Phi,t) -\ulineDel_{\ell+1}^{\bbeta}(\bbTheta_1; \Phi,t)\right\|_\infty
\le R e^{b t}\left( \|\bTheta_1\|_{\infty}+1 \right)\|\bTheta_1 -\bbTheta_1 \|_{\infty}, $$
for a constant $R$. Because of the boundedness of $h_2$ and $\ulineDel^{\bbeta}_{\ell+1}$ (shown in Lemma \ref{lmm:Phi-property}),
we have
\begin{eqnarray}
  && \left\|\ulineDel_{\ell+1}^{\balpha}(\bTheta_1;\Phi,t)-\ulineDel_{\ell+1}^{\balpha}(\bbTheta_1;\Phi,t) \right\|_{\infty}\notag\\
   &=&\left\|\th_2'\left(\balpha^t_{\ell+1}(\bTheta_1)\right)\cdot\ulineDel_{\ell+1}^{\bbeta}(\bTheta_1;\Phi,t)- \th_2'\left(\balpha^t_{\ell+1}(\bbTheta_1)\right)\cdot\ulineDel_{\ell+1}^{\bbeta}(\bbTheta_1;\Phi,t) \right\|_{\infty}\notag\\
  &\leq&  R' e^{b t}( \|\bTheta_1\|_{\infty}+1) \|\bTheta_1-\bbTheta_1 \|_{\infty}, \notag
\end{eqnarray}
for a constant $R'$. Consequently, following the same argument as \eqref{dellater}, we have  
$$   \left\|\ulineDel_{\ell+1}^{\bgamma}(\bTheta_1;\Phi,t)-\ulineDel_{\ell+1}^{\bgamma}(\bbTheta_1;\Phi,t) \right\|_{\infty}
  \leq R'' e^{b t}( \|\bTheta_1\|_{\infty}+1) \|\bTheta_1-\bbTheta_1 \|_{\infty}, $$
for a constant $R''$, which further yields
$$    \left\|\ulineDel_{\ell}^{\bbeta}(\bTheta_1;\Phi,t)-\ulineDel_{\ell}^{\bbeta}(\bbTheta_1;\Phi,t) \right\|_{\infty}\notag\\
  \leq  R''' e^{b t}( \|\bTheta_1\|_{\infty}+1) \|\bTheta_1-\bbTheta_1 \|_{\infty}. $$
 We achieve \eqref{asddd4} and \eqref{asddd5} and thus  obtain Lemma \ref{const Lstar R}.
\end{proof}

\section{Proofs of Theorems \ref{theorm:appres} -- \ref{theo:convex res} }
\subsection{Proof of Theorem \ref{theorm:appres}}\label{secsection:appproof}
The proof follows from Theorem \ref{theorm:app}.  In the proof, we fix $\Phi_*$ and the initialization $\{\bTheta_{i}\}_{i=1}^m$.
Similar to the notation $\bu_\ell$ in the proof of Theorem~\ref{theorm:resflow p0}, we introduce the notations $\bbfu_{\ell,i,j}$ that stands for $\bbTheta_{j}$,  $(\bbTheta_{i},\bbTheta_{i})$, $\bbTheta_{i}$ for $\ell=1$, $2\le\ell\le L$, $\ell=L+1$, respectively. 
% $\bu_\ell$ stands for $\bw_1$, $(\bw_1,\btheta_2)$, $(\btheta_{\ell-1},\btheta_{\ell})$, $\btheta_L$ for $\ell=1$, $\ell=2$, $3\le\ell\le L$, $\ell=L+1$, respectively. 
We also abbreviate the gradients of the ideal process as 
\begin{eqnarray}
\ulineDel_{\ell,i}^{\bbeta, t}:=\ulineDel_{\ell}^{\bbeta}(\lineTheta_i, \Phi_*, t),\quad \ulineDel_{\ell,i}^{\balpha, t}:=\ulineDel_{\ell}^{\balpha}(\lineTheta_i, \Phi_*, t),
\quad \ulineGrad_{\ell,i,j}^{\bv} =\ulineGrad_{\ell}^{\bv}(\bbfu_{\ell,i,j};  \Phi_*, t). \notag
\end{eqnarray}
We use a common notation $\linev_{\ell,i,j}^t$ to the weights at layer $\ell$; for $\ell=1$ let $\linev_{1,i,j}^t=\bbv_{1,j}^t$.
% ; for $\ell=L+1$ let $\linew_{L+1,i,j}^t=w_{L+1,i}^t$ 
To compare the discrete and continuous trajectories on the same time scale, we normalize discrete gradients by 
\[
 \NDel^{\balpha,k}_{\ell,i} = [m]~\Del^{\balpha,k}_{\ell,i}, ~~ \ell\in[2:L],\quad~~~
\NDel^{\bbeta,k}_{\ell,i}= [m]~\Del^{\bbeta,k}_{\ell,i}, ~~ \ell\in[L]
\]
and
\[\NGrad^{k}_{\ell,i,j}  = [m_{\ell-1}m_{\ell}]~ \Grad^k_{\ell,i,j},\quad~\ell\in[L+1].\]

Moreover, recalling the definition of  $\ulineDel^{\bgamma}_{\ell}$ in \eqref{def gamma}, we also introduce
$$\NDel^{\bgamma, k}_{\ell,i} = \frac{1}{m} \sum_{j=1}^{m}\left[\hv_{\ell,i,j}^k~ \NDel^{k,\balpha}_{\ell,i}\right] \cdot \th'_1\left(\hbbeta^k_{\ell-1, i}\right), \quad \ell\in[2:L],  $$
and let
\[\ulineDel_{\ell,i}^{\bgamma, t}:=\ulineDel_{\ell}^{\bgamma}(\lineTheta_i, \Phi_*, t), \quad \ell\in[2:L]. \]

We consider the following events:
\begin{align}
 &\!\!\!\left\| \frac{1}{m}\sum_{i=1}^m \left[\linev_{\ell+1,i,j}^{k\eta}~\th_1\left(\bbbeta^{k\eta}_{\ell,i}\right) \right] - \linealpha_{\ell+1,j}^{k\eta}\right\|_{\infty}\!\!\!\!\leq \left(\|\lineTheta_{j}\|_{\infty}+1\right)\ep_1, ~ \ell\in[L-1],~ k\in[0:K],~j\in[m], \label{event forward res} \\
 &\!\!\!\left\| \frac{1}{m}\sum_{j=1}^m \left[\linev_{\ell,i,j}^{k\eta}~ \ulineDel^{\balpha, k\eta}_{\ell,j} \right]\cdot \th_1'\left(\linebeta_{\ell-1,i}^{k\eta}\right)  - \ulineDel^{\bgamma, k\eta}_{\ell,i}\right\|_{\infty}\!\!\!\! \leq \ep_1, ~\ell\in[2:L],~ k\in[0:K], ~i\in[m], \label{event backward res} \\
&\!\!\!\left\|\bbTheta_{i}\right\|_{\infty} \leq C \sqrt{\log( \frac{m } {\delta})},  \quad  i\in[m],\label{event max res}\\
&\!\!\!\frac{1}{m}\sum_{i=1}^m \| \lineTheta_{i}\|_{\infty}^j\leq C,  
\quad j\in[2]\label{event moment res},
\end{align}
for a constant $C$.
In the proofs of this section, we condition on those events. 
\begin{lemma}
\label{lmm:high-prob res}
The events \eqref{event forward res} -- \eqref{event moment res} happen with probability $1-\delta$.
\end{lemma}

The proof consists of the deviation of the actual discrete trajectory from the ideal trajectory over the iteration $k\in[0:K]$.  For $k=0$, we have the deviation of weights $\|\linev_{\ell,i,j}^{0} - \hv_{\ell,i,j}^{0}\|_\infty$ from the initial conditions in Definition~\ref{Res-Net condition}. 
The induction proceeds as follows. 
In Lemma~\ref{lmm:induction-k res}, we first upper bound the deviation of features using the forward propagation, and then upper bound the deviation of gradients using the backward propagation.
Note that
\begin{equation}
\label{induction v^k}
\left\| \linev_{\ell,i,j}^{(k+1)\eta} - \hv_{\ell,i,j}^{k+1}   \right\|_\infty
\le \|\linev_{\ell,i,j}^{k\eta} - \hv_{\ell,i,j}^{k} \|_\infty 
    + \int_{k\eta}^{(k+1)\eta} \left\| \ulineGrad^{s}_{\ell,i,j} - \NGrad^k_{\ell, i,j}\right \|_\infty ds.
\end{equation}
Combining with the Lipschitz continuity of $\ulineGrad^{t}_{\ell,i,j}$ in Lemma~\ref{lmm:gradw-lip-t res}, we complete the inductive step.

\begin{lemma}
\label{lmm:induction-k res}
Given $k\in[0:K]$ and $\ep<1$. 
Suppose
\begin{equation}
\label{eq:induction weights res}
\left\| \linev_{\ell,i,j}^{k\eta} - \hv_{\ell,i,j}^{k}   \right\|_\infty
\le (\| \bbfu_{\ell,i,j}\|_\infty+1) \ep,\qquad \forall~\ell\in[L+1],~i\in[m_{\ell-1}],~j\in [m_\ell].
\end{equation}
Then there exists a constant $C$ such that
\begin{align}
& \left\| \linebeta_{L+1,1}^{k\eta} - \hbbeta_{L+1,1}^{k} \right\|_{\infty} \!\!\!\leq C\left(\ep+\ep_1\right), \label{eq:induction forward last res}\\
& \left\| \linebeta_{\ell,i}^{k\eta} - \hbbeta_{\ell,i}^{k} \right\|_{\infty} \!\!\!\leq  C\left(  \left\| \lineTheta_{i} \right\|_{\infty}+1 \right)\left(\ep+\ep_1\right),~ \forall~\ell \in[L],~ i\in[m],\label{eq:induction forward bbeta}\\
& \left\| \linealpha_{\ell,i}^{k\eta} - \hbalpha_{\ell,i}^{k} \right\|_{\infty} \!\!\!\leq  C\left(  \left\| \lineTheta_{i} \right\|_{\infty}+1 \right)\left(\ep+\ep_1\right),~ \forall~\ell \in[2:L],~ i\in[m],\label{eq:induction forward alpha}\\
& \left\|\ulineGrad_{\ell,i,j}^{k\eta} -  \NGrad_{\ell,i,j}^k  \right\|_\infty 
\!\!\!\leq  C\left(\| \bbfu_{\ell,i,j}\|_{\infty}+1\right)\left(\ep+\ep_1\right),~ \forall~\ell\in[L+1],~i\in[m_{\ell-1}],~j\in [m_\ell]\label{eq:induction backward res}.
\end{align}
\end{lemma}

\begin{lemma}
\label{lmm:gradw-lip-t res}
There exists a constant $C$ such that, for all $\ell\in[L+1]$, $t_1,t_2\in[0,T]$, and $\bu_\ell$,
\[
\left\|\ulineGrad_{\ell,i,j}^{t_1}-\ulineGrad_{\ell,i,j}^{t_2}\right\|_\infty \le C (\|\bbfu_{\ell,i,j}\|_\infty+1)|t_1-t_2|.
\]

\end{lemma}

\begin{proof}[Proof of Theorem \ref{theorm:appres}]
The proof directly follows from Theorem \ref{theorm:app}.
By Lemma~\ref{lmm:high-prob res}, the events in  \eqref{event forward res} -- \eqref{event moment res} happen with probability $1-\delta$.
Conditioned on those events, we prove by induction on $k\in[0:K]$ that
\begin{equation}
\label{eq:w-diff res}
\left\|\linev_{\ell,i,j}^{k\eta} - \hv_{\ell,i,j}^k  \right\|_\infty 
\leq  \left(\|\bbfu_{\ell,i,j} \|_{\infty}+1\right) e^{  C  k \eta} \ep_1,\quad \forall~\ell\in[L+1],~~i\in[m_{\ell-1}],~j\in [m_\ell],
\end{equation}
for some constant $C$ to be specified. 
The base case $k=0$ follows from Definition~\ref{Res-Net condition}. 
Suppose that \eqref{eq:w-diff res} holds for $k\in[0:K-1]$. 
By Lemmas~\ref{lmm:induction-k res} and \ref{lmm:gradw-lip-t res}, for $s\in[k\eta, (k+1)\eta]$,  
\[
\left\| \ulineGrad^{s}_{\ell,i,j} - \NGrad^k_{\ell, i,j}\right \|_\infty
\le C'\left(\| \bbfu_{\ell,i,j}\|_{\infty}+1\right)\left(e^{Ck\eta}\ep_1+\ep_1+s-k\eta\right).
\]
Applying \eqref{induction v^k} yields that
\begin{align*}
\left\| \linev_{\ell,i,j}^{(k+1)\eta} - \hv_{\ell,i,j}^{k+1}   \right\|_\infty
&\le \left\|\linev_{\ell,i,j}^{k\eta} - \hv_{\ell,i,j}^{k} \right\|_\infty 
    + \int_{k\eta}^{(k+1)\eta} \left\| \ulineGrad^{s}_{\ell,i,j} - \NGrad^k_{\ell, i,j}\right \|_\infty ds\\
&\le \left(\|\bbfu_{\ell,i,j} \|_{\infty}+1\right) 
\left(e^{  C  k \eta} \ep_1 + 2C'e^{Ck\eta}\ep_1 \eta + C'\frac{\eta^2}{2} \right)\\
&\le \left(\|\bbfu_{\ell,i,j} \|_{\infty}+1\right) e^{  C  k \eta} \ep_1(1+C'' \eta),
\end{align*}
for a constant $C''$. 
By letting $C=C''$, we arrive at \eqref{eq:w-diff res} for $k+1$ using $1+C\eta \le e^{C\eta}$.
Note that $k\eta\le T$ for $k\in[0:K]$, $\ep_1\le\tO(1/\sqrt{m})$, and $\|\bbfu_{\ell,i,j} \|_{\infty}\le \cO(\log(m))$ from \eqref{event max res}. The conclusion follows from Lemma~\ref{lmm:induction-k res}.
\end{proof}

\subsection{Proof of Theorem \ref{ini lemma2}}\label{sec:d2}
We first introduce the continuous Res-Net. We   let   $p$  be in \eqref{ppp}.  Under Assumption \ref{ass:8}, we have $\min_{\ell=1}^{L-1}\lambda_{\min}(\bK_{\ell}^{\bbeta}):= \blambda_1>0$ by Lemma \ref{maxgreat}. 
 For  $\bTheta =(\bv_1, \balpha_2, \dots, \balpha_L),~ \bbTheta=(\bbv_1, \bbalpha_2, \dots, \bbalpha_L)\in\supp(p)$,  we define the connecting weights between consecutive layers by 
\begin{eqnarray}
  v_{\ell}(\bTheta, \bbTheta) =  \th_1\left(\bbeta_{\ell-1}\right)^\top \left[\bK_{\ell-1}^{\bbeta}\right]^{-1} \bbalpha_{\ell}, \quad   \ell\in[2:L],
\end{eqnarray} 
 where $\bbeta_{\ell} =  \frac{1}{d}\bX \bv_1 + \sum_{i=2}^{\ell}\th_2(\balpha_{i})$.  The weights at the output layer are initialized as a constant $C_5$.
Then the forward propagation constraint \eqref{forres} is satisfied by the definitions of $\bK_\ell^{\bbeta}$.
The weights also satisfy the conditions in  Assumption \ref{ass:6} since $\left\|\left(\bK_\ell^{\bbeta}\right)^{-1} \right\|_2\le \blambda^{-1}_1 $ and $h_1$ is bounded and Lipschitz continuous.

Next we construct the initialization for ideal discrete Res-Net $(\bbv, \bbalpha,\bbbeta)$  which  follows a similar procedure in  Theorem \ref{ini lemma}.

Let $\linebv_{1,i} := \hbv_{i,1}$ for  $i\in[m]$.
For $\ell\in[L-1]$, define the empirical Gram matrix as
    $$  \hbK_{\ell}^{\bbeta} = \frac{1}{m}\sum_{i=1}^m \th_1\left( \hbbeta_{\ell,i}\right)\th_1^\top\left(\hbbeta_{\ell,i}\right).$$
Let $\linealpha_{\ell+1,j} :=  \left(\bK_\ell^{\bbeta}\right)^{1/2} \left(\hbK_\ell^{\bbeta}\right)^{-1/2}\hbalpha_{\ell+1,j}$ for all $j\in[m]$ when $\hbK_\ell^{\bbeta}$ is invertible, and otherwise let $\linealpha_{\ell+1,j}\iiddistr p_{\ell+1}^{\balpha}$. Here $\linealpha_{\ell+1,j}$ are determined by the outputs of previous layer $\hbbeta_{\ell,i}$ and the connecting weights $\hv_{\ell+1,i,j}$.
Thus they are independent of $\linebv_{1,i}$ and $\linealpha_{2,i},\dots,\linealpha_{\ell,i}$ for $i\in[m]$ given $\{\hbbeta_{\ell,i}\}_{i\in[m]}$.
% ..dependency relations ..
Since $\hv_{\ell+1,i,j}$ are independent Gaussian,  the conditional distribution of $\linealpha_{\ell+1,j}$ given  $\{\hbbeta_{\ell,i}\}_{i\in[m]}$ is $\mathcal{N}\left(\mathbf{0}^N, \sigma_1^2\bK_{\ell}^{\bbeta}\right) = p_{\ell+1}^{\balpha}$. 
Therefore, marginally $\linealpha_{\ell+1,j}\iiddistr p_{\ell+1}^{\balpha}$ and they are independent of $\linebv_{1,i}$ and $\linealpha_{2,i},\dots,\linealpha_{\ell,i}$ for $i\in[m]$. So $\{\bbTheta_{i}\}_{i\in[m]} \iiddistr p$.  

\begin{lemma}
\label{lmm:hK-K res}
Let $\ep_2:=\tO(\ep_1)$ such that $\ep_2\leq \blambda^{-1}_1$. With probability $1-\delta$, for all $\ell\in[L-1]$,
\begin{eqnarray}
& \left\|\hbK_{\ell}^{\bbeta}  - \bK_{\ell}^{\bbeta}\right\|_2
\le \ep_2,
\qquad
&\norm{\linealpha_{\ell+1,i}-\hbalpha_{\ell+1,i} }_2\le \ep_2\left\|\lineTheta_{i}\right\|_2,\notag\\
& \norm{\bbalpha_{\ell+1,i} }_2
\le B_4:= C\sqrt{\log(m/\delta)}, \qquad &\norm{\linebeta_{\ell+1,i}-\hbbeta_{\ell+1,i} }_2\le \ep_2\left\|\lineTheta_{i}\right\|_2.\notag
\end{eqnarray}

% where $\tO(\cdot)$ hides poly-logarithmic factors of $m$ and $\frac{1}{\delta}$.
\end{lemma}

Under Lemma \ref{lmm:hK-K res}, we show that the initial connecting weights are also close to the actual discrete Res-Net as specified by the upper bound of $| \linev_{\ell+1,i, j} -\hv_{\ell+1, i,j}|$ in Definition~\ref{Res-Net condition}. Under Lemma \ref{lmm:hK-K res}, $\hbK_{\ell}^{\bbeta}$ is invertible. Using the same argument as  Lemma~\ref{lmm:hat-w},  we have
\[
\hv_{\ell+1, i, j} = \hbalpha_{\ell+1,j}^\top\left[\hbK^{\bbeta}_{\ell}\right]^{-1}\th_1(\hbbeta_{\ell,i}), \quad \ell\in[L-1], ~i,j\in[m].
\]
By the triangle inequality,
\begin{align*}
&\phantomeq\left| \linev_{\ell+1,i, j} -\hv_{\ell+1, i,j}\right|\\
&=\norm{ \linealpha_{\ell+1,j}^\top\left[\bK^{\bbeta}_{\ell}\right]^{-1}\th_1(\linebeta_{\ell,i}) 
- \hbalpha_{\ell+1,j}^\top\left[\hbK^{\bbeta}_{\ell}\right]^{-1}\th_1(\hbbeta_{\ell,i})}_2\\
&\le \Norm{\bbalpha_{\ell+1,j}}_2 \norm{\left[\bK^{\bbeta}_{\ell}\right]^{-1}}_2 \norm{\th_1(\linebeta_{\ell,i})- \th_1(\hbbeta_{\ell,i})}_2
+ \Norm{\bbalpha_{\ell+1,j}}_2 \norm{\left[\bK^{\bbeta}_{\ell}\right]^{-1} - \left[\hbK^{\bbeta}_{\ell}\right]^{-1}}_2 \norm{\th_1(\hbbeta_{\ell,i})}_2\\
&~~+ \Norm{\bbalpha_{\ell+1,j}- \hbalpha_{\ell+1,j} }_2 \norm{\left[\hbK^{\bbeta}_{\ell}\right]^{-1}}_2 \norm{\th_1(\hbbeta_{\ell,i})}_2.
% &\le 2\sqrt{N}L_1\blambda^{-1}\Norm{\bbtheta_{\ell+1,j}}_2
% + \Norm{\bbtheta_{\ell+1,j}}_2 .. \sqrt{N}L_1 
% + .. .. \sqrt{N}L_1 
\end{align*}
Using a similar argument as Theorem \ref{ini lemma}, one can bound the three terms separately and achieve Theorem \ref{ini lemma2}.

\subsection{Proof of Theorem \ref{theo:convex res}}\label{proof of 10}
Letting $\ddp$  be the probability density function of $p$, we can change of variables as
\begin{eqnarray}
\ttv_{\ell}(\bTheta,\bbTheta ) &=& v_{\ell}\left(\bTheta,\bbTheta \right) \ddp\left(\bTheta\right)\ddp\left(\bbTheta\right), \quad \ell \in[2:L], \notag\\
\ttv_{L+1}(\bTheta ) &=& v_{L+1}\left(\bTheta \right) \ddp\left(\bTheta\right). \notag
\end{eqnarray}
Then we can  rewrite Problem \eqref{problem resnet} as
\begin{align}\label{problem resnet2}
  \minimize_{\left\{\ttv_{\ell}\right\}_{\ell=2}^{L+1}, ~p }  
   \quad&  \frac{1}{N}\sum_{n=1}^N \phi\left( \bbeta_{L+1}(n), y^n \right) +  \sum_{\ell=2}^{L+1}  \lambda^{v}_{\ell} {R}_{R,\ell}^{\ttv}( \ttv_{\ell}, p)+ \lambda^{p} R^{p}_R\left(p\right),\\
%   + R^R\left(\left\{v_{\ell}\right\}_{\ell=2}^{L+1},  ~p\right) \\
\text{s.t.}  \quad  \bbeta_{\ell}\left(\bTheta\right) &= \frac{1}{d}\bX \bv_1+  \sum_{i=2}^{\ell}   \th_2\left(\balpha_{i}\right), ~~\text{for all}~ \bTheta=(\bv_1, \balpha_2, \dots, \balpha_L),~\ell\in[L],\notag\\
~~~~~~~ \ddp(\bTheta)~  \balpha_{\ell} &= \int  \tv_{\ell}\left(\bTheta, \bbTheta\right) \th_1\left(\bbeta_{\ell-1}(\bbTheta)\right)  d \bbTheta,  ~~~~\text{for all}~ \bTheta=(\bv_1, \balpha_2, \dots, \balpha_L),~\ell\in[2:L],\notag\\
\bbeta_{L+1}  &=   \int  \ttv_{L+1}\left(\bTheta\right)  \th_1 \left(\bbeta_{L}\left(\bTheta\right)\right) d \bTheta,\notag 
\end{align}
where 
$$ {R}_{R,\ell}^{\ttv}( \ttv_{\ell}, p) =  \int \frac{\left(\int \left|\ttv_{\ell}( \bTheta,\bbTheta )\right|  d\bbTheta\right)^r }{\left(\ddp(\bTheta)\right)^{r-1}}    d\bTheta,\quad \ell \in[2:L], $$
and 
$$ {R}_{R,L+1}^{\ttv}( \ttv_{L+1}, p) =  \int \frac{ \left|\ttv_{L+1}( \bTheta )\right|^r }{\left(\ddp(\bTheta)\right)^{r-1}}    d\bTheta. $$
Problem \eqref{problem resnet2} only has linear constraints. Moreover, following the same argument in Theorem \ref{theo:con}, it is straightforward to obtain that
${R}_{R,\ell}^{\ttv}$ are convex for $\ell\in[2:L+1]$. We obtain Theorem \ref{theo:convex res}.

\subsection{Proofs of Lemmas}
\begin{proof}[Proof of Lemma~\ref{lmm:high-prob res}]
We prove each of the four events happens with probability $1-\frac{\delta}{4}$ by standard concentration inequalities. 
Both \eqref{event max res} and \eqref{event moment res} happen with probability $1-\frac{\delta}{4}$ by the concentration of sub-gaussian random variables; in particular, \eqref{event max res} follows from Lemma \ref{theo:subgau} and \eqref{event moment res} follows from Lemmas~\ref{azuma} and \ref{azumaexp}.

For \eqref{event forward res} with a given $k,\ell,j,n$, consider random vectors 
\[
\zeta_i := \frac{\linev_{\ell,i,j}^{k\eta} h_1\left(\bbbeta^{k\eta}_{\ell-1,i}(n)\right)}{\left\|\lineTheta_{j}\right\|_{\infty}+1},
\]
which are bounded by a constant $C'$ due to the upper bound of $\linev_\ell$ in Lemma~\ref{lmm:Phi-property}.
Conditioned on $\bbTheta_{j}$, when $i \neq j$,  $\zeta_i$ are independent and $\EE [ \zeta_i | \bbTheta_{j}]= \frac{\linealpha_{j}^{k\eta}(n)}{\|\lineTheta_{j}\|_{\infty}+1}$. 
By Hoeffding's inequality and the union bound, we have 
\[
\left| \frac{1}{m-1}\sum_{i=1,~i\neq j}^m \zeta_i - \frac{\linealpha_{\ell,j}^{k\eta}(n)}{\|\lineTheta_{j}\|_{\infty}+1} \right| 
< \ep_1/2,
\]
with probability $1-\frac{\delta}{4mL(K+1)N}$. On the other hand, when $i=j$, we also have
\[
\frac{1}{m}\left| \zeta_j - \frac{\linealpha_{\ell,j}^{k\eta}(n)}{\|\lineTheta_{j}\|_{\infty}+1} \right| 
\leq \tC'\ep^{2}_1 \leq \ep_1/2,
\]
where we use the upper bound of $\linealpha_{\ell}$ in Lemma~\ref{lmm:Phi-property}. Therefore, applying the union bound again over $k\in[0:K],~\ell\in[L],~j\in[m]$ and $n\in[N]$, we have \eqref{event forward res} with probability $1-\frac{\delta}{4}$.

For \eqref{event backward res} with a given $k,\ell,i,n$, consider the random vectors 
$$\zeta_j'  :=  [\linev_{\ell+1,i,j}^{k\eta}~ \ulineDel_{\ell+1,j}^{\balpha, k\eta}(n)] ~ h'_1\left( \bbeta_{\ell,i}^{k\eta}(n)\right).  $$
Conditioned on $\bTheta_{i}$,  when $i\neq j$,  $\zeta_j'$ are independent and $\EE[\zeta_j'|\bbTheta_{i}]=\ulineDel_{\ell+1,i}^{\bgamma, k\eta}(n)$.
By the boundedness of $h'_1$ and the upper bound of $\ulineDel_{\ell+1}^{\balpha}$ in Lemma~\ref{lmm:Phi-property}, we have  $\zeta_j'$
\[
| \zeta_j'|
\le C' | \linev_{\ell+1,i,j}^{k\eta} |
\le C (1+\|\bTheta_{j} \|_\infty),
\]
and thus $\xi_j'$  is sub-gaussian. 
Applying Lemma~\ref{azuma}, we obtain that
\[
\left| \frac{1}{m-1}\sum_{j=1,~j\neq i }^m \zeta_j' - \ulineDel_{\ell+1,i}^{\bgamma, k\eta}(n) \right| 
< \ep_1/2,
\]
with probability $1-\frac{\delta}{4mL(K+1)N}$. On the other hand, under event \eqref{event max res}, we have 
\[\frac{1}{m} \left|\zeta_i' - \ulineDel_{\ell+1,i}^{\bgamma, k\eta}(n)\right| \leq \tO(\ep_1^2)\leq \ep_1/2
\]
Therefore, applying the union bound again over $k\in[0:K],~\ell\in[L], ~j\in[m]$, and $n\in[N]$, we have \eqref{event backward res} with probability $1-\frac{\delta}{4}$.

\end{proof}

\begin{proof}[Proof of Lemma~\ref{lmm:induction-k res}]
We first consider the forward propagation and prove \eqref{eq:induction forward last res}, \eqref{eq:induction forward bbeta} and \eqref{eq:induction forward alpha}.
For $\ell=1$, since $\bX$ is bounded, 
\begin{eqnarray}
  \left\| \linebeta_{1,i}^{k\eta} - \hbbeta_{1,i}^{k} \right\|_{\infty} \leq C \|\linebv_{1,i}^{k\eta} - \hbv_{1,i}^k \|_{\infty} \leq   C (\|\lineTheta_{i}  \|_{\infty}+1) \ep. \notag
\end{eqnarray}
For $\ell\in[2:L]$,  following the same argument as \eqref{fordisa}, we have 
\begin{eqnarray}
  \left\| \linealpha_{\ell+1,i}^{k\eta} - \hbalpha_{\ell+1,i}^{k} \right\|_{\infty} \leq   C' (\|\lineTheta_{i}  \|_{\infty}+1) \ep, \notag
\end{eqnarray}
which implies that 
   \begin{eqnarray}
       &&\left\| \linebeta_{\ell+1,i}^{k\eta} - \hbbeta_{\ell+1,i}^{k} \right\|_{\infty}\notag\\
       &\leq& \left\| \linebeta_{\ell,i}^{k\eta} - \hbbeta_{\ell,i}^{k} \right\|_{\infty} +\left\| \th_2\left(\linealpha_{\ell+1,i}^{k\eta}\right) - \th_2\left(\hbalpha_{\ell+1,i}^{k}\right) \right\|_{\infty}\notag\\
       &\leq&  C'' (\|\lineTheta_{i}  \|_{\infty}+1) \ep. \notag
    \end{eqnarray}
The output layer $\ell=L+1$ is similar by applying the upper bound of $v_{L+1}$ in Assumption~\ref{ass:6}.

Next we consider the backward propagation and prove \eqref{eq:induction backward res}.
Since $\bX$ is bounded, $h_1$ is bounded and Lipschitz continuous, and $\ulineDel_{\ell}^{\balpha}$ is bounded by Lemma~\ref{lmm:Phi-property}, it suffices to prove that
\begin{equation}
    \begin{split}
       &\left\|\ulineDel_{L+1}(\Phi_*,k\eta) - \Del^{k}_{L+1,1}\right\|_{\infty}\leq C  \ep,\\  
     &\left\|\ulineDel_{\ell,i}^{\bbeta, k\eta} - \NDel^{\bbeta,k}_{\ell,i}\right\|_{\infty}\leq C \left(  1+\left\| \bbTheta_{i} \right\|_{\infty} \right) \ep, \quad \ell\in[L], ~ i\in[m],\\  
    &\left\|\ulineDel_{\ell,i}^{\balpha, k\eta} - \NDel^{\balpha,k}_{\ell,i}\right\|_{\infty}\leq C \left(  1+\left\| \bbTheta_{i} \right\|_{\infty} \right)  \ep, \quad \ell\in[2:L], ~ i\in[m],    
    \end{split}
\end{equation}
for a constant $C$.

At the output layer $\ell=L+1$, since $\phi_1'$ is Lipschitz continuous on the first argument, 
\[
\left\| \ulineDel_{L+1}(\Phi_*,k\eta) -  \NDel_{L+1,1}^k \right\|_{\infty}\leq L_5  \left\| \linebeta_{L+1,1}^{k\eta} - \hbbeta_{L+1,1}^{k} \right\|_{\infty}\leq C'\left(\ep +\ep_1\right),
\]
for a constant $C'$.
At layer $\ell=L$, using the same argument as \eqref{laqqa}, we have 
\[\left\|\ulineDel_{L,i}^{\bbeta, k\eta} - \NDel^{\bbeta,k}_{L,i}\right\|_{\infty}\leq C \left(  1+\left\| \bbTheta_{i} \right\|_{\infty} \right) \ep, \]
for a constant $C$.

For each layer $\ell$ from $L-1$ to $1$, suppose we have 
\[\left\|\ulineDel_{\ell+1,i}^{\bbeta, k\eta} - \NDel^{\bbeta,k}_{\ell+1,i}\right\|_{\infty}\leq C \left(  1+\left\| \bbTheta_{i} \right\|_{\infty} \right) \ep. \]
It follows that 
     \begin{eqnarray}
        &&\left\|\ulineDel_{\ell+1,i}^{\balpha, k\eta} -  \NDel_{\ell+1,i}^{\balpha, k} \right\|_{\infty}\notag\\
        &\leq& \left\|\ulineDel_{\ell+1,i}^{\bbeta, k\eta}\cdot\th_2'\left(  \linealpha^{k\eta}_{\ell+1,i } \right) -  \NDel_{\ell+1,i}^{\bbeta, k}\cdot\th_2'\left(  \hbalpha^{k}_{\ell+1,i } \right) \right\|_{\infty}\notag\\ 
  &\leq& C' \left(  1+\left\| \bbTheta_{i} \right\|_{\infty} \right) \ep,\notag
  \end{eqnarray}
  for a constant $C'$. Using the same argument as \eqref{laqqa1}, we have
\[\left\|\ulineDel_{\ell+1,i}^{\bgamma, k\eta} - \NDel^{\bgamma,k}_{\ell+1,i}\right\|_{\infty}\leq C'' \left(  1+\left\| \bbTheta_{i} \right\|_{\infty} \right) \ep, \]
for a constant $C''$.
We obtain 
\begin{eqnarray}
  &&\left\|\ulineDel_{\ell,i}^{\bbeta, k\eta}-  \NDel_{\ell,i}^{\bbeta,k} \right\|_{\infty}\notag\\
&\leq& \left\|\ulineDel_{\ell+1,i}^{\bgamma, k\eta} -  \NDel_{\ell+1,i}^{\bgamma,k} \right\|_{\infty} +\left\|\ulineDel_{\ell+1,i}^{\bbeta, k\eta} -  \NDel_{\ell+1,i}^{\bbeta,k} \right\|_{\infty}\notag\\
&\leq&  C''' \left(  1+\left\| \bbTheta_{i} \right\|_{\infty} \right) \ep, \notag
\end{eqnarray}
for a constant $C'''$.
We finish the proof. 
\end{proof}

\begin{proof}[Proof of Lemma~\ref{lmm:gradw-lip-t res}]
The proof is similar to the backward steps in Lemma~\ref{lmm:induction-k res}. Recalling  Lemma~\ref{lmm:Phi-property}, we have
\begin{eqnarray}
\left\|\bbeta_\ell^{t_1} -\bbeta_\ell^{t_2} \right\| &\leq&  C\left(\|\bTheta\|_{\infty}+1 \right)|t_1 -t_2 |, \quad \ell\in[L], \notag \\
\left\|\balpha_\ell^{t_1} -\balpha_\ell^{t_2} \right\| &\leq& C\left(\|\bTheta\|_{\infty}+1 \right)|t_1 -t_2 |, \quad \ell\in[2:L]. \notag
\end{eqnarray}
It is sufficient to prove the following:
\begin{eqnarray}
  &\left\|\ulineDel_{L+1}(\Phi_*,t_1) - \ulineDel_{L+1}(\Phi_*,t_2)\right\|_{\infty}\leq  C' ~|t_1 -t_2 |, \label{delt1}\\ 
     &\left\|\ulineDel_{\ell}^{\bbeta}(\bTheta;\Phi_*,t_1) - \ulineDel_{\ell}^{\bbeta}(\bTheta;\Phi_*,t_2)\right\|_{\infty}\leq C'\left(1+ \| \bTheta\|_{\infty}\right)|t_1 -t_2 |, \quad \ell\in[L], \label{delt2} \\ 
    &\left\|\ulineDel_{\ell}^{\balpha}(\bTheta;\Phi_*,t_1) - \ulineDel_{\ell}^{\balpha}(\bTheta;\Phi_*,t_2)\right\|_{\infty}\leq  C' \left(1+ \| \bTheta\|_{\infty}\right) |t_1 -t_2|, \label{delt3} \quad \ell\in[2:L]. 
\end{eqnarray}

At the output layer $\ell=L+1$, same as \eqref{qfrt}, we have 
\begin{eqnarray}
&&\left\|\ulineDel_{L+1}\left(\Phi_*, t_1\right) -\ulineDel_{L+1}\left(\Phi_*, t_2\right)  \right\|_{\infty}\notag\\
&\leq& L_5\left\| \bbeta_{L+1}^{t_1} -\bbeta_{L+1}^{t_2} \right\|_{\infty}\notag\\
&\leq&L_5\left\|  \int v_{L+1}^{t_1}~ \th_1\left(\bbeta_L^{t_1}\right) -v_{L+1}^{t_2}~ \th_1\left(\bbeta_L^{t_2}\right) d p(\bTheta)\right\|_{\infty}\notag.
\end{eqnarray}
By the upper bound and Lipschitz continuity of $v_{L+1}$ in Lemma~\ref{lmm:Phi-property}, we obtain \eqref{delt1}.  At layer $\ell=L$,   we obtain \eqref{delt2} from the upper bounds and the Lipschitz continuity of $\ulineDel_{L+1}$, $ v^t_{L+1}$, and $\th'_1\left(\bbeta_L^t\right)$.

For each layer $\ell$ from $L-1$ to $1$,  suppose we have \eqref{delt2} at layer $\ell+1$, from the upper bounds and the Lipschitz continuity of $\ulineDel_{\ell+1}^{\bbeta}$ and $\th'_2(\bbalpha^t_{\ell+1})$, we have
\begin{equation}\label{fli1}
  \left\| \ulineDel^{\balpha}_{\ell+1}(\bTheta; \Phi_*, t_1) -\ulineDel^{\balpha}_{\ell+1}(\bTheta; \Phi_*, t_2) \right\|_{\infty}\leq  C''\left( 1+\| \bTheta\|_{\infty}\right)|t_1 -t_2|,
\end{equation}
for a constant $C''$.
Using the the same argument as \eqref{btcon}, we have 
\begin{equation}\label{fli2}
    \left\| \ulineDel^{\bgamma}_{\ell+1}(\bTheta; \Phi_*, t_1) -\ulineDel^{\bgamma}_{\ell+1}(\bTheta; \Phi_*, t_2) \right\|_{\infty}\leq  C''' \left( 1+\| \bTheta\|_{\infty}\right)|t_1 -t_2|, 
\end{equation}
for a constant $C'''$. Combining \eqref{fli1} and \eqref{fli2},  we can achieve \eqref{delt2} at $\ell$. 
\end{proof}

\begin{lemma}\label{maxgreat}
For all $\ell\in[L-1]$, we have $\lambda_{\min} \left[\bK^{\bbeta}_{\ell}\right]\geq \blambda_1>0$,
\end{lemma}
\begin{proof}[Proof of Lemma \ref{maxgreat}]
Fix $\ell\in[L-1]$. For $(\bv_1,\balpha_2,\dots, \balpha_L)\in \supp(p)$,  given  $\balpha_2,\dots, \balpha_L$, we have 
\begin{eqnarray}
\left\|\bbeta_{\ell}(\bTheta)  - \frac{1}{d}\bX \bv_1 \right\|_{\infty} = \left\| \sum_{\ell_1=2}^{\ell}\th_2(\balpha_{\ell_1})  \right\|_{\infty} \leq L L_1. 
\end{eqnarray}
Note that conditioned on $\balpha_2,\dots, \balpha_L$,  $\bv_1$ follows $\mathcal{N}\left(0, d\sigma^2_1\mathbf{I}^d\right)$ which is equivalent to the standard Gaussian distribution $\tp_1$. By Assumption \ref{ass:8} with $C_B = L L_1$ and $f_2(\bv_1)\equiv  \sum_{\ell_1=2}^{\ell}\th_2(\balpha_{\ell_1})$,  we have 
$$  \EE \left[   \th_1\left(\bbeta_{\ell}(\bTheta) \right) \th_1^{\top}  \left(\bbeta_{\ell}(\bTheta) \right) \mid  \balpha_2,\dots, \balpha_L  \right] \succeq \blambda_1 \mathbf{I}^{N}.     $$
Taking full expectation, we obtain Lemma \ref{maxgreat}.
\end{proof}

\begin{proof}[Proof of Lemma~\ref{lmm:hK-K res}]
The proof directly follows from  Lemma \ref{lmm:hK-K}. % We focus on the dependency on $N$, $\delta$, and $\ep_2$. $L_1$, $L_2$, $L_3$,$L_4$, $L_5$, and $\sigma$ are treated as constants.
From Theorem \ref{ini lemma2}, $\bbv_{1,i}$, $\bbalpha_{2,i}$, $\dots$, $\bbalpha_{L,i}$ for all $i\in[m]$ are independent. Therefore,   $\linebeta_{\ell,i}\iiddistr p^{\bbeta}_{\ell}$. 
Consider auxiliary matrices
\[
\bbK_\ell^{\bbeta}:= \frac{1}{m}\sum_{i=1}^m \th_1(\linebeta_{\ell,i})\th_1^\top(\linebeta_{\ell,i}),
\]
Since $h_1$ is bounded, by Hoeffding's inequality, with probability $1-\frac{\delta}{3N^2 (L-1)}$,
\[
\left|\bbK_\ell^{\bbeta}(i,j)-\bK_\ell^{\bbeta}(i,j)\right|\le \frac{C}{\sqrt{m}}\log\frac{3N^2 (L-1)}{\delta}.
\]
By the union bound, with probability $1- \delta/3$,
\begin{eqnarray}\label{eq fl1 res}
\max_{\ell\in[L-1]}\norm{\bbK_\ell^{\bbeta}-\bK_\ell^{\bbeta}}_2
\le N \max_{\ell\in[L-1]}\norm{\bbK_\ell^{\bbeta}-\bK_\ell^{\bbeta}}_\infty
\le \ep_3:=\frac{CN}{\sqrt{m}}\log\frac{1}{\delta}.
\end{eqnarray}
The upper bounds of $\Norm{\bbalpha_{\ell+1,i} }_2$ happen with probability $1-\delta/3$ due to the sub-gaussianness of $p$ and Lemma~\ref{theo:subgau}.
We will also use the following upper bound that happen with probability $1-\delta/3$ by the sub-gaussianness of $p$ and Lemma~\ref{azuma}:
\[
\frac{1}{m}\sum_{i=1}^m\Norm{\bbTheta_{i}}_2
\le \beta_1:=C\sqrt{N}\log(e/\delta),\qquad
\]

Next we inductively prove that, for $\ell\in[L-1]$, 
\begin{align}
\norm{\bK_\ell^{\bbeta}-\hbK_\ell^{\bbeta}}_2
&\le (CN^{3/2}\blambda^{-1}_1\beta_1)^{\ell-1}\ep_3,\label{eq:init-bk-induction res}\\
\left\|\hbalpha_{\ell+1,i} - \linealpha_{\ell+1,i}  \right\|_2
&\le (CN^{3/2}\blambda^{-1}_1\beta_1)^{\ell-1}N\blambda^{-1}_1\ep_3 \norm{\lineTheta_{i}}_2.\label{eq:init-theta-induction res}\\
\left\|\hbbeta_{\ell+1,i} - \linebeta_{\ell+1,i}  \right\|_2
&\le (CN^{3/2}\blambda^{-1}_1\beta_1)^{\ell-1}N\blambda^{-1}_1\ep_3 \norm{\lineTheta_{i}}_2.\label{eq:init-theta-induction res2}
\end{align}
% upper bound $\Norm{\bbK_\ell-\hbK_\ell}_2$ and prove the lemma for $\ell\in[L-1]$.
For $\ell=1$, by definition $\bbK_1^{\bbeta}=\hbK_1^{\bbeta}$.
Following \eqref{matrix sqrt} and \eqref{theta-ub-cmp} in Lemma \ref{lmm:hK-K}, we have
\begin{equation}
\label{matrix sqrt res}
\left\|\left[\hbK^{\bbeta}_{1}\right]^{1/2}  - \left[\bK^{\bbeta}_{1}\right]^{1/2}\right\|_2
\le \frac{N}{\sqrt{2\blambda_1}}\norm{\hbK_1^{\bbeta}-\bK_1^{\bbeta}}_2,
\end{equation}
and
\begin{align}
\left\|\hbalpha_{2,i} - \linealpha_{2,i}  \right\|_2
& = \left\|\left( \left[\hbK^{\bbeta}_{1}\right]^{1/2}\left[\bK^{\bbeta}_{1}\right]^{-1/2} - \mathbf{I}^N \right)     \linealpha_{2,i} \right\|_2\le \frac{N\blambda^{-1}_1\ep_3}{\sqrt{2}}\Norm{\linealpha_{2,i}}_2\le \frac{N\blambda^{-1}_1\ep_3}{\sqrt{2}}\Norm{\lineTheta_{i}}_2.\label{theta-ub-cmp res}
\end{align}
Then  by the Lipschitz continuity of $h_2$, we  have 
\begin{eqnarray}
\left\|\hbbeta_{2,i} - \linebeta_{2,i}  \right\|_2= \left\|\left[\th_2(\hbalpha_{2,i}) - \th_2(\linealpha_{2,i})\right]  \right\|_2\leq  \frac{C' N\blambda^{-1}_1\ep_3}{\sqrt{2}}\Norm{\lineTheta_{i}}_2.
\end{eqnarray}

For $\ell\in[2:L-1]$, suppose we have
\[ \left\|\hbbeta_{\ell,i} - \linebeta_{\ell,i}  \right\|_2
\le (CN^{3/2}\blambda^{-1}_1\beta_1)^{\ell-2}N\blambda^{-1}_1\ep_3 \norm{\lineTheta_{i}}_2.\]
By  the boundedness and Lipschitz continuity of $h_1$, we have
\begin{equation}\label{eq fl2 res}
    \norm{\hbK_\ell^{\bbeta}-\bbK_\ell^{\bbeta}}_2
\le \frac{C''N^{1/2}}{m} \sum_{i=1}^m \norm{\th_1(\linebeta_{\ell,i}) - \th_1(\hbbeta_{\ell,i})}_2
\le (C''' N^{3/2}\blambda_1^{-1}\beta_1)^{\ell-1}\ep_3.
\end{equation}
Then we obtain \eqref{eq:init-bk-induction res} by triangle inequality from \eqref{eq fl1 res} and \eqref{eq fl2 res}. 
The upper bound in \eqref{eq:init-theta-induction res} for $\ell+1$ follows from a similar argument of \eqref{matrix sqrt res} and \eqref{theta-ub-cmp res}. Then \eqref{eq:init-theta-induction res2} for $\ell+1$ follows from 
\begin{eqnarray}
\left\|\hbbeta_{\ell+1,i} - \linebeta_{\ell+1,i}  \right\|_2= \left\|\sum_{j=2}^{\ell+1}\left[\th_2(\hbalpha_{j,i}) - \th_2(\linealpha_{j,i})\right]  \right\|_2\leq C' (CN^{3/2}\blambda^{-1}_1\beta_1)^{\ell-1}N\blambda^{-1}_1\ep_3 \norm{\lineTheta_{i}}_2. \notag
\end{eqnarray}
We finish the proof. 
\end{proof}

\section{Proofs of Theorem  \ref{global conver} and Lemma \ref{lemma:uni}}

\subsection{Proof of Theorem \ref{global conver}}\label{sec:10}
In the proof we use the following abbreviated notations: for $t\in[0,\infty)$ and $\bTheta\in \supp(p)$, let
\begin{eqnarray}
 \bbeta^t_{\ell}(\bTheta) &=&  \Phi_{*,\ell}^{\bbeta}(\bTheta)(t), \quad \ell\in[L],\notag\\
 \balpha^t_{\ell}(\bTheta) &=&  \Phi_{*,\ell}^{\balpha}(\bTheta)(t),\notag \quad\ell\in[2:L],\notag\\
\bv_1^t(\bTheta) &=& \Phi_{*,1}^{\bv}(\bTheta)(t),\notag\\
 v_{L+1}^t(\bTheta) &=& \Phi_{*,L+1}^{\bv}(\bTheta)(t).\notag
\end{eqnarray}
From the convergence assumptions, it is clear that $\bbeta_{L+1}^{t}$ converges as $t\to\infty$. Indeed, the convergence assumptions imply that, for any $\ep_2>0$, there exists $T$, for any $t\ge T$,  
\begin{equation}\label{converge A1}
  \left\|\bbeta_{L}^t(\bTheta) -   \bbeta_{L}^{\infty}(\bTheta)\right\|_{\infty} \leq \ep_2  
\end{equation}
holds $p$-almost surely and
\begin{equation}
  \int \left|v_{L+1}^t(\bTheta) -v_{L+1}^{\infty}(\bTheta)\right| p(\bTheta) \leq \ep_2 \label{converge A2}.  
\end{equation}
Then, since $h_1$ is bounded and Lipschitz continuous, we have 
\begin{eqnarray}
 &&\left\| \bbeta_{L+1}^{t}- \bbeta_{L+1}^{\infty}\right\|_{\infty}\notag \\
 &=&\left\|\int v_{L+1}^{t}(\bTheta) ~ \th_1\left(\bbeta_{L}^{t}(\bTheta)\right) - v_{L+1}^{\infty}(\bTheta)~  \th_1\left(\bbeta_{L}^{\infty}(\bTheta)\right) d p(\bTheta)\right\|_{\infty} \notag\\
 &\leq& \int\left|v_{L+1}^{\infty}(\bTheta)\right|\left\|\th_1\left(\bbeta_{L}^{t}(\bTheta)\right) -\th_1\left(\bbeta_{L}^{\infty}(\bTheta)\right) \right\|_{\infty}dp(\bTheta)\notag\\
 &&+  \int\left|  v_{L+1}^{t}(\bTheta) -v_{L+1}^{\infty}(\bTheta)\right| \left\|\th_1(\bbeta_{L}^{t}(\bTheta))\right\|_{\infty}d p(\bTheta)\notag\\
 &\leq & C\ep_2. \label{converge A3}
\end{eqnarray}
The goal of the proof is to show that
\begin{equation}
\label{eq:zero grad}
\left\|\tphi'_1\left(\bbeta_{L+1}^{\infty}\right)\right\|_2=0.    
\end{equation}
To this end, for any $\ep>0$, we will construct a function
\begin{equation}
\label{eq:fbv1}
f_\ep(\bv_1) :=\tphi'_1\left(\bbeta_{L+1}^{\infty}\right)^\top \th_1\left(\frac{1}{d} \bX\bv_1 + g_{\ep}(\bv_1)\right),
\end{equation}
where the functions $g_{\ep}$ is uniformly bounded, such that $|f_\ep(\bv_1)|< \ep$.
% Then, since $f$ is bounded, we have $\EE_{\tp_1}[|f(\bv_1)|]\le\cO(\ep+\eta)$.
Then it follows from \eqref{eq:fbv1} that
\[
\tphi'_1\left(\bbeta_{L+1}^{\infty}\right) =  \bK^{-1} \int   f_\ep(\bv_1) ~ \th_1\left( \frac{1}{d}\bX\bv_1 + g_{\ep}(\bv_1)\right) d \tp_1(\bv_1),
\]
where $\tp_1=\mathcal{N}(\mathbf{0}^d,\mathbf{I}^d )$ and $\bK:= \int  \th_1\left(\frac{1}{d} \bX\bv_1 + g_{\ep, \eta}(\bv_1)\right) \th_1^{\top}\left(\frac{1}{d} \bX\bv_1 + g_{\ep, \eta}(\bv_1)\right) d \tp_1(\bv_1)$
whose minimum eigenvalue is at least $\blambda_1 >0$ by Assumption \ref{ass:8}.
The boundedness of $h_1$ yields that
\[
\left\|\tphi'_1\left(\bbeta_{L+1}^{\infty}\right)\right\|_2
\le C\blambda_1^{-1}\ep.
\]
Since $\blambda_1$ is independent of $\ep$, by letting $\ep\to 0$, we obtain \eqref{eq:zero grad}.

Next we construct $g_{\ep}$ and $f$ in \eqref{eq:fbv1}. 
Let $T$ be the time such that \eqref{converge A1} and \eqref{converge A2} hold with $\ep_2\le c\ep$ for a constant $c$ to be specified.
Note that $\bv_1^T$ is surjective by Lemma~\ref{lmm:surjection}. 
Let $\tilde g:\RR^d\to \supp(p)$ be the inverse function such that $\bv_1^T(\tilde g(\bv_1))=\bv_1$.  
% $ \Phi_*^{-1,t}(\bv_1) = \bTheta(\hbv_1)$ where $\hbv_1$  satisfies $\bTheta(\hbv_1)\in \supp(p)$ and $f_3(\hbv_1)= \bv_1$.
Define
\[
g_{\ep}(\bv_1)  = \sum_{\ell=2}^{L} \th_2\left(\balpha^{T}_{\ell}\left( \tilde g(\bv_1)\right)  \right),
\qquad
f_{\ep}(\bv_1)=\tphi'_1\left(\bbeta_{L+1}^{\infty}\right)^\top \th_1\left(\bbeta_{L}^{T}( \tilde g(\bv_1) )\right),
\]
where $g_{\ep}$ is uniformly bounded by the boundedness of $h_2$.
Suppose on the contrary that there exists $\bv_1'$ such that $|f_{\ep}(\bv_1')|>\ep$.
Let $\bTheta'=\tilde g(\bv_1')$.
Since $\bTheta\mapsto\tphi'_1\left(\bbeta_{L+1}^{\infty}\right)^\top \th_1\left(\bbeta_{L}^{T}( \bTheta )\right)$ is continuous by Theorem \ref{theorm:phi}, there exists a ball around $\bTheta'$ denoted by $S$ such that $p(S)>0$ and $|\tphi'_1\left(\bbeta_{L+1}^{\infty}\right)^\top \th_1\left(\bbeta_{L}^{T}( \bTheta )\right)|>\ep/2$ with the same sign for all $\bTheta\in S$.
However, for $t>T$, 
\begin{eqnarray}
&&\int\left| v_{L+1}^{t}(\bTheta) - v_{L+1}^{T}(\bTheta) \right|dp(\bTheta)\notag\\
&\geq & \frac{1}{N}\int \mathcal{I}_{S} \left|  \int_{T}^{t} \tphi'_1\left(\bbeta^{t}_{L+1}\right)^\top \th_1\left(\bbeta_{L}^t(\bTheta) \right) dt \right|dp(\bTheta)\notag\\
&\geq&  \frac{1}{N}\int  \mathcal{I}_{S} \left( 
\left| \int_{T}^{t} \tphi'_1\left(\bbeta^{\infty}_{L+1}\right)^\top \th_1\left(\bbeta_{L}^{\infty}(\bTheta) \right)dt \right| 
- \int_{T}^{t} C\ep_2 dt \right) dp(\bTheta),\label{grow-t-diverge}
\end{eqnarray}
where in the last step we used \eqref{converge A1}, \eqref{converge A3}, and the boundedness and Lipschitz continuity of $\phi_1'$ and $h_1$.
Let $c=\frac{1}{4C}$. The lower bound in \eqref{grow-t-diverge} diverges with $t$, which contradicts \eqref{converge A2}.

Finally from \eqref{eq:zero grad} we show the convergence statement. 
Since $\phi$ is convex on the first argument, we obtain 
$$ \sum_{n=1}^N\phi\left(\bbeta_{L+1}^{\infty}(n),y^n \right)  = \sum_{n=1}^N\left[\min_{y'}\phi \left(y', y^{n}\right) \right]. $$
Since $\bbeta^t_{L+1} \to  \bbeta^{\infty}_{L+1}$ and $\phi$ is continuous, we obtain that
$$\lim_{t\to\infty}L^t = \sum_{n=1}^N\phi(\bbeta_{L+1}^{\infty}(n),y^n ) = \sum_{n=1}^N\left[\min_{y'}\phi \left(y', y^{n}\right) \right],$$
which completes the proof.

\begin{lemma}
\label{lmm:surjection}
The function $t<\infty$, $\bv_1^t:\supp(p)\to \RR^{d}$ is a surjection. 
\end{lemma}
\begin{proof}
Recall that at the initialization we have $\bTheta(\bv) = (\bv, f_1(\bv) )\in \RR^D$ for a continuous function $f_1$ by Assumption \ref{ass:7}. 
Given $t<\infty$, consider $f_t:\RR^d \to \RR^d$ as
$$ f_t(\bv)  = \bv_1^t\left(\bTheta(\bv)\right).  $$
It suffices to show that $f_t$ is surjective. 
Note that $f_t$ is continuous since $\bTheta\mapsto \bv_1^t(\bTheta)$ is continuous by Theorem \ref{theorm:phi}.
Furthermore, for any $\bv\in\RR^d$, by Lemma~\ref{lmm:Phi-property},
$$   
\left\|f_t(\bv) - \bv\right\|_{\infty} 
= \int_{0}^t \left\|\ulineGrad_1^{\bv}\left([\bv;f_1(\bv)], \Phi_*,s   \right)\right\|_\infty ds  \leq Ct.   
$$
For any $\bx\in\RR^d$, consider $g(\bv):=\bx - ( f_t(\bv)  - \bv )$ which continuously maps $\mathcal{B}_\infty(\bx,Ct)$ to itself.
By the Brouwer's fixed-point theorem (see, e.g. \cite{granas2013fixed}), there exists $\bv_*\in \mathcal{B}_\infty(\bx,Ct)$  such that $g(\bv_*) = \bv_*$; equivalently, we have $f_t(\bv_*)=\bx$.
\end{proof}

\subsection{Proof of Lemma \ref{lemma:uni}}
We first note the following results in \cite[Lemma F.1]{du2018gradient}: suppose $C_B=0$, the support of a random vector $\bfV\in\RR^d$ denoted by $R$ has positive Lebesgue measure, and $h$ is an analytic non-polynomial function on $R$. Then
\[
\min_{\Norm{\bfa}_2=1}\EE\norm{\sum_{i=1}^N a_i h(\bx_i\cdot \bfV)}_2^2 =\lambda>0,
\]
where $\bfa=(a_1,\dots,a_N)$.
Lemma~\ref{lemma:uni} shows that, for $\bfV'\sim p= \mathcal{N}\left(\mathbf{0}^d, \mathbf{I}^d\right)$, the same result holds with a constant perturbation of the functions $h_1$; namely, by letting $g_i(\bv) = h_1(\bx_i\cdot \bv+ C_i(\bv))$ where $\Norm{C_i}_\infty\le C_B$,
\begin{equation}
\label{eq:lambda min goal}
\min_{\Norm{\bfa}_2=1}\EE\norm{\sum_{i=1}^N a_i g_i(\bfV')}_2^2 =\lambda'>0,
\end{equation}
where $\lambda'$ is uniform over all perturbations $\Norm{C_i}_\infty\le C_B$. 
It suffices to prove \eqref{eq:lambda min goal} for $\bfV'\sim q=\mathrm{Uniform}(R')$ where $R'$ is determined by $h$ and $C_B$, as the Radon–Nikodym derivative $\frac{dq}{dp}$ is bounded.

We first prove (\ref{lemmai}).
Consider a compact region $R$ such that, for ${\bfV}\sim\mathrm{Uniform}(R)$ and any unit vector $\bfa$, 
\[
\EE\norm{\sum_{i=1}^N a_i h_3(\bx_i\cdot \bfV)}_2^2 \ge \lambda_R>0.
\]
Then for any $\alpha>0$, since $h_3$ is bounded and Lipschitz continuous, we have
\[
\EE\norm{\sum_{i=1}^N a_i h_3\left(\bx_i\cdot \bfV + \alpha C_i(\bfV/\alpha)\right)}_2^2 \ge \lambda_R - C\alpha C_B\ge \frac{\lambda_R}{2},
\]
when $\alpha\le \frac{\lambda_R}{2C_BC}$.
Equivalently, $\EE\Norm{\sum_{i=1}^N a_i h_1\left(\bx_i\cdot \bfV/\alpha + C_i(\bfV/\alpha)\right)}_2^2 
\ge \frac{\lambda_R}{2}.$
% \[
% \EE\norm{\sum_{i=1}^N a_i h_1\left(\bx_i\cdot \bfV/\alpha + C_i(\bfV/\alpha)\right)}_2^2 
% \ge \frac{\lambda_R}{2}.
% \]
We achieve \eqref{eq:lambda min goal} by letting $\bfV'=\bfV/\alpha$.

For (\ref{lemmaii}), 
consider $R=\{\bv: 1/2\le \Norm{\bv}_2\le 1\}$. 
% $R=\mathcal{B}_2(0,1)$. 
Then, for ${\bfV}\sim\mathrm{Uniform}(R)$ and any unit vector $\bfa$, 
\[
\EE\norm{\sum_{i=1}^N a_i h_1(\bx_i\cdot \bfV)}_2^2 \ge \lambda_R>0.
\]
Since $h_1(\beta x)=\beta^\alpha x$ for any $\beta>0$, then we have $\EE\Norm{\sum_{i=1}^N a_i h_1(\bx_i\cdot \beta\bfV)}_2^2 \ge \beta^{2\alpha}\lambda_R$.
Note that $|\bx_i\cdot \beta\bfV|= \Theta(\beta)$.
For $x= \Theta(\beta)$, we have $|h_1(x)|\le C\beta^\alpha$ and $h_1$ is $C\beta^{\alpha-1}$-Lipschitz continuous for a constant $C$.
Therefore,
\[
\EE\norm{\sum_{i=1}^N a_i h_1(\bx_i\cdot \beta\bfV + C_i(\beta\bfV) )}_2^2 
\ge \beta^{2\alpha}\lambda_R - C'\beta^{2\alpha-1}C_B
\ge (C'C_B)^{2\alpha} \left( \frac{2}{\lambda_R}\right)^{2\alpha-1}.
\]
for a constant $C'$ when $\beta= \frac{2C'C_B}{\lambda_R}$.
We achieve \eqref{eq:lambda min goal} by letting $\bfV'=\beta\bfV$.

For (\ref{lemmaiii}), we first show that there exists a compact set $R$ such that, for all $\bv\in R$ and $\bx_i$,
\begin{equation}
\label{eq:R}
|\bx^i\cdot\bv| \ge c'.
\end{equation}
This can be done by a simple probabilistic argument.
Let $\bv$ be drawn from the uniform distribution on $S^{d-1}$, for any fixed $\bx\in\RR^d$, we have
\begin{align*}
P \{ |\bv^\top \bx | < t\Norm{\bx}_2 \} 
= \frac{2\pi^{\frac{d-1}{2}}/\Gamma(\frac{d-1}{2})}{2\pi^{\frac{d}{2}}/\Gamma(\frac{d}{2})}  \int_{-t}^t (1 - u^2 )^{\frac{d-3}{2}} \ du <  t\sqrt{d}.
\end{align*}
By a union bound, we have $|\bx^i\cdot\bv|\ge \frac{\Norm{\bx^i}_2}{2N\sqrt{d}}$ with probability 0.5.
Denote the set of $\bv\in S^{d-1}$ by $S'$.
Since $\min_{i}\Norm{\bx^i}_2:= C_x >0$, we obtain \eqref{eq:R} with $R=\{t \bv: \bv\in S', \frac{2c'N\sqrt{d}}{C_x}\le t \le \frac{4c'N\sqrt{d}}{C_x}\}$.
Then, for ${\bfV}\sim\mathrm{Uniform}(R)$ and any unit vector $\bfa$, 
\[
\EE\norm{\sum_{i=1}^N a_i h_1(\bx_i\cdot \bfV)}_2^2 \ge \lambda_R>0.
\]
Then, for any $\beta>0$, we have $\EE\Norm{\sum_{i=1}^N a_i h_1(\bx_i\cdot \beta\bfV)}_2^2 \ge \beta^{-2\alpha}\lambda_R$.
For $x=\Theta(\beta)$ we have $|h_1(x)|\le C\beta^{-\alpha}$ and $h_1$ is $C\beta^{-\alpha-1}$-Lipschitz continuous for a constant $C$. 
Therefore, 
\[
\EE\norm{\sum_{i=1}^N a_i h_1(\bx_i\cdot \beta\bfV + C_i(\beta\bfV) )}_2^2 
\ge \beta^{-2\alpha}\lambda_R - C'\beta^{-2\alpha-1}C_B
\ge (C'C_B)^{-2\alpha} \left( \frac{2}{\lambda_R}\right)^{-2\alpha-1}.
\]
for a constant $C'$ when $\beta= \frac{2C'C_B}{\lambda_R}$.
We achieve \eqref{eq:lambda min goal} by letting $\bfV'=\beta\bfV$.

\section{Additional Results}
\subsection{Proof of Proposition \ref{prop:degenerate}}
Explicitly shown in \cite[Corollay 25]{nguyen2020rigorous},  in the mean-field limit that $m\to \infty$, the weights remain mutually independent and follow a common distribution  that only depends on time $t$ in the intermediate layers.  Therefore, by the law of large numbers,  the features are the same. We have Proposition \ref{prop:degenerate}.

\subsection{Discussion on the Global Convergence under Proper Regularizers}\label{subsection:non-ri}
As an example, we consider a $3$-layer NN and  $\ell_{1,2}$-norm regularizer.  We can reformulated  an $3$-layer NN learning problem  as:
\begin{eqnarray}
\minimize && \bL_2 ( p_1, p_2, p_3, \tw_2, \tw_3, \tu)\label{problem 3}\\
&& =\sum_{i=1}^N\phi\left( \int h(\btheta_3({i})) \tu(\btheta_3) d\btheta_3, y^{i}\right) + \lambda_4^w \int \frac{|\tu(\btheta_3)|^2}{ \ddp_3(\btheta_3)} d\btheta_3 \notag\\
&& +\lambda_3^w \int  \frac{\left(\int |\tw_3(\btheta_2, \btheta_3)| d \btheta_3\right)^2}{ \ddp_2(\btheta_2)} d\btheta_2+\lambda_2^w \int \frac{ \left(\int |\tw_2(\bw_1, \btheta_2)| d \btheta_2\right)^2}{ \ddp_1(\bw_1)} d\bw_1 \notag\\&& +\lambda_1^w \left(\int |w_1(\theta_0, \bw_1)|\ddp_1(\bw_1) d\bw_1\right)^2 dp_0(\theta_0) +
f_p^1(p_1)+ f_p^2(p_2)+  f_p^3(p_3),\notag\notag\\
\text{s.t.}&& \int h\left(  (\bx^{i})^\top\bw_1 \right) \tw_2(\bw_1, \btheta_2)  d\bw_1 = \btheta_2(i) \ddp_2(\btheta_2),   \quad i \in [N],  \quad \btheta_2\in \RR^N, \label{oo1}\\
&&\int h(\btheta_2(i)) \tw_3(\btheta_2, \btheta_3) d\btheta_2 = \btheta_3(i) \ddp_3(\btheta_3), \quad i \in [N], \quad \btheta_3\in \RR^N,\label{oo2}\\
&&   \int \ddp_1(\bw_1) d\bw_1 = 1, \quad   \int \ddp_2(\btheta_2) d\btheta_2 = 1, \quad  \int \ddp_3(\btheta_3) d\btheta_3 = 1,\label{ooo}\\
&& \ddp_1(\bw_1)\geq 0, \quad   \ddp_2(\btheta_2)\geq 0, \quad   \ddp_3(\btheta_3)\geq 0, \notag
\end{eqnarray}
where $w_1(\theta_0, \bw_1)$ denote the $\theta_0$-th coordinate of $\bw_1$ and $p_0(\theta_0)$ is the  uniform distribution on $[d]$, and $f_{\ell}^p = \lambda^r_1 R^p_{\ell}$ with $\ell\in[3]$.

%We assume that $\{\lambda_\ell\}_{1}^4$ are non-negative, and  $\{f^p_{\ell}\}_{\ell=1}^3$ are convex that guarantee that $\{p_{\ell}\}_{\ell=1}^3$ are equivalent to Lebesgue measure. 

Given a continuous $3$-layer DNN  $(p_1, p_2, p_3, w_2, w_3, u)$,  by backward-propagation, the functional gradient with respect to the weights can be obtained as:
\begin{subequations}
\begin{align}
  g_4(\btheta_3) &= \sum_{i=1}^{N}[ \phi'_i h(\btheta_3({i}))]+2\lambda_4 u(\btheta_3), \quad \btheta_3\in\supp(p_3),\label{ggf1}\\
g_3(\btheta_2, \btheta_3) &= \sum_{i=1}^{N}\underbrace{[ \phi'_i h'(\btheta_3(i)) u_3(\btheta_3)+ \nabla_{\btheta_3,i}f^{p'}_{3}(p_3)  ]}_{:=~q_3^i(\btheta_3)} h(\btheta_2(i))+2\lambda_3 R_2(\btheta_2) \left|w_3(\btheta_2, \btheta_3)\right|', \label{ggf2} \\
g_2(\bw_1, \btheta_2) &= \sum_{i=1}^{N}\underbrace{\left[ \int q_3^i(\btheta_3)h'(\btheta_2,i)w_3(\btheta_2, \btheta_3) dp_3(\btheta_3)+ \nabla_{\btheta_2,i}f_2^{p'}(p_2)  \right]}_{:=~q_2^i(\btheta_2)} h\left(\bw_1^\top \bx^{i}\right)+2\lambda_2 R_1(\bw_1) \left|w_2(\bw_1, \btheta_2)\right|', \label{ggf3}\\
g_1(\theta_0, \bw_1) &= \sum_{i=1}^{N}\left[ \int q_2^i(\btheta_2)h'(\bw_1^\top x^i)w_2(\bw_1, \btheta_2) p_2(\btheta_2) d\btheta_2 \bx^i(\theta_0)+ \lambda_r\nabla_{\bw_1, \theta_0}f_1^{p'}(p_1)  \right] +2\lambda_0 R_0(\theta_0) |w_1(\theta_0, \bw_1)|',\label{ggf4} 
\end{align}
\end{subequations}
where 
\begin{eqnarray}
\phi'_i &=& \phi'\left( \int h(\btheta_3(i))u(\btheta_3)dp_3(\btheta_3), y^i\right), \quad i\in[N]\notag\\
R_2(\btheta_2) &=& \int |w_3(\btheta_2, \btheta_{3})| d p_{3}(\btheta_{3}),\quad \btheta_2\in\supp(p_2),\notag\\
R_1(\bw_1) &=& \int |w_2(\bw_1, \btheta_{2})|dp_{2}(\btheta_2),\quad \bw_1\in\supp(p_1)\notag\\ 
R_0(\theta_0) &=& \int |w_1(\theta_0, \bw_1)| dp_{1}(\bw_1),\quad \theta_0\in\supp(p_0),\notag
\end{eqnarray}
in \eqref{ggf3}, $\btheta_2\in\supp(p_2)$ and $\btheta_3\in\supp(p_3)$; in \eqref{ggf2},   $\bw_1\in\supp(p_1)$ and $\btheta_2\in\supp(p_2)$; and in \eqref{ggf1},  $\theta_0\in\supp(p_0)$, and $\bw_1\in\supp(p_1)$; $\nabla_{\btheta_3,i} f^{p'}_3$ denote the gradient of $f^{p'}_3$ on coordinate $\btheta_3(i)$.

If the regularizers $\{f^p_{\ell}\}_{\ell=1}^3$ guarantee that $\{p_{\ell}\}_{\ell=1}^3$ have a full support, we can expect that  $g_j= 0$ holds almost surely  for  $j\in[4]$ after training, which guarantees the global convergence by the following property:

\begin{property}\label{prof}
Assume that  $\lambda_\ell\geq0$ for  $\ell\in[4]$, and  $\{f^p_{\ell}\}_{\ell=1}^3$ are convex and guarantee that $\{p_{\ell}\}$ are equivalent to Lebesgue measure.
If there is $(p_1, p_2, p_3, w_2, w_3, u)$ such that   $g_j= 0$ holds for $j\in[4]$ a.e. and $\nabla u$, $\nabla w_3$, $\nabla w_2$ exist a.e., then $(p_1, p_2, p_3, w_2, w_3, u)$ achieves the global minimal solution.
\end{property}

\begin{proof}[Proof of Property \ref{prof}]
\begin{enumerate}[(A)]
    \item  We first write down the KKT condition of Problem \eqref{problem 3}. Concretely, let $\Lambda^i_2(\btheta_2)$ and $\Lambda^i_3(\btheta_3)$ be the multipliers for the constraints \eqref{oo1} and \eqref{oo2}, respectively.  Let $\lambda_p^i$ is the multiplier for the constraint \eqref{ooo}. 
    $f_i^p$  guarantees that $\ddp_i>0$ for $i\in[3]$.  
    \begin{enumerate}[(1)]
        \item Considering $\tu(\btheta_3)$,  for $\btheta_3$ 
        \begin{eqnarray}
        \sum_{i=1}^N[\phi'_ih(\btheta_3(i))] + 2 \lambda_4 u(\btheta_3) =0, \quad \text{a.e.} ;\label{tu}
        \end{eqnarray}
        \item Considering  $\tw_3(\btheta_2, \btheta_3)$,   for $\btheta_2$ and $\btheta_3$,
        \begin{eqnarray}
2\lambda_3R_2(\btheta_2)\left| w_3(\btheta_2, \btheta_3)\ddp_3(\btheta_3) \right|' +\sum_{i=1}^N \Lambda_3^i(\btheta_3)h(\btheta_2(i)) = 0, \quad     \text{a.e.}; \label{tw23}
\end{eqnarray}
\item Considering  $\ddp_3 (\btheta_3)$,  for $\btheta_3$,
\begin{eqnarray}
 - \lambda_4|u(\btheta_3) |^2 - \sum_{i=1}^N\Lambda_3^i(\btheta_3) \btheta_3(i) +  f^{p'}_{3}(p_3) = \lambda^3_p,  \quad \text{a.e.};  \label{p3}
\end{eqnarray}
\item Considering  $\tw_2(\bw_1, \btheta_2)$, for $\bw_1$ and $\btheta_2$,
\begin{eqnarray}
&&2\lambda_2R_1(\bw_1)\left| w_2(\bw_1, \btheta_2)\ddp_2(\btheta_2) \right|' + \sum_{i=1}^N\Lambda_2^i(\btheta_2)h\left(\bw_1^\top \bx^i\right) = 0,     \quad \text{a.e.} ; \label{tw12}
\end{eqnarray}
\item Considering  $\ddp_2 (\btheta_2)$,  for  $\btheta_2$, 
\begin{eqnarray}
- \lambda_3(R_2(\btheta_2))^2 - \sum_{i=1}^N\Lambda^i_2(\btheta_2) \btheta_2(i) +  f_2^{p'}(p_2) =\lambda^2_p, \quad \text{a.e.};\label{p2}
\end{eqnarray}
\item Considering  $\ddp_1 (\bw_1)$, for $\bw_1$,
\begin{eqnarray}
 - \lambda_2(R_1(\bw_1))^2 +2\lambda_1 \E_{\theta_0}R_0(\theta_0)\left|w_1(\theta_0 , \bw_1) \right| +  f_1^{p'}(p_1) = \lambda^1_p,  \quad \text{a.e.} ; \label{p1}
\end{eqnarray}
\item The constraints hold.
    \end{enumerate}

    \item Because Problem \eqref{problem 3} is convex, it is sufficient to show that $(p_1, p_2, p_3, w_2, w_3, u)$ satisfies the  KKT condition.  Clearly,  the constraints holds.   Moreover, $g_4(\btheta_3) =0$ indicates \eqref{tu}.  $g_3(\btheta_2, \btheta_3) =0$ implies \eqref{tw23} with $\Lambda_3^i(\btheta_3) = q_3^i(\btheta_3)$.  $g_2(\bw_1, \btheta_2) =0$ implies \eqref{tw12} with $\Lambda_2^i(\btheta_2) = q_2^i(\btheta_2)$.

    \item We prove \eqref{p3}.     
    Since $g_3=0$, we have $ \int g_3(\btheta_2, \btheta_3) w_3(\btheta_2, \btheta_3) dp(\btheta_2)=0 $. 
    Applying the formula of $g_3$ in \eqref{ggf2}, we obtain that 
\begin{eqnarray}\label{p3f1}
\sum_{i=1}^N q_3^i(\btheta_3)\btheta_3(i) +2 \lambda_3 \E_{\btheta_2}R_2(\btheta_2)|w_2(\btheta_2, \btheta_3)|=0,\quad \text{a.e.},
\end{eqnarray}
where we used $\left|w_3(\btheta_2, \btheta_3) \right|' w_3(\btheta_2, \btheta_3) = |w_3(\btheta_2, \btheta_3) |$, and the constraint that
\begin{eqnarray}\label{constrain1}
\int h(\btheta_2(i)) w_3(\btheta_2, \btheta_3) dp(\btheta_2) = \btheta_3(i), \quad i \in [N].
\end{eqnarray}
Similarly, consider $\int g_3(\btheta_2, \btheta_3)\nabla_{\btheta_3,j}w_3(\btheta_2, \btheta_3)dp(\btheta_2) =0$, where $\nabla_{\btheta_3,j}$ denote the gradient on coordinate $\btheta_3(j)$.
We obtain that, for almost every $\btheta_3$,
\begin{eqnarray}\label{p3temp1}
 q_3^j(\btheta_3) +2 \lambda_3 \E_{\btheta_2} R_2(\btheta_2)|w_3(\btheta_2, \btheta_3)|'\nabla_{\btheta_3,j}w_3(\btheta_2, \btheta_3)=0,
\end{eqnarray}
where  we used
\[
\int h(\btheta_2(i)) \nabla_{\btheta_3,j} w_3(\btheta_2, \btheta_3) dp_2(\btheta_2)  = \cI_{i=j}, \quad i \in [N],
\]
from \eqref{constrain1}.
Then it follows from \eqref{p3temp1} that
\begin{eqnarray}\label{p3f2}
\sum_{i=1}^N \int q_3^i(\btheta_3) d\btheta_3(i) +2 \lambda_3 \E_{\btheta_2} R_2(\btheta_2)|w_3(\btheta_2, \btheta_3)|=C_{1}, \quad \text{a.e.},
\end{eqnarray}
where $C_{1}$ is a constant.  
In the same way, since $g_4=0$, we have
\begin{align}
&\sum_{i=1}^N[\phi'_ih(\btheta_3(i))]u(\btheta_3) + 2 \lambda_4 u(\btheta_3)^2 =0,\label{p3f3}\\
&    \sum_{i=1}^N\sum_{j=1}^N\int \phi'_ih(\btheta_3(i))\nabla_{
\btheta_3,j}u(\btheta_3) d\btheta_3(j) +  \lambda_4 u(\btheta_3)^2 =C_{2},\label{p3f4}
\end{align}
for a constant $C_2$.
Finally,  computing \eqref{p3f2} $+$ \eqref{p3f4} $-$ \eqref{p3f1}$-$\eqref{p3f3}, using the fact that 
\begin{eqnarray}
\sum_{i=1}^N\sum_{j=1}^N\int \phi'_i h(\btheta_3(i))\nabla_{
\btheta_3,j}u(\btheta_3) d\btheta_3(j) + \sum_{i=1}^N \int \phi'_i h'(\btheta_3(i))u(\btheta_3) d\btheta_3(i) =  \sum_{i=1}^N\phi'_ih(\btheta_3(i))u(\btheta_3) + C_{3}, \notag
\end{eqnarray}
we can obtain \eqref{p3}.

\item   We use the same technique to prove \eqref{p2}.  
Since $g_2=0$, we have $ \EE_{\bw_1} [g_2(\bw_1, \btheta_2) w_2(\bw_1, \btheta_2)]=0 $. 
Applying the formula of $g_2$ in \eqref{ggf3}, we obtain that 
\begin{eqnarray}\label{p2f1}
\sum_{i=1}^N q_2^i(\btheta_2)h(\btheta_2(i)) + 2\lambda_2 \E_{\bw_1} R_1(\bw_1) | w_1(\bw_1, \btheta_2) |=0,  \quad \text{a.e.}
\end{eqnarray}
Similarly, since $\EE_{\bw_1}[g_2(\bw_1, \btheta_2)\nabla_{\btheta_2,j} w_2(\bw_1, \btheta_2)]=0$, we have
\begin{eqnarray}\label{p2temp}
 q_2^j(\btheta_2) + 2\lambda_2 \E_{\bw_1} R_1(\bw_1) | w_2(\bw_1, \btheta_2) |' \nabla_{\btheta_2,j} w_2(\bw_1, \btheta_2)=0,  \quad \text{a.e.}
\end{eqnarray}
It follows from \eqref{p2temp} that
\begin{eqnarray}\label{p2f2}
\sum_{i=1}^N \int q_2^i(\btheta_2) d\btheta_2(i) + 2\lambda_2 \E_{\bw_1} R_1(\bw_1) \left| w_2(\bw_1, \btheta_2) \right|=C_{4}.
\end{eqnarray}
In the same way,  using $g_3=0$, we have
\begin{align}
&\sum_{i=1}^N \E_{\btheta_3} q^i_3(\btheta_3) h(\btheta_2(i)) w_3(\btheta_2, \btheta_3)  +2\lambda_3 \underbrace{\E_{\btheta_3}  R_2(\btheta_2)| w_3(\btheta_2, \btheta_3)|}_{R_2(\btheta_2)^2}=0,\label{p2f3}\\
&\sum_{i=1}^N\sum_{j=1}^N \E_{\btheta_3}  \int q^i_3(\btheta_3) h(\btheta_2(i)) \nabla_{\btheta_2,j }w_3(\btheta_2, \btheta_3) d\btheta_2(j)  + \lambda_3 R_2(\btheta_2)^2=C_{5}.\label{p2f4}
\end{align}
Finally, computing  
\eqref{p2f2} $+$ \eqref{p2f4} $-$ \eqref{p2f1}$-$\eqref{p2f3}, using the fact that 
\begin{eqnarray}
&&\sum_{i=1}^N \sum_{j=1}^N\int  q_3^i(\btheta_3) h(\btheta_2(i))\nabla_{
\btheta_2,j}w_3(\btheta_2,\btheta_3)d\btheta_2(j) + \sum_{i=1}^N \int  q_3^i(\btheta_3) h'(\btheta_2(i))w_3(\btheta_2,\btheta_3) d\btheta_2(i) \notag\\
&=& \sum_{i=1}^N q_3^i(\btheta_3) h(\btheta_2(i))w_3(\btheta_2,\btheta_3)+C_{6}, \notag
\end{eqnarray}
we can obtain \eqref{p2}.
\item We prove \eqref{p1} in a similar way. 
Using $g_2=0$, we obtain that
\begin{align}
&\sum_{i=1}^N \E_{\btheta_2} q_2^i(\btheta_2)w_2(\bw_1, \btheta_2)h(\bw_1^\top x^i) + 2\lambda_2 \underbrace{\E_{\btheta_2} R_1(\bw_1)| w_2(\bw_1, \btheta_2)|}_{ R_1(\bw_1)^2} =0,\label{p11}\\
&\sum_{i=1}^N\sum_{j=1}^d \E_{\btheta_2}\int q_2^i(\btheta_2)\nabla_{\bw_1,j} w_2(\bw_1, \btheta_2)h\left(\bw_1^\top \bx^i\right)d \bw_1(j) + \lambda_2 R_1(\bw_1)^2 =0.\label{p12}
\end{align}
Finally, computing \eqref{p12} $+$ $\sum_{\theta_0=1}^d \int g_1(\theta_0, \bw_1) d\bw_1(j) -$ \eqref{p11}, and using the fact that 
\begin{eqnarray}
&&C_{7}+\sum_{i=1}^N\sum_{j=1}^d \int q_2^i(\btheta_2)\nabla_{\bw_1,j} w_2(\bw_1, \btheta_2)h\left(\bw_1^\top \bx^i\right)d \bw_1(j)\notag\\
&+& \!\!\! \sum_{i=1}^N \sum_{j=1}^d \ \int q_2^i(\btheta_2) w_2(\bw_1, \btheta_2)h'\left(\bw_1^\top \bx^i\right)\bx^i(j)d \bw_1(j) = \sum_{i=1}^N  q_2^i(\btheta_2) w_2(\bw_1, \btheta_2)h\left(\bw_1^\top \bx^i\right),  \notag
\end{eqnarray}
we can obtain \eqref{p1}, which completes the proof. 
\end{enumerate}

\end{proof}

\section{Preliminary}
In this paper, we adopt the definition of sub-gaussian distributions in \cite{vershynin2010introduction}.
Below we present properties of sub-gaussian distributions. 
The equivalence among those properties are given in \cite[Lemma 5.5]{vershynin2010introduction}.

% We adopt the definition of sub-gaussian distribuitions in \citep{vershynin2010introduction}. 
% We present some basic results for sub-gaussian distributions.

\begin{lemma}\label{theo:subgau}
Let $\xi $ be an univariate random variable that follows a $\sigma$-sub-gaussian distribution. Then there exists an absolute constant $C$ such that
\begin{enumerate}[(1)]
    \item  \label{theo:subgau1} Tails $ \PP(|\xi| >t) \leq \exp(1- t^2/(C\sigma)^2)$ for all $t\ge 0$;
    \item  \label{theo:subgau2} Moments: $(\EE | \xi|^q)^{1/q} \leq C\sigma \sqrt{q} $ for all $q\geq1$;
    \item  \label{theo:subgau4} If $\E[\xi]=0$, then  $\EE [\exp(t\xi)] \leq \exp (t^2 (C\sigma)^2)$ for all $t\in \RR$.
\end{enumerate}
\end{lemma}

% \begin{lemma}[Sub-gaussian Distribution \citep{vershynin2010introduction}]\label{theo:subgau}
% Let $\xi $ be random variable that follows a sub-gaussian distribution, the following statement for equivalent:
% \begin{enumerate}[(1)]
%     \item  \label{theo:subgau1} Tails $ \mathcal{P}(|\xi| >t) \leq \exp(1- t^2/K_1^2).$
%     \item  \label{theo:subgau2} Moments: $(\EE | \xi|^q)^{1/q} \leq K_2 \sqrt{q}. $ for all $q\geq1$.
%     \item  \label{theo:subgau3} Super-exponential moment: $\EE\exp(\xi^2/K_3^2) \leq e  $.
%     \item  \label{theo:subgau4} if $\E[\xi]=0$, then  $\EE \exp(t\xi) \leq \exp (t^2 K_4^2)$.
% \end{enumerate}
% The equivalence means that there exists an absolute constant $C_s$ such that property $i$ implies property $j$ with parameter $K_j\leq C_s K_i$, for any two properties $i,j = 1,2,3,4$. 
% \end{lemma}

\begin{lemma}[Concentration Inequality for Sub-gaussian Distributions {\cite[Proposition 5.10]{vershynin2010introduction}}]\label{azuma}
Let $\{\xi_i\}_{i=1}^m$ be independent centered $\sigma$-sub-gaussian random variables.  Then, for an absolute constant $C$,
$$  \PP\left( \left|\frac{1}{m}\sum_{i=1}^m\xi_i\right|\geq \ep     \right) \leq 2 \exp\left(-\frac{ m\ep^2}{4(C\sigma)^2}\right).  $$
\end{lemma}

% \begin{lemma}[Azuma's Inequality with Sub-gaussian Tails \citep{vershynin2010introduction}]\label{azuma}
% Let $\{\xi_i\}_{i=1}^m$ be independent  sub-gaussian random variables such that  $\EE[\xi_i]=0$ and  $\EE \exp(t\xi_i) \leq \exp (t^2 K_4^2)$ for all $i\in[m]$.  We have
% $$  \mathcal{P}\left( \left|\frac{1}{m}\sum_{i=1}^m\xi_i\right|\geq \ep     \right) \leq 2 \exp\left(\frac{ -m\ep^2}{2K_4^2}\right).  $$
% \end{lemma}

We say a random variable $\xi$ is sub-exponential if 
\[
\sup_{q\ge 1}q^{-1}(\EE|\xi|^q)^{1/q}<\infty.
\]
A sub-exponential random variable is equivalent to the a squared sub-gaussian random variable \cite[Lemma 5.14]{vershynin2010introduction}.
It satisfies the following concentration inequality:

\begin{lemma}[Bernstein's Inequality for Sub-Exponential Distributions {\cite[Corollary 5.17]{vershynin2010introduction}}]\label{azumaexp}
Let $\{\xi_i\}_{i=1}^m$ be independent centered sub-exponential random variables such that $(\EE|\xi|^q)^{1/q}\le Kq$ for all $q\ge 1$.
Then, for an absolute constant $c$,
$$  \PP\left( \left|\frac{1}{m}\sum_{i=1}^m z_i\right|\geq \ep     \right) \leq  2\exp\left[-cm\min\left(\frac{\ep^2}{K^2}, \frac{\ep}{K} \right)\right].  $$
\end{lemma}

% \begin{lemma}[Bernstein's Inequality with Sub-Exponential Tails \citep{vershynin2010introduction}]\label{azumaexp}
% Let $\{\xi_i\}_{i=1}^m$ be independent  sub-gaussian random variables such that   $\EE \exp(t\xi_i) \leq \exp (t^2 K_4^2)$.  For $i\in[m]$, let $z_i = \xi_i^2 - \E [\xi_i^2]$, then $z_i$ is a sub-exponential random  variable, such that 
% $$  \EE \exp(t z_i) \leq \exp \left[t^2 \left(C_{s_1}K_4^2)^2/2\right)\right],  \quad\quad \forall~ | t| \leq  \frac{1}{C_{s_1}K_4^2}, $$
% where $C_{s_1}$ is a constant.  Moreover, we have
% $$  \mathcal{P}\left( \left|\frac{1}{m}\sum_{i=1}^m z_i\right|\geq \ep     \right) \leq  2\exp\left[-\frac{m}{2}\min\left(\frac{\ep^2}{(C_{s_1}K_4^2)^2}, \frac{\ep}{C_{s_1}K_4^2} \right)\right].  $$
% \end{lemma}

\begin{lemma}\label{sug-gau vec}
For a $d$-dimensional random vector $\bxi\in\RR^{d}$, we have:
\begin{enumerate}[(1)]
    \item \label{f11} $\|\bxi\|_\infty$ is $\sigma$-sub-gaussian $\Longrightarrow$ 
    $\bxi$ is $(\sigma\sqrt{d})$-sub-gaussian;
    \item \label{f12} $\bxi$ is $\sigma$-sub-gaussian $\Longrightarrow$ 
    $\|\bxi\|_\infty$ is $\cO(\sigma\sqrt{\log d})$-sub-gaussian.
\end{enumerate}
\end{lemma}
\begin{proof}
(\ref{f11}) For any $\bu\in  \mathbb{S}^{d-1}$, we have $|\bu^\top \bxi|\le \sqrt{d}\|\bxi\|_\infty$.

(\ref{f12}) 
Note that $\|\bxi\|_\infty=\max_{i\in[d]} |\e_i^\top \bxi| $, where $\e_i$ denotes the unit vector with $i$-th coordinate being one.
Applying Lemma~\ref{theo:subgau} (\ref{theo:subgau1}) and the union bound yields that
\[
\PP[\|\bxi\|_\infty>t]
\le \min\left\{d e^{1-\frac{t^2}{(C\sigma)^2}},1\right\}
% =\exp\left\{-\frac{(t^2-(C\sigma)^2\log(ed))_+}{(C\sigma)^2}\right\}
= e^{-\left(\frac{t^2}{(C\sigma)^2}-\log (ed)\right)_+}
\le e^{1-\frac{t^2}{(C\sigma \sqrt{\log(ed)})^2}},
\]
where we used the fact that $(\frac{t^2}{a}-b)_+\ge \frac{t^2}{ab}-1$ when $b\ge 1$. 
Therefore, $\|\bxi\|_\infty$ is $\cO(\sigma\sqrt{\log d})$-sub-gaussian by the equivalent definition of sub-gaussian distributions in \cite[Lemma 5.5]{vershynin2010introduction}.
\end{proof}

\begin{corollary}
\label{cor:subg-q}
For a $\sigma$-sub-gaussian random  vector $\bxi\in\RR^{d}$, we have 
$$ (\EE \| \bxi\|^q_{\infty})^{1/q} \leq \cO(\sigma\sqrt{q \log d}), \quad q\geq1. $$
\end{corollary}
\begin{proof}
From Lemma \ref{sug-gau vec},  $\|\bxi \|_{\infty}$ is $\cO(\sigma\sqrt{\log d})$-sub-gaussian. Applying Lemma \ref{theo:subgau} (\ref{theo:subgau2}), we achieve the desired result. 
\end{proof}

\end{document}